\documentclass{article}

\usepackage{microtype}
\usepackage{graphicx}
\usepackage{subcaption}
\usepackage{booktabs} 

\usepackage{hyperref}

\usepackage[preprint]{icml2026}

\usepackage{amsmath}
\usepackage{amssymb}
\usepackage{mathtools}
\usepackage{amsthm}
\usepackage{comment}

\usepackage[capitalize,noabbrev]{cleveref}

\theoremstyle{plain}
\newtheorem{theorem}{Theorem}[section]
\newtheorem{proposition}[theorem]{Proposition}
\newtheorem{lemma}[theorem]{Lemma}
\newtheorem{corollary}[theorem]{Corollary}
\theoremstyle{definition}

\newtheorem{assumption}[theorem]{Assumption}
\theoremstyle{remark}
\newtheorem{remark}[theorem]{Remark}

\usepackage{amsthm}
\renewcommand{\theassumption}{\thesection.\arabic{assumption}}

\definecolor{LightRed}{RGB}{255,182,193} 
\definecolor{LightBlue}{RGB}{173, 216, 230}
\definecolor{LightMint}{RGB}{170,235,205} 

\colorlet{shadecolor}{LightBlue}
\colorlet{shadecolor}{LightRed}
\colorlet{shadecolor}{LightMint}

\usepackage[textsize=tiny]{todonotes}

\usepackage{multirow}

\usepackage{xcolor}         

\usepackage{tabularx}
\usepackage{graphicx}
\usepackage{colortbl}

\definecolor{LightRed}{RGB}{255,182,193} 
\definecolor{LightBlue}{RGB}{173, 216, 230}
\definecolor{LightGreen}{RGB}{144, 238, 144}

\colorlet{shadecolor}{LightGreen}
\colorlet{shadecolor}{LightBlue}
\colorlet{shadecolor}{LightRed}

\newcommand{\inc}[2]{#1{\scriptsize\;($\uparrow #2$)}}
\usepackage{enumitem}
\setlist[enumerate]{leftmargin=3em}

\usepackage{arydshln} 

\icmltitlerunning{Taming Preconditioner Drift in Second-Order Federated Learning  on Non-IID Data}

\begin{document}

\twocolumn[
  \icmltitle{Taming Preconditioner Drift: Unlocking the Potential of Second-Order Optimizers for Federated Learning on Non-IID Data}



\begin{icmlauthorlist}
	\icmlauthor{Junkang Liu}{tju}
	\icmlauthor{Fanhua Shang$^*$}{tju}
	\icmlauthor{Hongying Liu$^*$}{tju2,pcl}
	\icmlauthor{Jin Liu}{xdu2}
    \icmlauthor{Weixin An}{xdu}
    \icmlauthor{Yuanyuan Liu$^*$}{xdu}
\end{icmlauthorlist}

\icmlaffiliation{tju}{College of Intelligence and Computing, Tianjin University, Tianjin, China}
\icmlaffiliation{tju2}{Medical School, Tianjin University, Tianjin, China}
\icmlaffiliation{pcl}{Peng Cheng Lab, Shenzhen, China}
\icmlaffiliation{xdu}{School of Artificial Intelligence, Xidian University, Xi'an, China}
\icmlaffiliation{xdu2}{School of Cyber Engineering, Xidian University, Xi'an, China}

\icmlcorrespondingauthor{Yuanyuan Liu}{yyliu@xidian.edu.cn}
\icmlcorrespondingauthor{Fanhua Shang}{fhshang@tju.edu.cn}
\icmlcorrespondingauthor{Hongying Liu}{hyliu2009@tju.edu.cn}

  \icmlkeywords{Machine Learning, ICML}

  \vskip 0.3in
]



\printAffiliationsAndNotice{}  

\begin{abstract}
	Second-order optimizers can significantly accelerate large-scale training, yet their naive federated variants are often unstable or even diverge on non-IID data.
	We show that a key culprit is \emph{preconditioner drift}: client-side second-order training induces heterogeneous \emph{curvature-defined geometries} (i.e., preconditioner coordinate systems), and server-side model averaging updates computed under incompatible metrics, corrupting the global descent direction.
	To address this geometric mismatch, we propose \texttt{FedPAC}, a \emph{preconditioner alignment and correction} framework for reliable federated second-order optimization.
	\texttt{FedPAC} explicitly decouples parameter aggregation from geometry synchronization by:
	(i) \textbf{Alignment} (i.e.,aggregating local preconditioners into a global reference and warm-starting clients via global preconditioner); and
	(ii) \textbf{Correction} (i.e., steering local preconditioned updates using a global preconditioned direction to suppress long-term drift).
	We provide drift-coupled non-convex convergence guarantees with linear speedup under partial participation.
	Empirically, \texttt{FedPAC} consistently improves stability and accuracy across vision and language tasks, achieving up to $5.8\%$ absolute accuracy gain on CIFAR-100 with ViTs.
	Code is available at \url{https://anonymous.4open.science/r/FedPAC-8B24}.
\end{abstract}

\newcommand{\methodname}{\textsc{FedSubFull}}  

\section{Introduction}
\label{sec:intro}

In recent years, the training paradigm of large-scale models, especially large language models (LLMs) \cite{devlin2019bert}, has shifted significantly. It is increasingly recognized that the optimizer itself is part of the computational budget: under a fixed compute budget, better optimization methods can directly yield shorter training time and higher model quality. Second-order optimizers, such as Sophia \cite{liu2023sophia}, SOAP \cite{vyas2024soap}, and Muon \cite{jordan6muon}, leverage richer curvature information and have demonstrated substantial acceleration in centralized large-scale training, achieving 1.5–2× faster convergence over AdamW and SGD \cite{abreu2025potential,loshchilov2017fixing}. This evidence suggests that second-order optimization is no longer merely a theoretical luxury, but is rapidly becoming a practical backbone of large-scale deep learning.
\begin{figure}[!t]
	\centering
	\includegraphics[width=\linewidth]{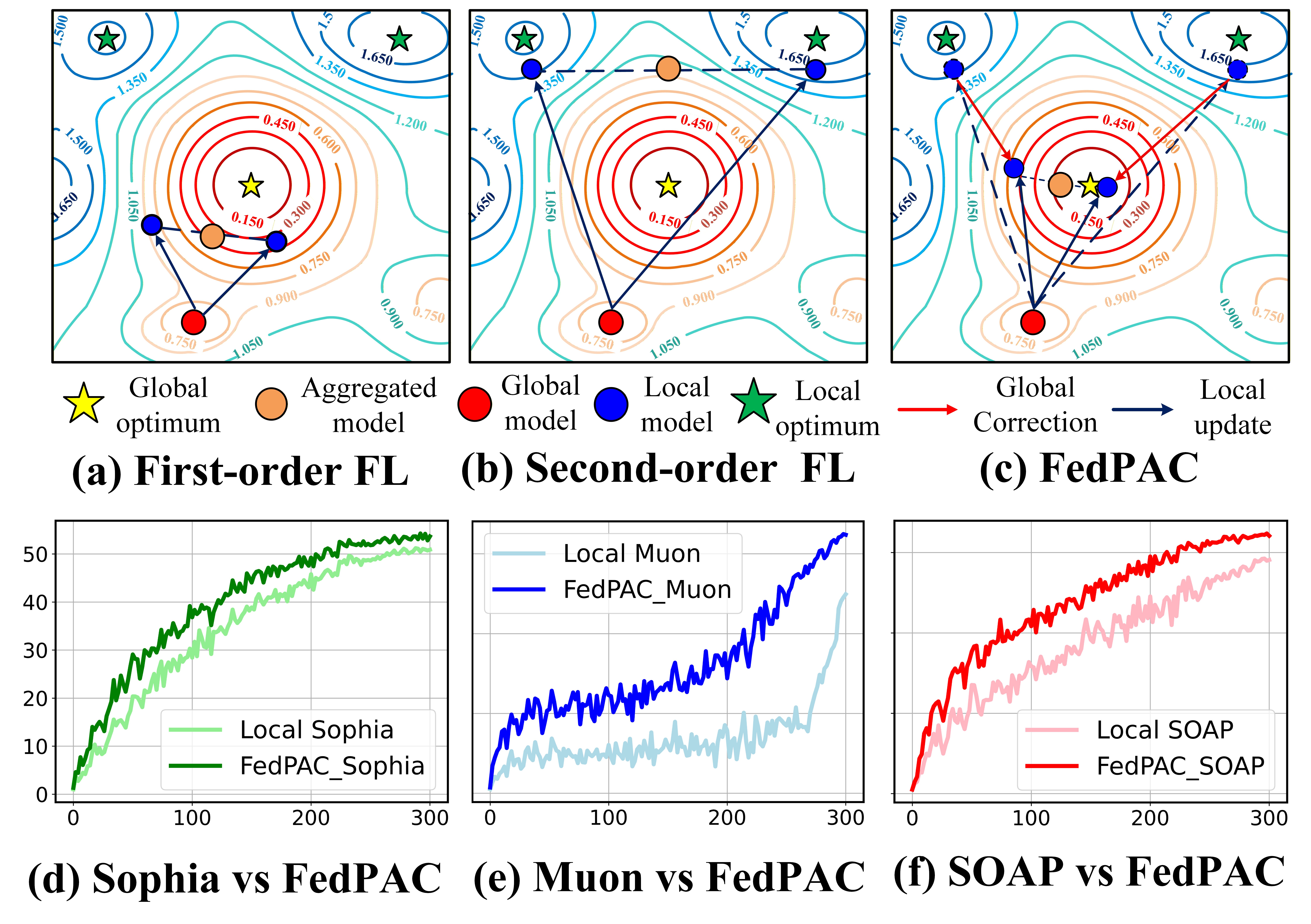}
	\vspace{-4mm}
	\caption{\small (a) In non-IID FL, first-order methods converge slowly, inducing little client drift. (b) Second-order methods converge faster locally and thus drift toward local optima, causing the aggregated global model to deviate from  global optimum. (c) \texttt{FedPAC} corrects local second-order updates, yielding faster convergence and a global model closer to the global optimum. (d–f) \texttt{FedPAC} accelerates Sophia, Muon and SOAP to train ResNet-18 on CIFAR-100. The x-axis denotes the number of  communication rounds, and the y-axis denotes test accuracy.}
	\vspace{-3mm}  
	\label{fig:fedpac}
\end{figure}
\vspace{-3mm}  

In contrast, another highly practical and increasingly important training paradigm, Federated Learning (FL) \cite{mcmahan2017communication}, has remained largely confined to the first-order regime in terms of optimization. Most widely adopted federated algorithms, including FedAvg \cite{mcmahan2017communication},  and various variance-reduced variants, rely on first-order SGD updates \cite{bottou2010large}. Although these methods partially mitigate data heterogeneity and communication constraints, they continue to suffer from evident limitations in convergence speed, scalability to large models. This discrepancy naturally raises the following question:

\begin{figure*}[t]
	\centering
	\begin{subfigure}[b]{0.245\textwidth}
		\centering
		\includegraphics[width=\linewidth]{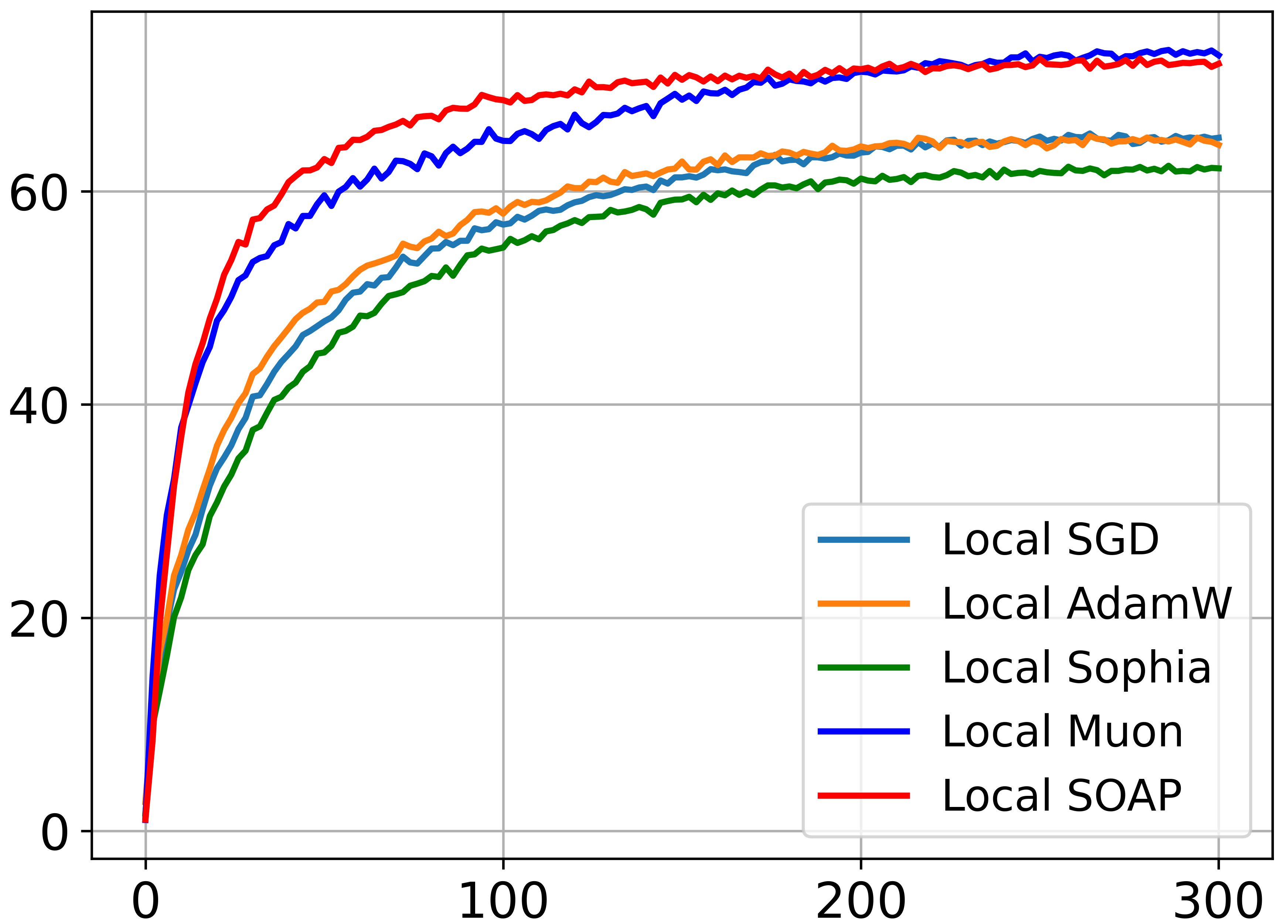}
		\caption{ResNet-18, IID}
		\label{fig:resnet_iid}
	\end{subfigure}
	\begin{subfigure}[b]{0.245\textwidth}
		\centering
		\includegraphics[width=\linewidth]{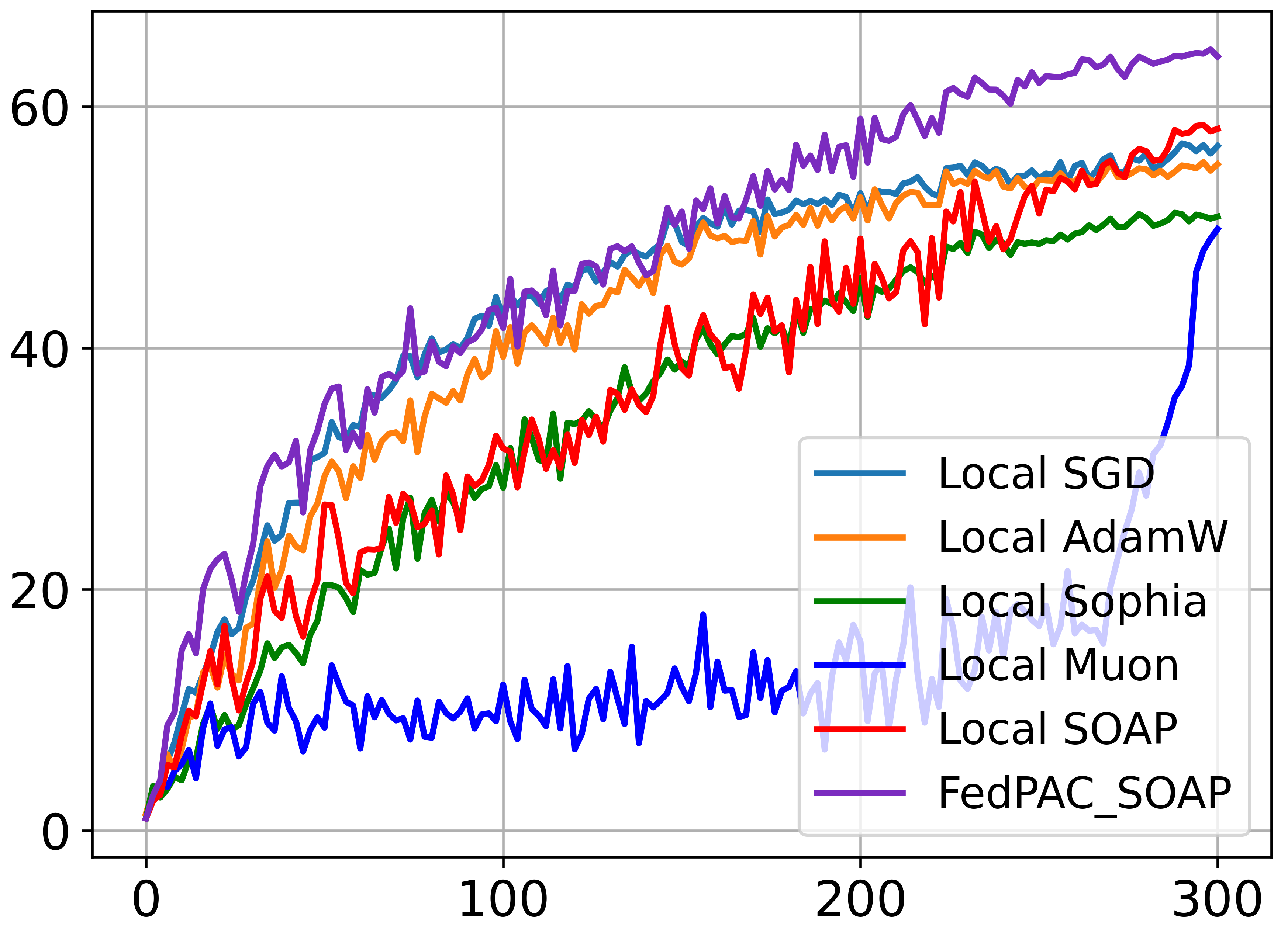}
		\caption{ResNet-18, non-IID}
		\label{fig:vit_iid}
	\end{subfigure}
	\begin{subfigure}[b]{0.245\textwidth}
		\centering
		\includegraphics[width=\linewidth]{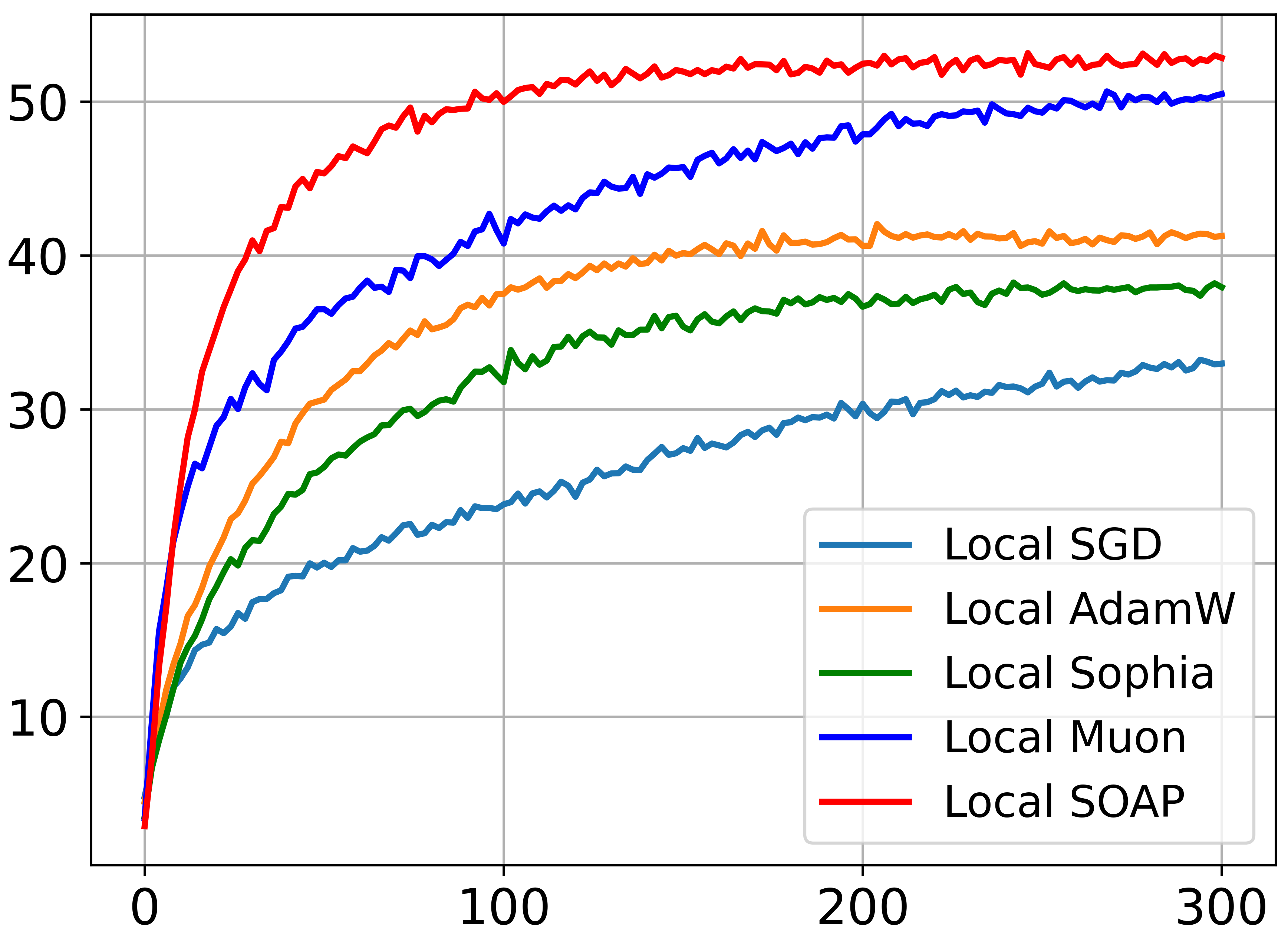}
		\caption{ViT-Tiny, IID}
		\label{fig:sgd_grid_search:bert}
	\end{subfigure}
	\begin{subfigure}[b]{0.245\textwidth}
		\centering
		\includegraphics[width=\linewidth]{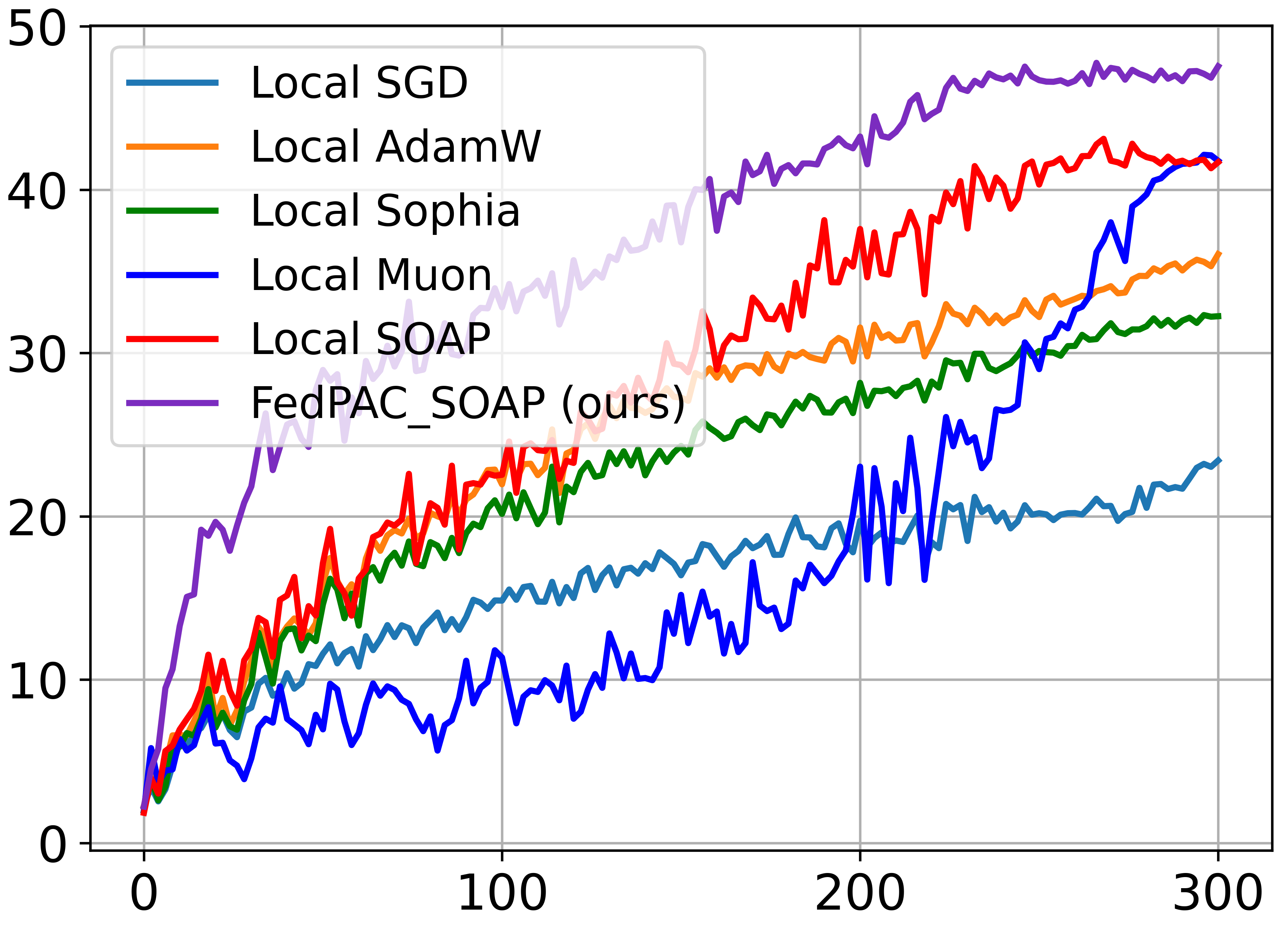}
		\caption{ViT-Tiny, non-IID}
		\label{fig:sgd_grid_search:bert}
	\end{subfigure}
	\vspace{-2mm}
	\caption{\small  The x-axis denotes   communication rounds, and the y-axis is test accuracy. (a, c): In FL on IID data, second-order optimizers converge significantly faster than SGD and AdamW for training  ResNet and Transformer. (b),(d): However, on non-IID data,  second-order optimizer (e.g., Local Muon) converges much more slowly and can even underperform first-order methods such as Local SGD.}
	\label{fig:non_iid_local}
\end{figure*}

\emph{If second-order optimizers have already demonstrated their effectiveness in centralized large-scale training, why can we not achieve similar acceleration benefits in FL?}

A natural baseline is to employ Sophia/SOAP/Muon as the local optimizer on each client and aggregate parameters as in FedAvg named Local Sophia/SOAP/Muon, as shown in Figure \ref{fig:non_iid_local}. Yet, on non-IID data \cite{liuimproving}, such a naive federated adaptation exhibits a critical failure mode: it fails to deliver the expected acceleration and often results in slower convergence or even divergence as shown  in Figure \ref{fig:non_iid_local}.  Our systematic experiments and mechanistic analysis indicate that this behavior is not mainly due to the computational cost of second-order information or insufficient client-side compute. Instead, it is driven by a previously underexamined phenomenon, which we refer to as  \emph{preconditioner drift}.

In centralized training, second-order preconditioners (e.g., the matrix preconditioners  implicit curvature structures exploited by SOAP and Muon) are continuously updated under a single data distribution. In federated learning, however, each client is exposed to its own non-IID data distribution. As clients perform multiple local updates, their preconditioners adapt toward the local geometry associated with their respective data distributions, which leads to the following effects as in Figure \ref{fig:fedpac}:  (\textit{i})  The preconditioners across different clients gradually drift apart in both scale and orientation;  (\textit{ii})  The server aggregates only the model parameters via simple averaging, implicitly assuming that all clients employ similar preconditioners;  (\textit{iii})  The resulting global update direction effectively combines inconsistent local curvature estimates to update a shared model, which severely distorts the global optimization trajectory.

From an intuitive perspective, each client makes effective progress under its own curvature-adapted coordinate system; however, when these trajectories are naively aggregated, the resulting global optimization can become highly inefficient as  in Figure \ref{fig:fedpac}. Consistent with this intuition, we observe that in non-IID data settings, directly applying second-order optimizers not only fails to replicate the acceleration achieved in centralized training, but may even perform worse than simple first-order methods as in Figure \ref{fig:non_iid_local}. This observation gives rise to an urgent and natural question:

\emph{Do second-order optimizers still hold genuine potential in data heterogeneous federated learning?}

We answer this question affirmatively. We demonstrate that second-order optimizers can achieve substantial acceleration in FL when equipped with a dedicated framework to explicitly mitigate \emph{ preconditioner drift} as in Figure \ref{fig:fedpac}. Accordingly, we propose \texttt{FedPAC} (\textbf{\underline{Fe}}derate\textbf{\underline{d}} \textbf{\underline{P}}reconditioner \textbf{\underline{A}}lignment and \textbf{\underline{C}}orrection), a principled correction and acceleration framework for second-order federated optimization.
In summary, main contributions are as follows:\\
\textbf{$\bullet$} We identify a previously underexplored failure mode of second-order federated learning: \emph{preconditioner drift}---the mismatch of client-side curvature-induced geometries. We formalize it with an explicit drift metric and empirically show it strongly correlates with degraded convergence.\\
\textbf{$\bullet$} We develop a unified framework, \texttt{FedPAC}, that targets this geometric mismatch by \emph{decoupling} parameter updates from preconditioner synchronization: clients perform efficient local second-order updates, while preconditioners are periodically \emph{aligned} and local steps are \emph{corrected} using lightweight global curvature statistics.\\
\textbf{$\bullet$} We provide drift-coupled convergence guarantees under standard smoothness and bounded-heterogeneity assumptions, where the optimization error contains an explicit \emph{drift term}. We show \texttt{FedPAC} provably reduces this term and achieves faster convergence than first-order FL and naive second-order baselines.

\section{Related Work}
\label{sec:related}
\vspace{-2mm}
\subsection{First-order FL under data heterogeneity}
\vspace{-2mm}
In heterogeneous federated learning, most improvements of first-order methods target mitigating client drift and convergence instability caused by non-IID data. FedProx ~\cite{li2018federated} stabilizes local updates via a proximal regularization term, while SCAFFOLD ~\cite{karimireddy2020scaffold} theoretically analyzes client drift in FedAvg and corrects it using server–client control variates. FedCM~\cite{xu2021fedcm} leverages client momentum to stabilize updates. Nevertheless, these methods remain fundamentally first-order, operating solely at the gradient level without explicitly leveraging second-order geometry such as the Hessian or Gauss–Newton information, leaving considerable room to explore the potential of second-order optimization in heterogeneous federated learning. Our alignment operates on the \emph{preconditioner operators} that define the local metric, not on first-order states such as momentum or control variates.

\subsection{Second-order optimization methods}
\vspace{-2mm}
Traditional second-order methods such as Newton’s method are often prohibitively expensive. Recently, several practical second-order or quasi-second-order optimizers have been shown to scale to large models in centralized deep learning, including Shampoo \cite{gupta2018shampoo}, SOAP \cite{vyas2024soap}, Muon \cite{jordan6muon}, and Sophia \cite{liu2023sophia}. Shampoo \cite{gupta2018shampoo} performs structured gradient preconditioning via Kronecker-factored approximations of the Hessian. SOAP \cite{vyas2024soap} stabilizes such preconditioning by running AdamW in the induced feature basis. Muon \cite{jordan6muon} orthogonalizes gradient momentum using Newton–Schulz iterations. Sophia \cite{liu2023sophia} leverages lightweight diagonal Hessian estimates with element-wise clipping. 
\subsection{Second-order optimization in FL}
\vspace{-2mm}
The utilization of second-order information in FL is still in its early stages but is developing rapidly. Fed-Sophia \cite{elbakary2024fed} adapts the Sophia optimizer to the federated setting and proposes a communication-efficient second-order federated algorithm. FedMuon~\cite{liu2025fedmuonacceleratingfederatedlearning} introduces the matrix-orthogonalization-based Muon optimizer into FL.
Our framework generalizes the alignment and correction principle in FedMuon, and thus the latter can be viewed as a special case of the former.
FedPM \cite{ishii2025fedpm} systematically studies the empirical behavior of the second-order classical Newton method in FL, which is computational overhead and  not suitable for large-scale models.


\section{The Proposed Unified Algorithm}
\vspace{-2mm}
\label{sec:method}
\subsection{Problem Setup}
\vspace{-2mm}
FL aims to optimize global model with the collaboration local clients, i.e., minimizing the following population risk:
\vspace{-2mm}
\begin{equation}
	F(\boldsymbol{x})=\frac{1}{N} \sum_{i=1}^N\left(F_i(\boldsymbol{x}):=\mathbb{E}_{\xi_i\sim \mathcal{D}_i}\left[F_i\left(\boldsymbol{x} ; \xi_i\right)\right]\right).
	\label{eq 1}
\end{equation}
The function $F_i$ is the loss function on client $i$. $\mathbb{E}_{\xi_i \sim \mathcal{D}_i}[\cdot]$ denotes conditional expectation with respect to  the sample $\xi_i$. $N$ is the number of clients, and 
$\boldsymbol{x}$ is global model.

\begin{figure*}[tb]
	\centering
	\begin{subfigure}[b]{0.245\textwidth}
		\centering
		\includegraphics[width=\linewidth]{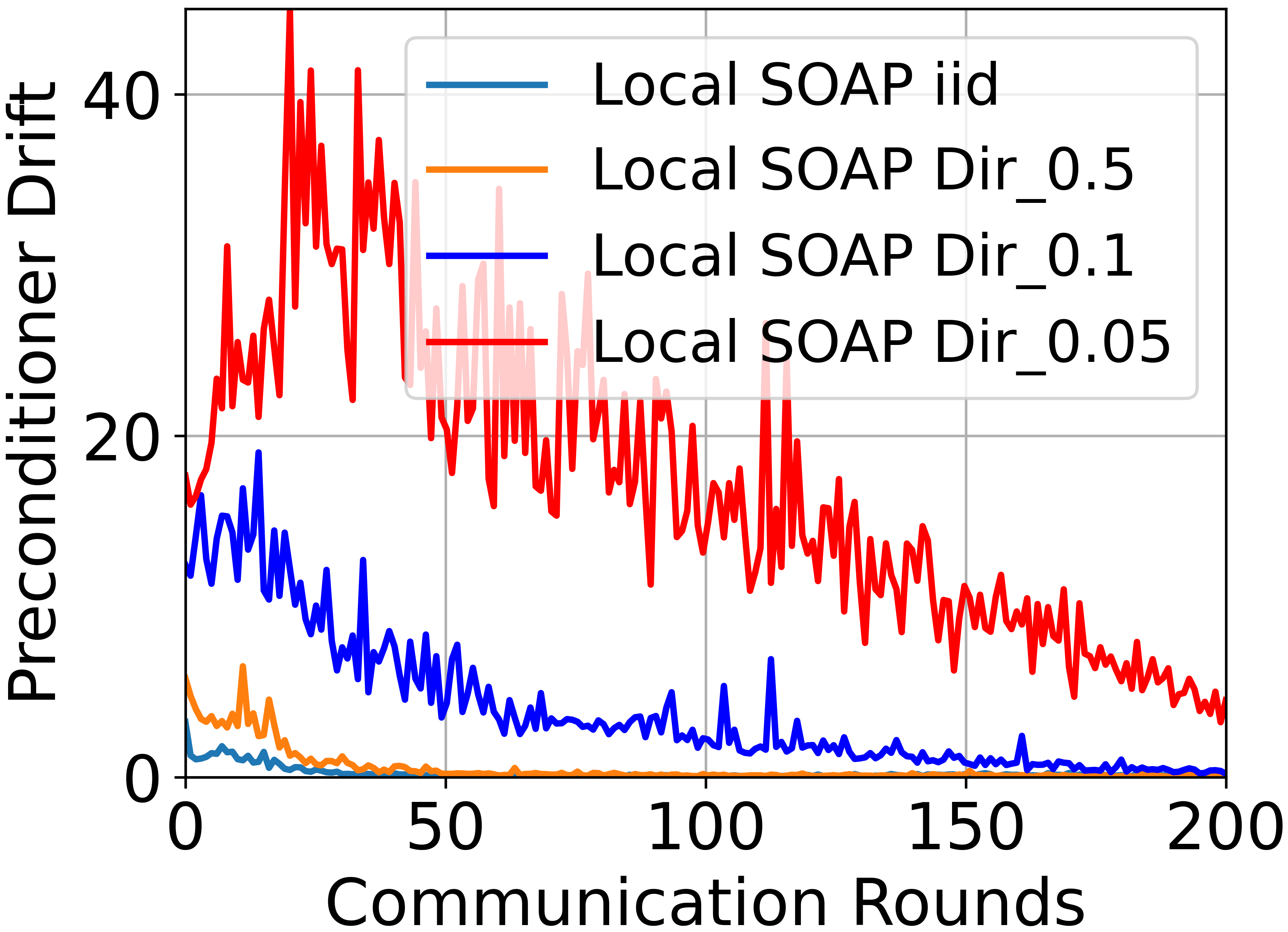}
		\caption{Local SOAP }
		\label{fig:resnet_iid}
	\end{subfigure}
	\hfill
	\begin{subfigure}[b]{0.245\textwidth}
		\centering
		\includegraphics[width=\linewidth]{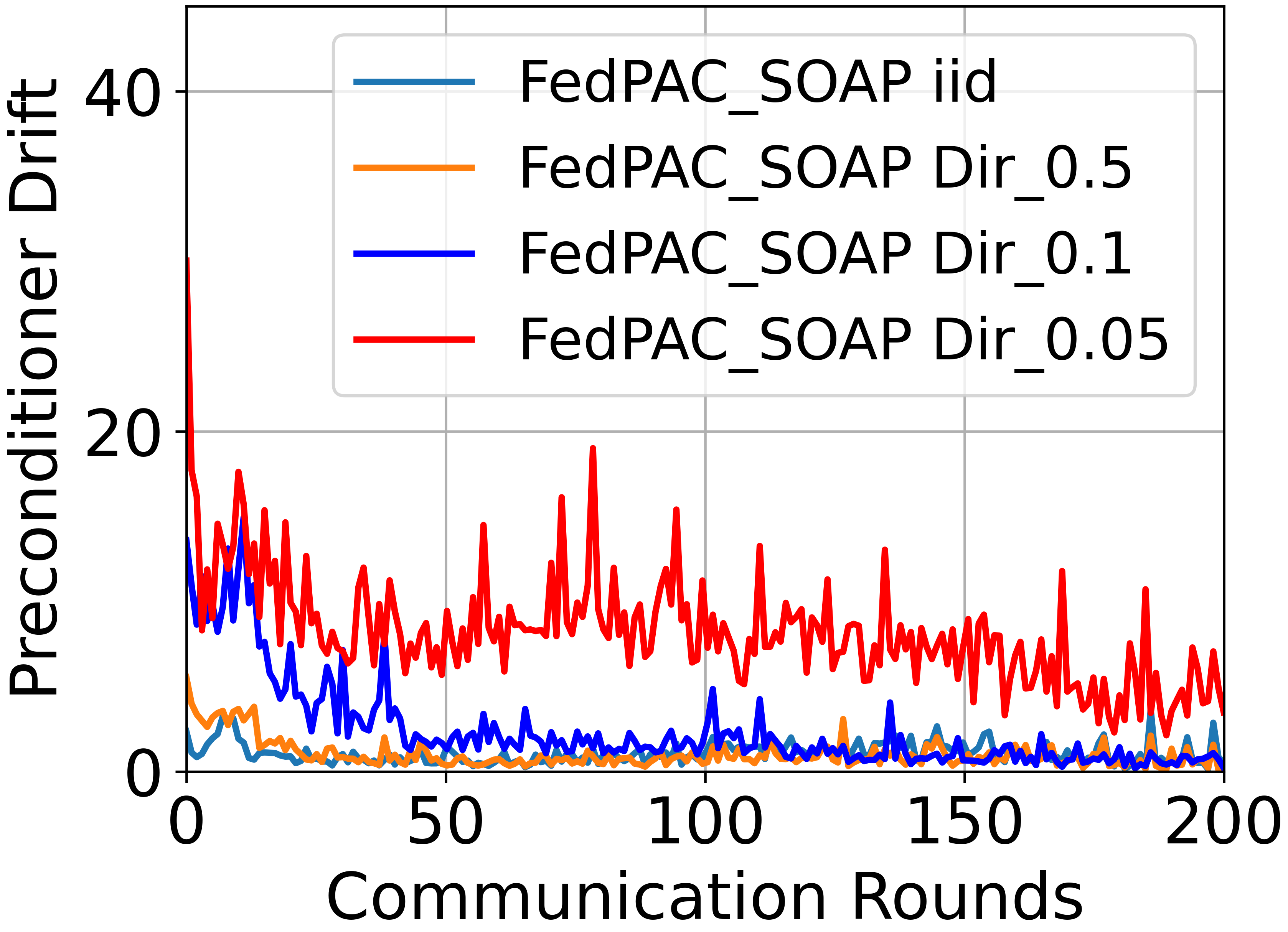}
		\caption{FedPAC\_SOAP}
		\label{fig:vit_iid}
	\end{subfigure}
	\begin{subfigure}[b]{0.245\textwidth}
		\centering
		\includegraphics[width=\linewidth]{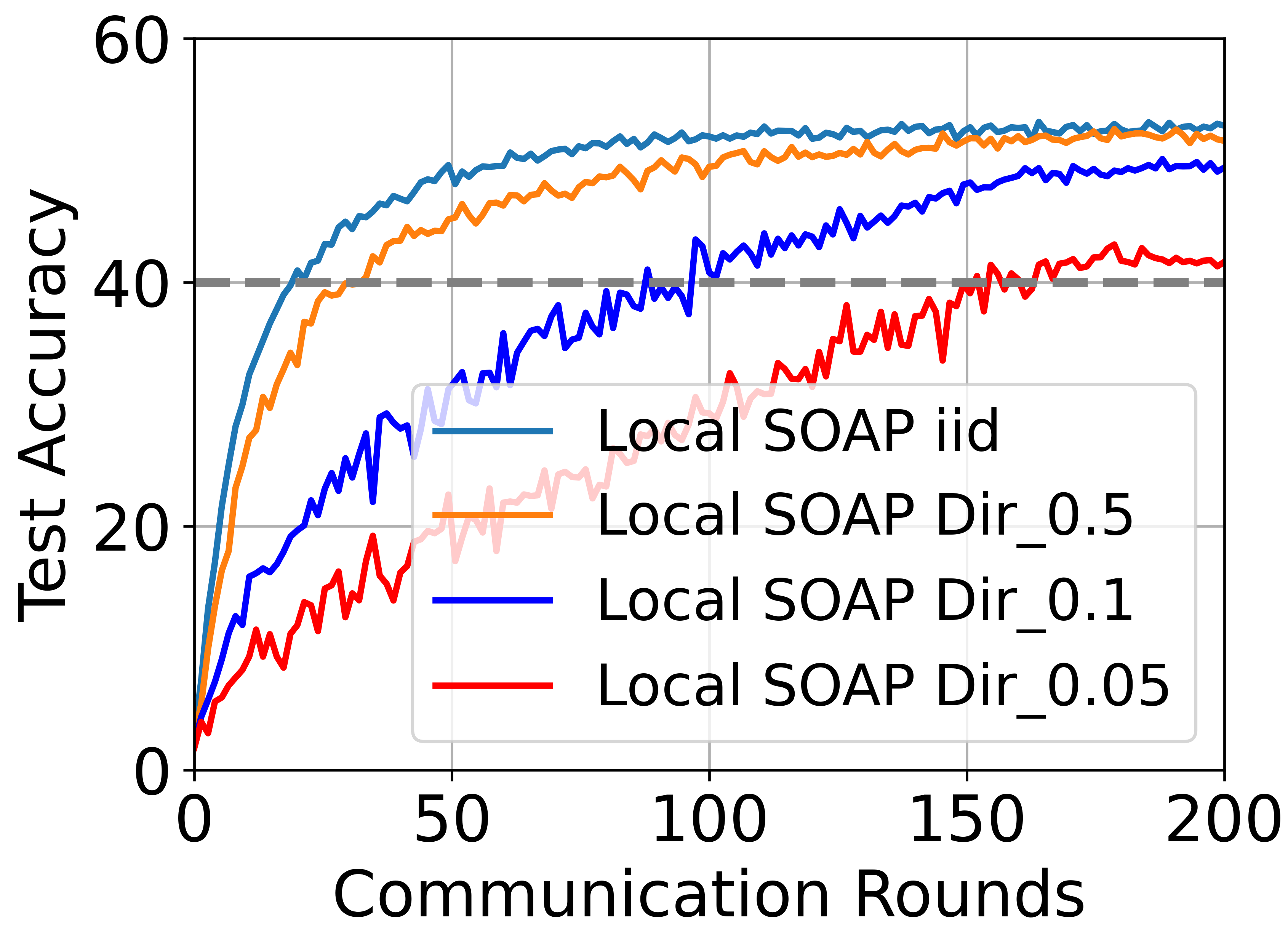}
		\caption{Local SOAP}
		\label{fig:resnet_iid}
	\end{subfigure}
	\hfill
	\begin{subfigure}[b]{0.245\textwidth}
		\centering
		\includegraphics[width=\linewidth]{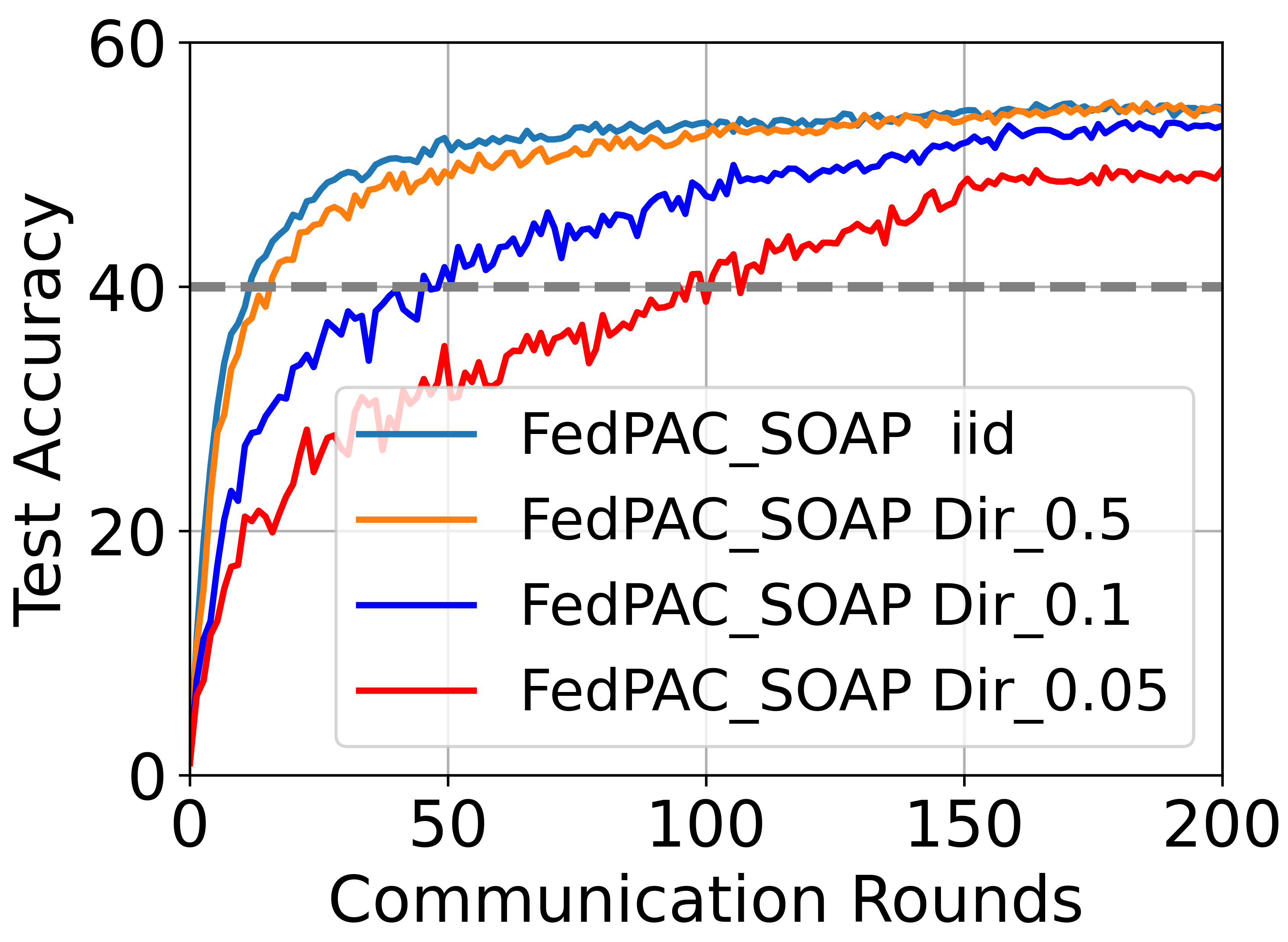}
		\caption{FedPAC\_SOAP}
		\label{fig:vit_iid}
	\end{subfigure}
	
	\caption{\small Performance and Preconditioner Drift  of Local SOAP,  \texttt{FedPAC\_SOAP}. For SOAP, we compute the preconditioner drift in a layer-wise manner by measuring the spectral norm of the difference between each client’s left/right preconditioner  and the aggregated global preconditioner. Our \texttt{FedPAC\_SOAP} substantially reduces the preconditioner drift of Local SOAP and accelerates convergence in (a,b), reaching 40\% test accuracy in fewer rounds, as shown in (c) and (d).
	}
	\label{fig:precondition drift}
\end{figure*}

\begin{algorithm}[tb]
	\small
	
	\caption{Federated Second-Order Algorithm (FedSOA)}
	\begin{algorithmic}[1]
		\REQUIRE  
		Per client, we maintain
		$\boldsymbol{\Theta}_i^{r,k}$ within a round.
		Hyperparameters: learning rate $\eta$, communication rounds $R$, the number of local updates  $K$, the number of client $N$.
		\FOR{$r = 0, \dots, R$}
		\FOR{each client $i \in \mathcal{S}_r$ in parallel}
		\STATE$\boldsymbol{\Theta} _i^{r,0} \gets \boldsymbol{0} $;
		\FOR{$k = 1, \dots, K$}
		\STATE Sample batch $B_i^{r,k};$
		\STATE $\boldsymbol{g}_i^{r,k} \in \mathbb{R}^{m\times n} \gets \nabla F_i(\boldsymbol{x}_i^{r,k}; \xi_i^{r,k});$ 
		\STATE$\boldsymbol{\Theta}_i^{r,k+1}
		=
		\mathrm{UpdateState}\bigl(\boldsymbol{\Theta}_i^{r,k}, \boldsymbol{g}_i^{r,k}\bigr)$; 
		\STATE$\tilde{\boldsymbol{g}}_i^{r,k}  \gets \mathcal{P}_{\boldsymbol{\Theta}_i^{r,k}}\bigl(\boldsymbol{g}_i^{r,k}\bigr)$;
		\STATE $\boldsymbol{x}^{r,k+1}_i\! =\!\boldsymbol{x}^{r,k}_i \!\!-\!  \eta_l \tilde{\boldsymbol{g}}_i^{r,k}$;
		\ENDFOR
		\STATE $\Delta \boldsymbol{x}_{i}^r \coloneqq \boldsymbol{x}_i^{r,K} - \boldsymbol{x}^r$; 
		\STATE Client $i$ communicate $(\boldsymbol{\Delta} \boldsymbol{x}_{i}^r)$ to Server;
		\ENDFOR
		\STATE$\boldsymbol{x}^{r+1} =\boldsymbol{x}^{r} +\frac{1}{|\mathcal{S}_r|} \sum_{i=1}^{|\mathcal{S}_r|} (\boldsymbol{x}^{r, K}_i-\boldsymbol{x}^{r, 0}_i)$;
		\ENDFOR
	\end{algorithmic}
	\label{algorithm_FedSOA}
\end{algorithm}
\subsection{Unified Abstraction of  Second-Order FL}
\vspace{-2mm}
To investigate the mechanisms of second-order optimizers in FL, we first propose  a generalized \emph{Federated Second-Order Algorithm} (FedSOA), as shown in Algorithm \ref{algorithm_FedSOA}.

At the $k$-th local step of round $r$, client $i$ samples a mini-batch
$\xi_i^{r,k} \sim \mathcal{D}_i$ and computes a stochastic gradient
\vspace{-2mm}
\begin{equation}
	\boldsymbol{g}_i^{r,k} = \nabla F_i(\boldsymbol{x}_i^{r,k}; \xi_i^{r,k}).
\end{equation}
Second-order optimizers maintain an internal \emph{preconditioner state}
$\boldsymbol{\Theta}_i^{r,k}$, which parameterizes a preconditioning operator that maps
the gradient to an update direction.
To unify different optimizers, we define a generic operator
$\mathcal{P}_{\boldsymbol{\Theta}}(\cdot)$ and write the local update as follows:
\vspace{-2mm}
\begin{equation}
	\label{eq:local-second-order-update}
	\boldsymbol{x}_i^{r,k+1}
	=
	\boldsymbol{x}_i^{r,k}
	- \eta_l \, \mathcal{P}_{\boldsymbol{\Theta}_i^{r,k}}\bigl(\boldsymbol{g}_i^{r,k}\bigr),
\end{equation}
where $\eta_l$ denotes the local learning rate.

The specific form of $\mathcal{P}_{\boldsymbol{\Theta}}$ depends on the underlying optimizer.
For example, SOAP apply structured preconditioning in a Kronecker-factored
basis, Sophia rescales gradients using diagonal Hessian estimates, and Muon updates
gradients in an orthogonalized rule.
After each step, the preconditioner state is updated according to the optimizer rule:
\vspace{-2mm}
\begin{equation}
	\boldsymbol{\Theta}_i^{r,k+1}
	=
	\mathrm{UpdateState}\bigl(\boldsymbol{\Theta}_i^{r,k}, \boldsymbol{g}_i^{r,k}\bigr).
\end{equation}
After $K$ local steps, client $i$ computes the model difference
$
	\Delta \boldsymbol{x}_{i}^r \coloneqq \boldsymbol{x}_i^{r,K} - \boldsymbol{x}^r .
$
A naive second-order federated baseline aggregates these updates in the same manner
as FedAvg:
\vspace{-2mm}
\begin{equation}
	\boldsymbol{x}^{r+1}
	=
	\boldsymbol{x}^r + \sum_{i \in \mathcal{S}_r} \frac{1}{|\mathcal{S}_r|} \,\Delta \boldsymbol{x}_{i}^r.
\end{equation}
However, under strong statistical heterogeneity and multiple local steps,
this naive aggregation leads to \emph{preconditioner drift}, which
significantly degrades global convergence.

\textbf{Three instantiations as $(\boldsymbol{\Theta},\mathcal{P})$ pairs.}\\
\textbf{$\bullet$} \textbf{SOAP:}	$\boldsymbol{\Theta}\!=\!\{L, R\},
	\mathcal{P}_{\boldsymbol{\Theta}}(g)\!\!=\!\!A\big(\beta_1 m+(1-\beta_1)g\big)B,
	A=\mathrm{InvRoot}(L+\epsilon I), B=\mathrm{InvRoot}(R+\epsilon I)$. The $\mathrm{InvRoot}(\cdot)$ denotes the inverse matrix \(p\)-th root operator. $L$ corresponds to second-order information on the left dimension. $R$ corresponds to second-order information on the right dimension. $m$ denotes the momentum.\\
	\textbf{$\bullet$} \textbf{Sophia:}
	$\boldsymbol{\Theta} = \{h\},
	\mathcal{P}_{\boldsymbol{\Theta}}(g)=\mathrm{clip}\!\left(\frac{\beta_1 m+(1-\beta_1)g}{h+\epsilon},\,\pm \rho\right).$ $h$ denotes the curvature (second-order information) estimate. $\mathrm{clip}(\cdot)$ denotes element-wise clipping.\\
	\textbf{$\bullet$} \textbf{Muon:}
	$\boldsymbol{\Theta}\!=\!\{m\},
	\mathcal{P}_{\boldsymbol{\Theta}}(g)=\mathrm{Ortho}\left(\beta_1 m+(1-\beta_1)g\right).$ $\mathrm{Ortho}(\cdot)$ denotes the orthogonalization operator.

\subsection{Local Preconditioner Drift}
To characterize the geometric deviation of second-order optimization in federated environments, we consider an idealized centralized second-order optimizer that constructs a global preconditioner $\Theta^{r}$ on the joint data distribution
$
\mathcal{D} = \frac{1}{N}\sum_{i=1}^N \mathcal{D}_i
$
and performs the update
\vspace{-2mm}
\begin{equation}
	\boldsymbol{x}^{r+1}
	=
	\boldsymbol{x}^r
	- \eta \, \mathcal{P}_{\Theta^{r}}\bigl(\nabla F(\boldsymbol{x}^r)\bigr).
\end{equation}
In this idealized and IID setting, the model always evolves under a unified second-order geometry characterized by $\boldsymbol{\Theta}^{r}$.

In the federated setting, each client $i$ locally updates its preconditioner state
$\Theta_i^{r,k}$ based on its own data distribution $\mathcal{D}_i$, which in turn induces a local preconditioner $\boldsymbol{\Theta}_i^{r,k}$. As $k$ and $r$ increase, these local preconditioners progressively \emph{drift} across clients. We formally define the \textbf{preconditioner drift}  $\boldsymbol{\Delta}_{\mathcal{D}}$, as the difference gap of the preconditioner drift between the global and local preconditioners.

\textbf{Definition 1 \textit{\textbf{(Preconditioner drift)}}.} \textit{	 $\boldsymbol{\Delta}_{\mathcal{D}}$ of the global preconditioner
	$\boldsymbol{\Theta}^{r+1}=\frac{1}{N} \sum_{i=1}^N \boldsymbol{\Theta}_{i}^{r,K}$ and the local preconditioners $\left\{\boldsymbol{\Theta}_{i}^{r,K}\right\}_{i=1}^N$ is defined as:}
\vspace{-2mm}
\begin{equation}
	\boldsymbol{\Delta}_{\mathcal{D}}=\frac{1}{N} \sum_{i=1}^N \mathbb{E}\left\|\boldsymbol{\Theta}_{i}^{r,K}-\boldsymbol{\Theta}^{r+1}\right\|^2.
\end{equation}
Preconditioner drift is a client-level phenomenon: different clients' preconditioners $\boldsymbol{\Theta}_i^{r,k}$ gradually adapt toward their respective local geometries. As shown in Figure~\ref{fig:precondition drift} (a,b), we observe that as data heterogeneity increases, the preconditioner drift $\boldsymbol{\Delta}_{\mathcal{D}}$ becomes more severe, and the algorithm converges increasingly slowly.

\begin{algorithm}[tb]
	\small
	\caption{Federated Preconditioner Alignment and Correction Framework (FedPAC)}
	
	\begin{algorithmic}[1]
		\REQUIRE  
		Per client, we maintain: 
		$\boldsymbol{\Theta}_i^{r,k}$ within a round.
		Hyperparameters: learning rate $\eta$, communication rounds $R$, the number of local updates  $K$, the number of client $N$.
		\FOR{$r = 0, \dots, R$}
		\FOR{each client $i \in \mathcal{S}_r$ in parallel}
		\STATE \fcolorbox{LightRed}{LightRed}{$\boldsymbol{\Theta}_i^{r,0} \gets \boldsymbol{\Theta}^{r};$}
		\FOR{$k = 1, \dots, K$}
		\STATE Sample batch $B_i^{r,k}$;
		\STATE $\boldsymbol{g}_i^{r,k} \in \mathbb{R}^{m\times n} \gets \nabla F_i(\boldsymbol{x}_i^{r,k}; \xi_i^{r,k})$; 
		\STATE$\boldsymbol{\Theta}_i^{r,k+1}
		=
		\mathrm{UpdateState}\bigl(\boldsymbol{\Theta}_i^{r,k}, \boldsymbol{g}_i^{r,k}\bigr);$
		\STATE$\tilde{\boldsymbol{g}}_i^{r,k}  \gets \mathcal{P}_{\boldsymbol{\Theta}_i^{r,k}}\bigl(\boldsymbol{g}_i^{r,k}\bigr);$
		\STATE \fcolorbox{LightBlue}{LightBlue}{$\boldsymbol{x}^{r,k+1}_i\! =\!\boldsymbol{x}^{r,k}_i \!\!-\!  \eta_l [(1\!-\!\beta)\tilde{\boldsymbol{g}}_i^{r,k}\!+\!\beta\boldsymbol{g}_G^r ]$;}
		\ENDFOR
		\STATE $\Delta \boldsymbol{x}_{i}^r \coloneqq \boldsymbol{x}_i^{r,K} - \boldsymbol{x}^r$; 
		\STATE Client $i$ communicate $(\boldsymbol{\Delta} \boldsymbol{x}_{i}^r,\boldsymbol{\Theta}_i^{r,K})$ to Server;
		\ENDFOR
		\STATE \fcolorbox{LightBlue}{LightBlue}{$\boldsymbol{g}_G^{r+1}=-\frac{1}{SK\eta} \sum_{i=1}^S (\boldsymbol{x}^{r, K}_i-\boldsymbol{x}^{r, 0}_i)$;}
		\STATE$\boldsymbol{x}^{r+1} =\boldsymbol{x}^{r} +\frac{1}{S} \sum_{i=1}^S (\boldsymbol{x}^{r, K}_i-\boldsymbol{x}^{r, 0}_i)$;
		\STATE\fcolorbox{LightRed}{LightRed}{$	\boldsymbol{\Theta}^{r+1}\;=\;
			\frac{1}{|\mathcal{S}_r|}
			\sum_{i \in \mathcal{S}_r}
			\boldsymbol{\Theta}_i^{r,K};$}
		\STATE Server broadcasts $(\boldsymbol{x}^{r+1}, \boldsymbol{\Theta}_{r+1},\boldsymbol{g}_G^{r+1});$
		\ENDFOR
	\end{algorithmic}
	\label{algorithm_FedPAC}
\end{algorithm}

\section{Our Algorithm: A Federated Preconditioner Alignment and Correction Framework}

To address these issues, we propose a federated preconditioner alignment and correction framework (\texttt{FedPAC}) for general second-order optimizers, consisting of two stages.
\begin{figure}[tb]
	\includegraphics[width=0.48\textwidth]{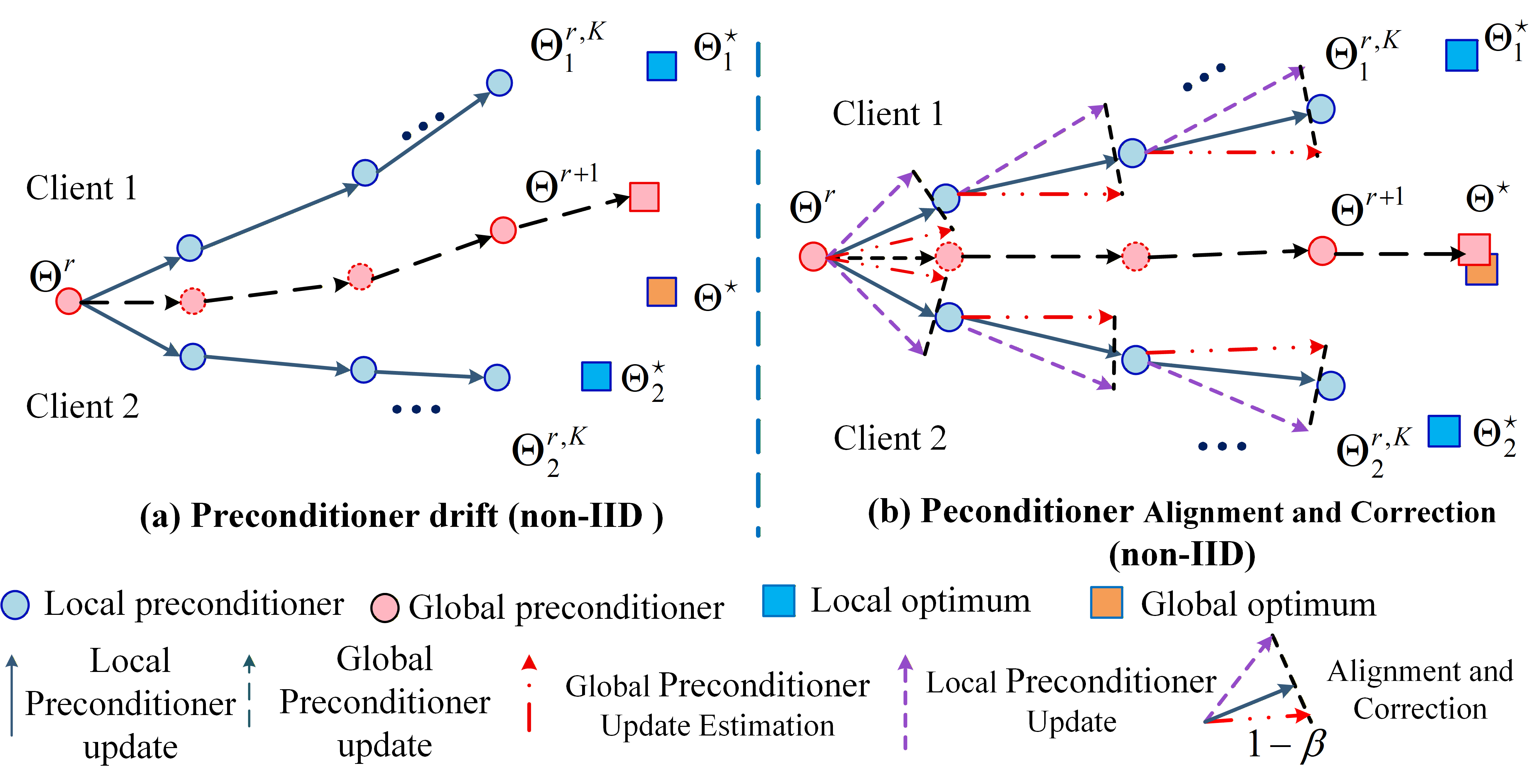}
	\vspace{-2mm}
	\caption{\small{An illustration of  preconditioner drift, which corrects client drift  through local-global alignment. }}
	\label{predictioner_drift}
\end{figure}

\paragraph{(1) Aggregation of Preconditioner States (Alignment)}
At the end of round $r$, client $i$ uploads $\Delta \boldsymbol{x}_i^r$ and its local preconditioner state $\boldsymbol{\Theta}_i^{r,K}$; the server then aggregates the preconditioners in  Algorithm \ref{algorithm_FedPAC}  (i.e., Lines  3 \& 16):
\vspace{-2mm}
\begin{equation}
	\boldsymbol{\Theta}^{r+1}
	\;=\;
	\frac{1}{|\mathcal{S}_r|}
	\sum_{i \in \mathcal{S}_r}
	\boldsymbol{\Theta}_i^{r,K}.
\end{equation}
At the beginning of round $r\!+\!1$, the server broadcasts $(\boldsymbol{x}^r, \boldsymbol{\Theta}^{r})$ to the participating clients.
Upon receiving them, client $i$ aligns its local preconditioner state with the global reference $\boldsymbol{\Theta}^{r}$ to obtain the aligned
preconditioner used for the new round:
$\boldsymbol{\Theta}_i^{r,0} \gets \boldsymbol{\Theta}^{r} $.
The client then uses $\boldsymbol{\Theta}_i^{r,0}$ as the initial preconditioner state for round $r\!+\!1$ and performs
the local second-order updates according to~\eqref{eq:local-second-order-update}. This mechanism preserves local second-order adaptivity while periodically pulling the optimization geometry back toward a shared
reference.

\paragraph{(2) Local Preconditioner Correction (Correction)}
On each client, we correct the preconditioner drift, as follows:
\begin{equation}
	\boldsymbol{x}^{r,k+1}_i\! =\!\boldsymbol{x}^{r,k}_i \!\!-\! \eta_l [(1\!-\!\beta)\tilde{\boldsymbol{g}}_i^{r,k}\!+\!\beta\boldsymbol{g}_G^r ],
\end{equation}
where $\boldsymbol{g}_G^r\!=\!-\frac{1}{SK\eta}\sum_{i=1}^S\big(\boldsymbol{x}_i^{r-1,K}\!-\!\boldsymbol{x}_i^{r-1,0}\big)$ is the estimated global update, obtained by averaging the preconditioned-gradient global change from the previous round. $\beta$ is the trade-off coefficient between local and global updates in Figure \ref{predictioner_drift} and Algorithm \ref{algorithm_FedPAC} (Lines 9 \& 14).

Overall, the above steps constitute a unified correction and acceleration framework for second-order
federated optimization. The underlying local optimizer can be any structured second-order method, such as
Sophia, Muon, SOAP. In this work, we instantiate our framework with three representative second-order optimizers, yielding \texttt{FedPAC\_Sophia}, \texttt{FedPAC\_Muon}, and \texttt{FedPAC\_SOAP}, respectively. 

\begin{figure*}[tb]
	\centering
	\begin{subfigure}[b]{0.245\textwidth}
		\centering
		\includegraphics[width=\linewidth]{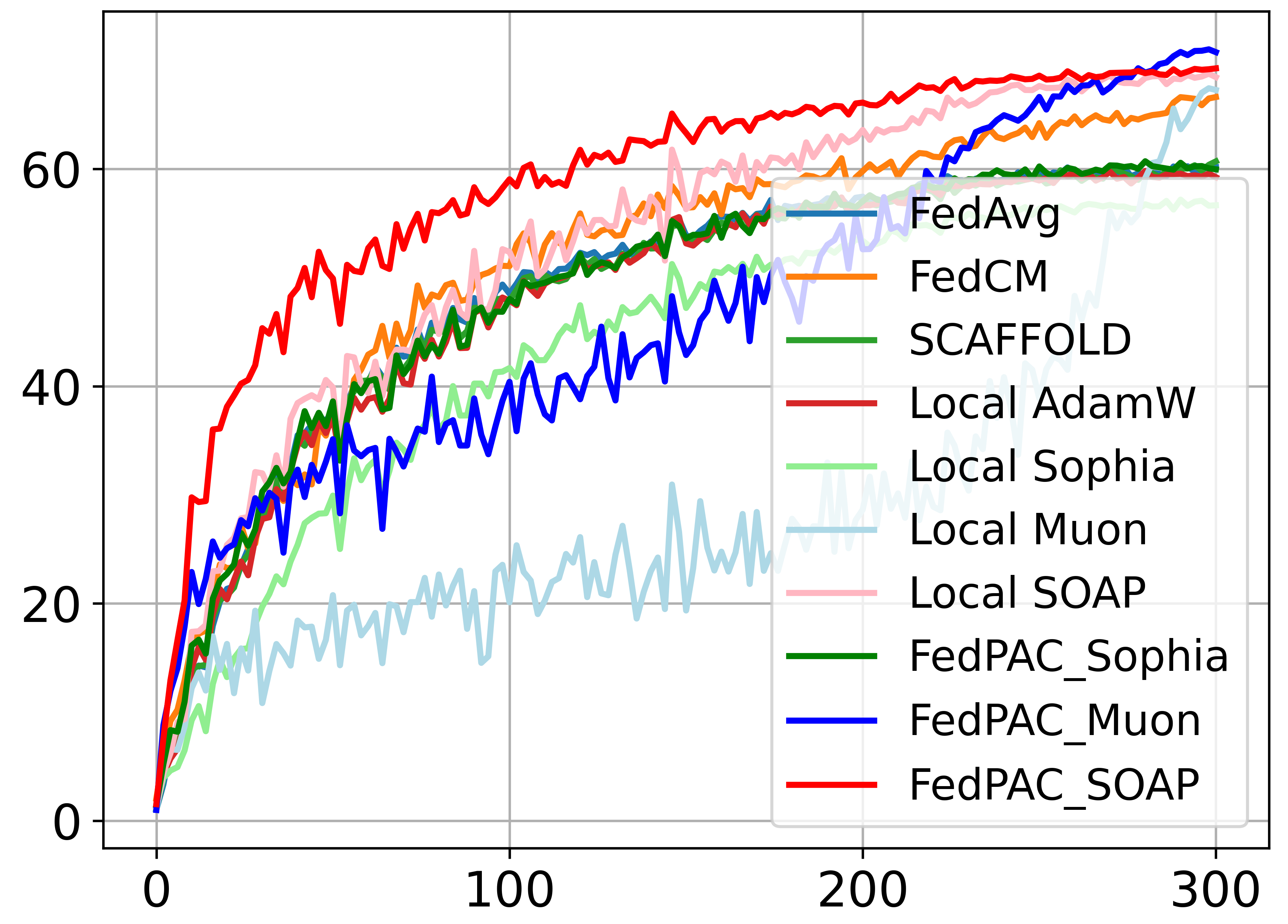}
		\caption{CIFAR-100,Dir-0.1}
		\label{fig:resnet_iid}
	\end{subfigure}
	\hfill
	\begin{subfigure}[b]{0.245\textwidth}
		\centering
		\includegraphics[width=\linewidth]{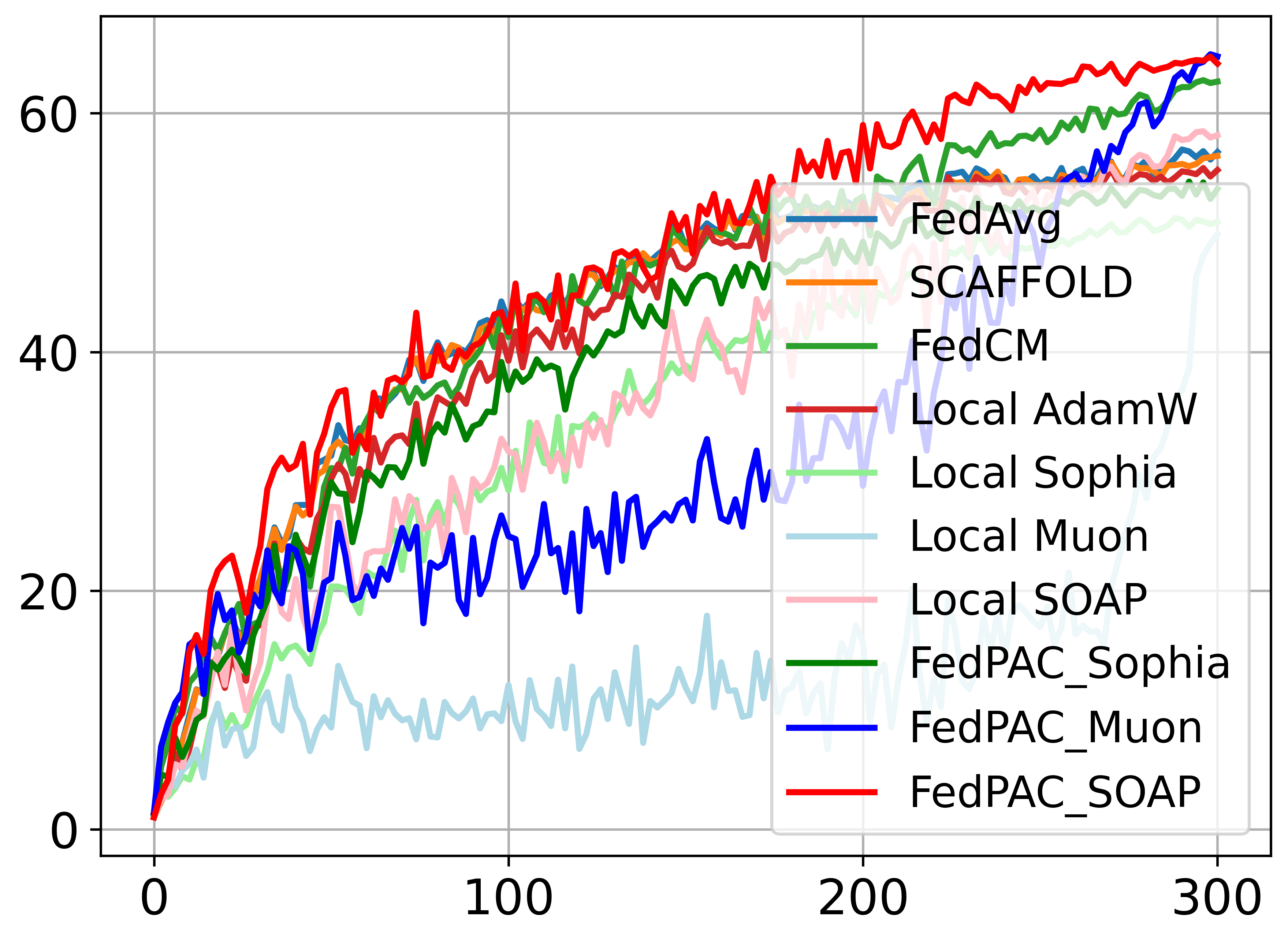}
		\caption{CIFAR-100,Dir-0.05}
		\label{fig:vit_iid}
	\end{subfigure}
	\begin{subfigure}[b]{0.245\textwidth}
		\centering
		\includegraphics[width=\linewidth]{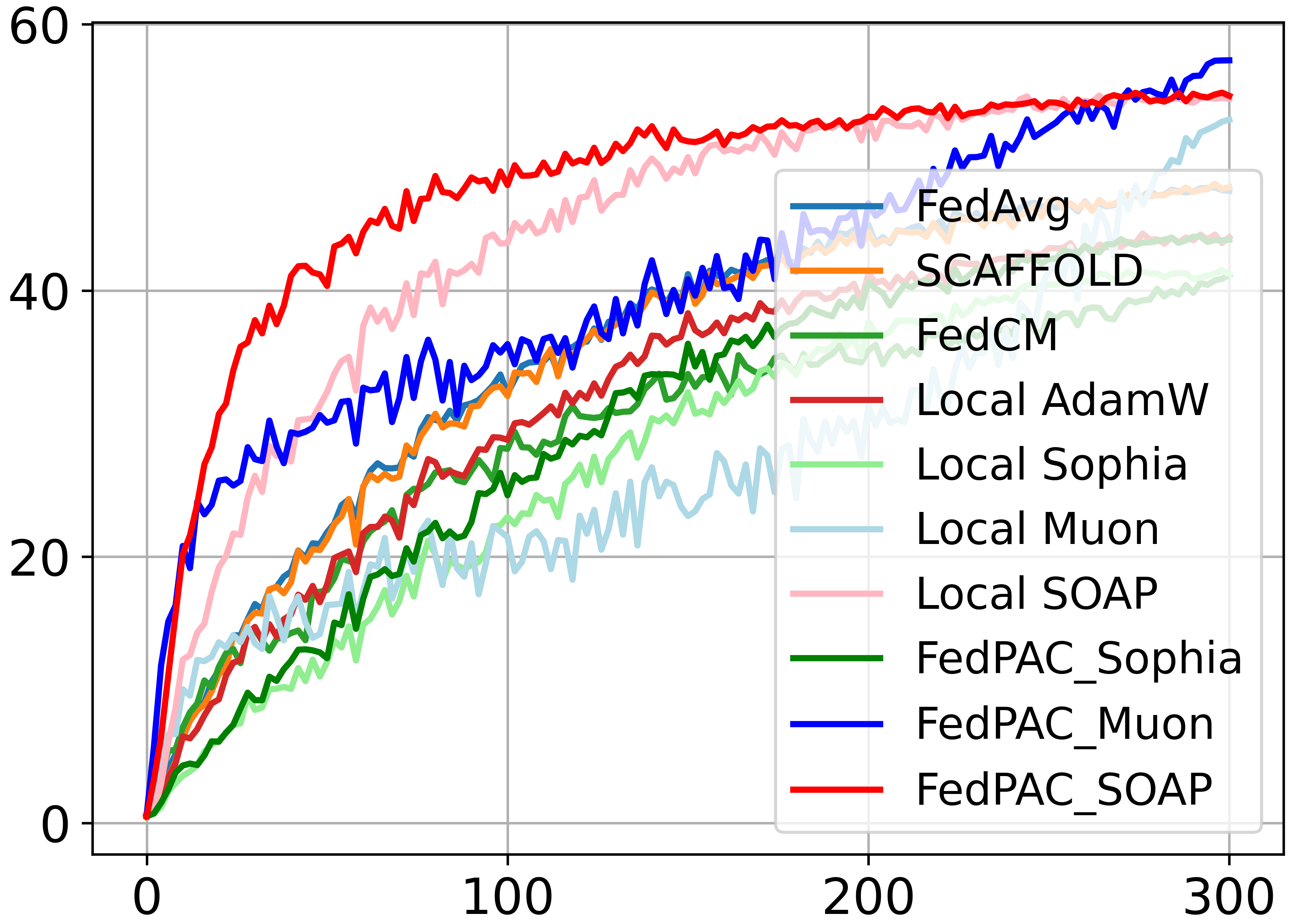}
		\caption{Tiny-ImageNet,Dir-0.1}
		\label{fig:resnet_iid}
	\end{subfigure}
	\hfill
	\begin{subfigure}[b]{0.245\textwidth}
		\centering
		\includegraphics[width=\linewidth]{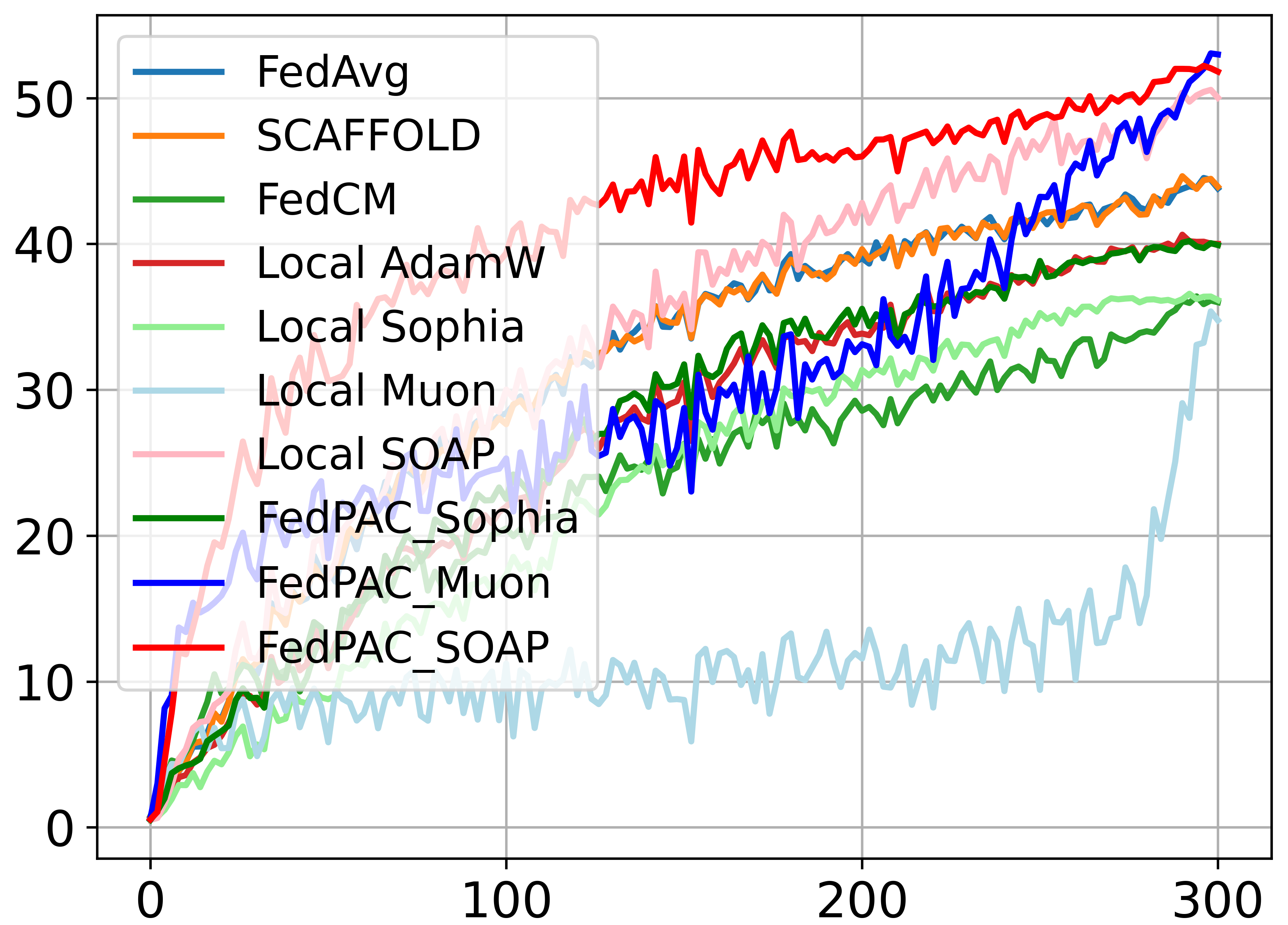}
		\caption{Tiny-ImageNet,Dir-0.05}
		\label{fig:vit_iid}
	\end{subfigure}
	\begin{subfigure}[b]{0.245\textwidth}
		\centering
		\includegraphics[width=\linewidth]{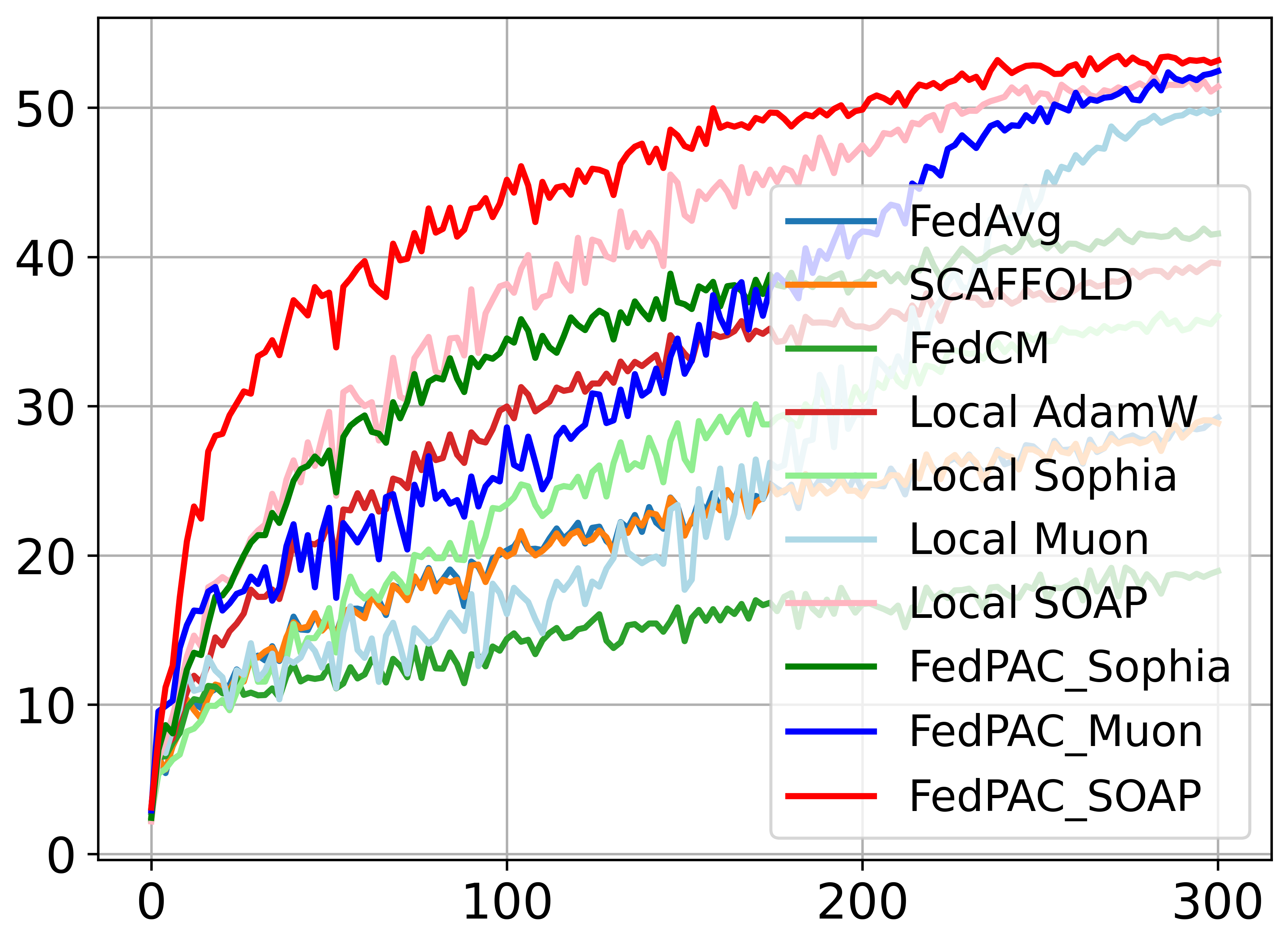}
		\caption{CIFAR-100,Dir-0.1}
		\label{fig:resnet_iid}
	\end{subfigure}
	\hfill
	\begin{subfigure}[b]{0.245\textwidth}
		\centering
		\includegraphics[width=\linewidth]{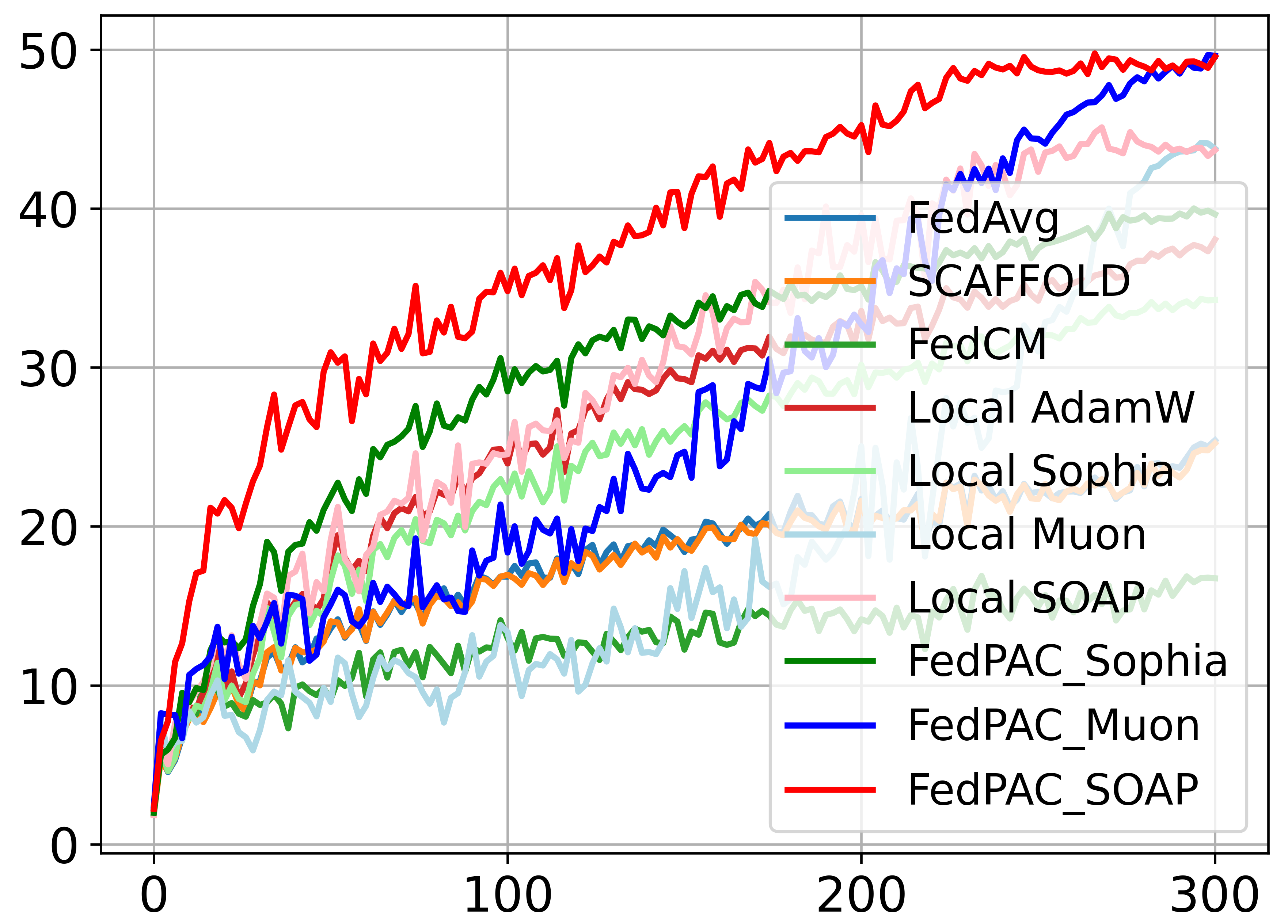}
		\caption{CIFAR-100,Dir-0.05}
		\label{fig:vit_iid}
	\end{subfigure}
	\begin{subfigure}[b]{0.245\textwidth}
		\centering
		\includegraphics[width=\linewidth]{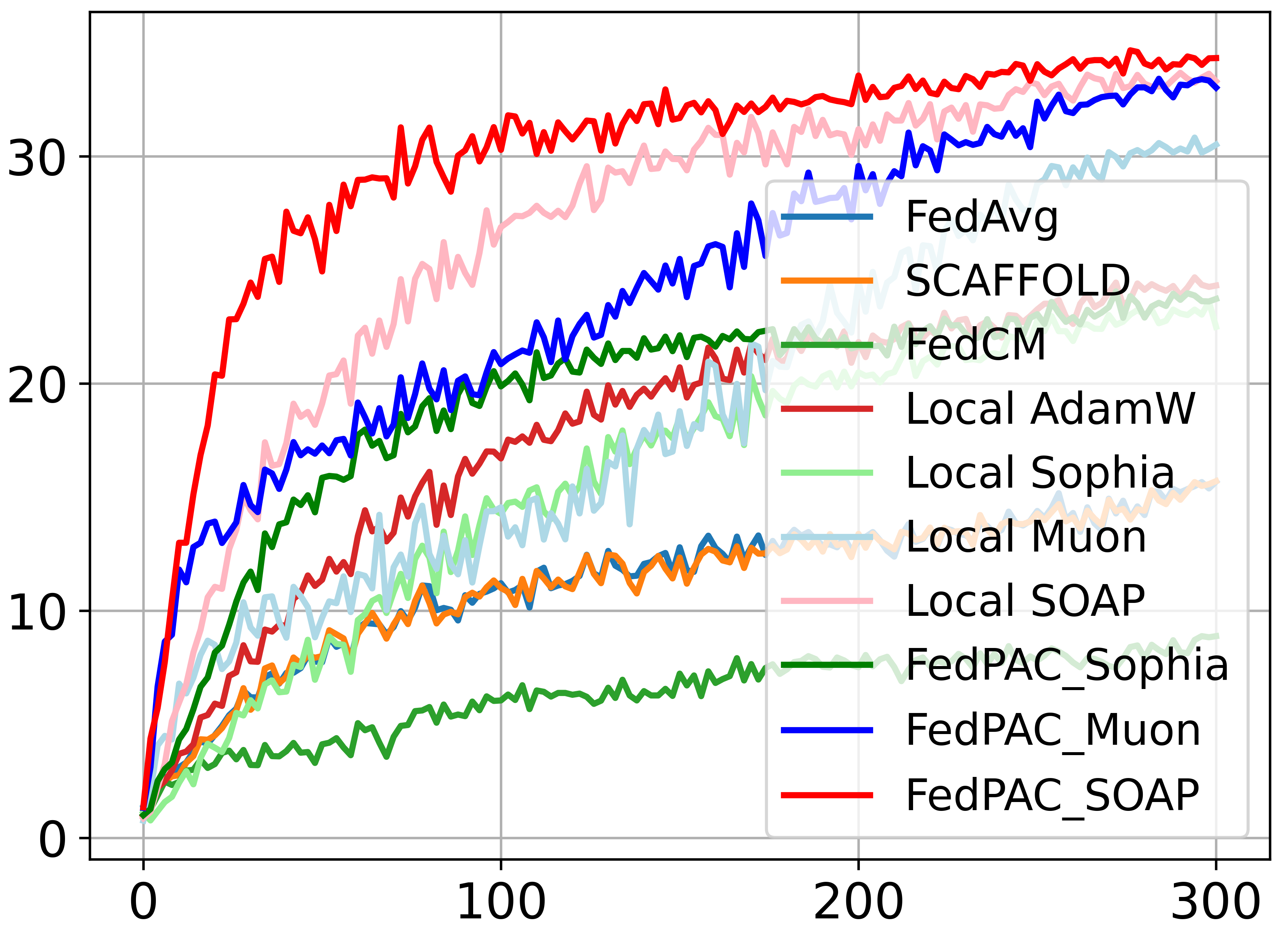}
		\caption{Tiny-ImageNet,Dir-0.1}
		\label{fig:resnet_iid}
	\end{subfigure}
	\hfill
	\begin{subfigure}[b]{0.245\textwidth}
		\centering
		\includegraphics[width=\linewidth]{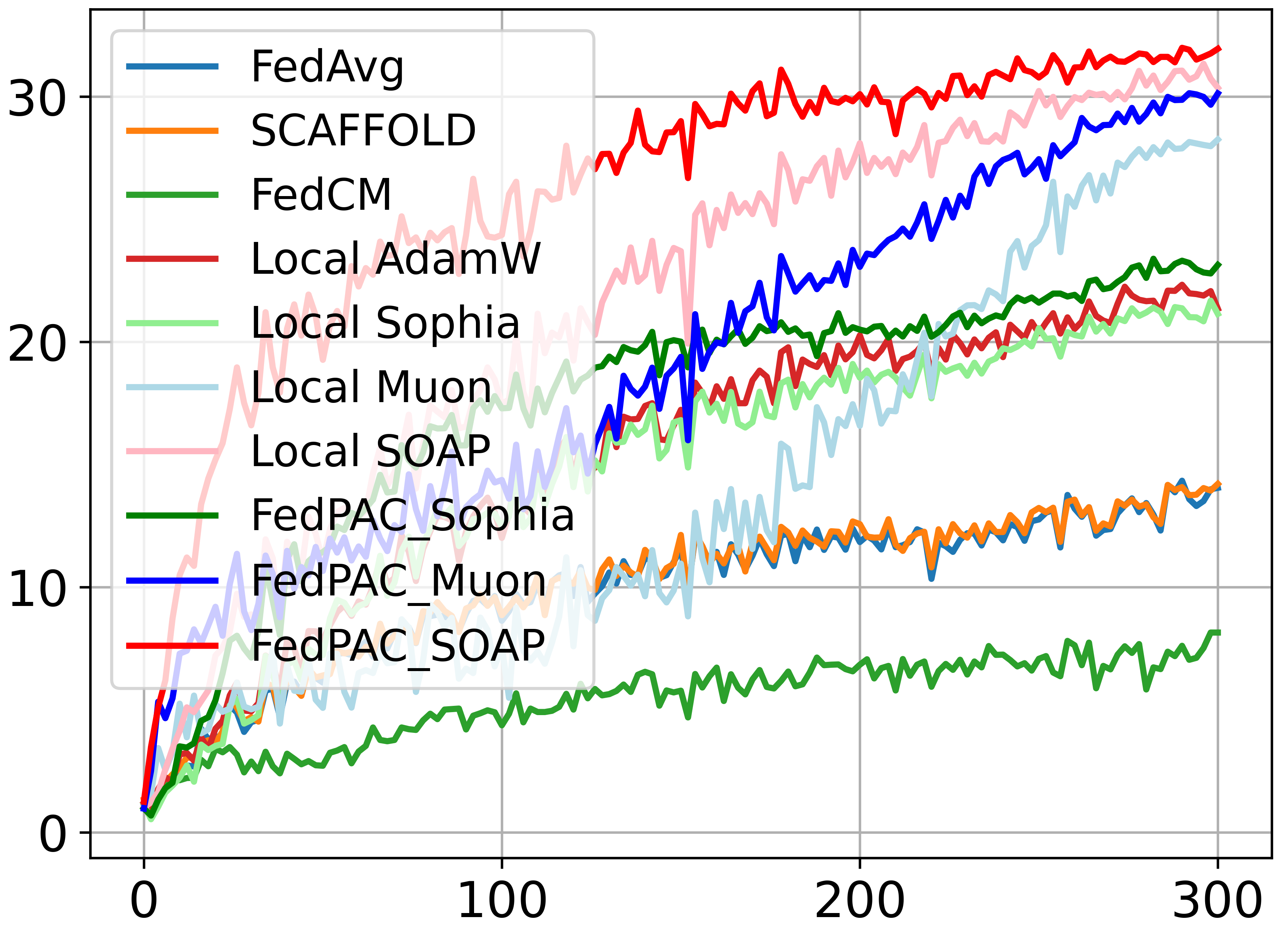}
		\caption{Tiny-ImageNet,Dir-0.05}
		\label{fig:vit_iid}
	\end{subfigure}
	\vspace{-2mm}
	\caption{\small \textbf{Test accuracy versus communication rounds.} The x-axis denotes communication rounds and the y-axis denotes the test accuracy. Panels (a--d) correspond to training with \textbf{ResNet-18}, while panels (e--h) correspond to training with \textbf{ViT-Tiny}. }
	\label{fig:compare_resnet}
\end{figure*}

\begin{table*}[tb]
	\centering
	\caption{\small Test accuracy of each method on CIFAR-100 and Tiny-Imagenet using \textbf{ResNet-18} and \textbf{ViT-Tiny} over 300 communication rounds under Dir-0.1 and Dir-0.05 (100 clients, 10\% participation, batch size 50, $K=50$).}
	\vspace{-2mm}
	\label{tab:combined_dir01_dir005}
	\setlength{\tabcolsep}{2pt}
	\small
	\begin{tabular}{lllllllllll}
		\toprule
		\multirow{3}{*}{\textbf{Method}}
		& \multicolumn{4}{c}{\textbf{ResNet-18}}
		& \multicolumn{4}{c}{\textbf{ViT-Tiny}} \\
		\cmidrule(lr){2-5} \cmidrule(lr){6-9}
		& \multicolumn{2}{c}{\textbf{CIFAR-100}}
		& \multicolumn{2}{c}{\textbf{Tiny-Imagenet}}
		& \multicolumn{2}{c}{\textbf{CIFAR-100}}
		& \multicolumn{2}{c}{\textbf{Tiny-Imagenet}} \\
		\cmidrule(lr){2-3} \cmidrule(lr){4-5} \cmidrule(lr){6-7} \cmidrule(lr){8-9}
		& Dir-0.1 & Dir-0.05 & Dir-0.1 & Dir-0.05
		& Dir-0.1 & Dir-0.05 & Dir-0.1 & Dir-0.05 \\
		\midrule
		FedAvg     & 60.17 & 56.75 & 47.48 & 43.80 & 27.24 & 23.42 & 15.68 & 14.05 \\
		SCAFFOLD   & 60.69 & 56.43 & 47.76 & 43.92 & 26.86 & 23.23 & 15.70 & 14.21 \\
		FedCM      & 66.61 & 62.65 & 41.16 & 36.00 & 16.95 & 14.74 &  8.88 &  8.15 \\
		Local AdamW& 59.23 & 55.24 & 44.01 & 40.00 & 37.57 & 36.06 & 24.31 & 21.35 \\
		\cdashline{1-9}  
		Local Sophia & 56.65 & 50.89 & 41.23 & 36.15 & 34.05 & 32.25 & 22.49 & 21.14 \\
		\rowcolor{LightMint}
		\texttt{FedPAC\_Sophia} &
		\inc{59.96}{3.31} & \inc{53.66}{2.77} & \inc{43.81}{2.58} & \inc{36.37}{0.22} &
		\inc{39.79}{5.74} & \inc{32.71}{0.46} & \inc{23.01}{0.52} & \inc{22.37}{1.23} \\
		\cdashline{1-9}  
		Local Muon & 67.26 & 49.86 & 52.83 & 34.76 &  44.00 & 39.68 & 30.51 & 28.25 \\
		\rowcolor{LightBlue}
		\texttt{FedPAC\_Muon} &
		\inc{\textbf{71.85}}{4.59} & \inc{\textbf{65.56}}{15.70} & \inc{\textbf{57.95}}{5.12} & \inc{\textbf{54.00}}{19.24} &
		\inc{47.81}{3.81} & \inc{41.76}{2.08} & \inc{31.45}{0.94} & \inc{30.25}{2.00} \\
		\cdashline{1-9}  
		Local SOAP & 68.44 & 58.16 & 54.42 & 50.02 & 49.41 & 41.68 & 33.30 & 30.36 \\
		\rowcolor{LightRed}
		\texttt{FedPAC\_SOAP} &
		\inc{69.25}{0.81} & \inc{64.16}{6.00} & \inc{55.62}{1.20} & \inc{51.81}{1.79} &
		\inc{\textbf{51.16}}{1.75} & \inc{\textbf{47.55}}{5.87} & \inc{\textbf{34.32}}{1.02} & \inc{\textbf{31.33}}{0.97} \\
		\bottomrule
	\end{tabular}
\end{table*}

\section{Theoretical Analysis}

\label{convergence_analysis}
In this section, we give the convergence theoretical analysis of our proposed \texttt{FedPAC} algorithm.  Firstly, we state some standard assumptions for the non-convex function $f$.
\begin{assumption}[Smoothness]
	\label{smoothness}
	\textit{The non-convex $f_{i}$ is a $L$-smooth function for all $i\in[N]$, i.e., $\Vert\nabla f_{i}(\boldsymbol{x})-\nabla f_{i}(\boldsymbol{y})\Vert\leq L\Vert\boldsymbol{x}-\boldsymbol{y}\Vert$, for all $\boldsymbol{x},\boldsymbol{y}\in\mathbb{R}^{d}$.}
\end{assumption}
\begin{assumption}[Bounded Stochastic Gradient]
	\label{bounded_stochastic_gradient_I}
	\textit{$\boldsymbol{g}_{i}^{r}=\nabla f_{i}(\boldsymbol{x}_{i}^{r}, \xi_i^{r})$ computed by using a sampled mini-batch data $\xi_i^{r}$ in the local client $i$ is an unbiased estimator of $\nabla f_{i}$ with bounded variance, i.e., $\mathbb{E}_{\xi_i^{r}}[\boldsymbol{g}_{i}^{r}]=\nabla f_{i}(\boldsymbol{x}_{i}^{r})$ and     $\mathbb{E}_{\xi_i^{r}}\Vert \boldsymbol{g}_{i}^{r} - \nabla f_{i}(\boldsymbol{x}_{i}^{r})\Vert^{2} \leq \sigma_{l}^{2}$, for all $\boldsymbol{x}_{i}^{r}\in\mathbb{R}^{d}$.}
\end{assumption}

\begin{assumption}[Bounded Heterogeneity]
	\label{bounded_heterogeneity}
	\textit{The dissimilarity between local clients is bounded on the gradients, i.e., $\Vert\nabla f_{i}(\boldsymbol{x})-\nabla f(\boldsymbol{x})\Vert^{2}\leq\sigma_{g}^{2}$, for all $\boldsymbol{x}\in\mathbb{R}^{d}$.}
\end{assumption}

\begin{assumption}[Preconditioner Coercivity and Boundedness]
	\label{precond_boundedness}
	\textit{For any preconditioner state $\Theta$ and any vector $\boldsymbol{v}\in\mathbb{R}^d$, the
		preconditioned mapping $P_{\Theta}(\cdot)$ satisfies the following two properties:
		(i) (\emph{coercivity}) $\langle \boldsymbol{v}, P_{\Theta}(\boldsymbol{v})\rangle \ge \mu \|\boldsymbol{v}\|^{2}$ for some constant $\mu>0$;
		(ii) (\emph{boundedness}) $\|P_{\Theta}(\boldsymbol{v})\| \le M\|\boldsymbol{v}\|$ for some constant $M>0$.}
\end{assumption}

\begin{assumption}[Lipschitz Continuity in Preconditioner State]
	\label{precond_lipschitz_state}
	\textit{There exists a constant $L_{\boldsymbol{\Theta}}>0$ such that for any two preconditioner states $\boldsymbol{\Theta},\boldsymbol{\Theta}'$ and any vector $\boldsymbol{v}\in\mathbb{R}^d$,
		\vspace{-2mm}
		\begin{equation*}
			\|P_{\boldsymbol{\Theta}}(\boldsymbol{v})-P_{\boldsymbol{\Theta}'}(\boldsymbol{v})\|
			\le L_{\boldsymbol{\Theta}}\|\boldsymbol{\Theta}-\boldsymbol{\Theta}'\|\cdot \|\boldsymbol{v}\|.
		\end{equation*}
	}
\end{assumption}
	\vspace{-2mm}
Our convergence theory holds for the class of state-dependent preconditioned mappings characterized by Assumptions~\ref{precond_boundedness} and \ref{precond_lipschitz_state} (spectrally bounded and Lipschitz in the preconditioner state), which is reasonably satisfied by SOAP/Muon/Sophia.


\begin{theorem}[Convergence of  \texttt{FedSOA} for non-convex functions ]
	\label{theorem_convergence_fedsoa}
	Under Assumptions \ref{smoothness}, \ref{bounded_stochastic_gradient_I}, \ref{bounded_heterogeneity},
	\ref{precond_boundedness}, \ref{precond_lipschitz_state},
$\exists\,G^2>0\ \text{s.t.}\ \sup_{r,i,k}\mathbb{E}\|g_i^{r,k}\|^2\le G^2,$
	if we take $g^0=0$ and choose the stepsize $\eta$ such that
	$\eta \;\le\; \min\left\{\frac{\mu}{4LM^2K},\ \frac{1}{2LMK}\right\},
		\label{eq:fedsoa_stepsize_cond}$
	then \texttt{FedSOA} satisfies
	\vspace{-2mm}
	\begin{equation*}
		\frac{1}{R}\sum_{r=0}^{R-1}\mathbb{E}\big\|\nabla f(\boldsymbol{x}^r)\big\|^2
		\!\lesssim\!
		\frac{L\Delta}{R}
		+
		\sqrt{\frac{L\Delta}{R}\!\cdot\!
			\frac{\sigma_l^2+\sigma_g^2 + S\kappa_{\Theta}\,\bar\Delta_D}{SK}},
		\label{eq:fedsoa_rate_final}
	\end{equation*}
	where $\Delta=f(\boldsymbol{x}^0)-f^\star$, and
	$\bar\Delta_D=\frac{1}{R}\sum_{r=0}^{R-1}\Delta_D^r,
		\Delta_D^r=\frac{1}{S}\sum_{i\in S_r}\mathbb{E}\big\|\Theta_i^{r,K}-\bar\Theta^{r,K}\big\|^2.$
	and the drift-to-noise coupling constant is
	$\kappa_{\Theta} \;:=\; L_{\Theta}^2\,G^2.
		\label{eq:kappa_theta_def}$
	Here $S$ is the number of participating clients per round, $K$ is the number of local iterations,
	and $R$ is the total number of communication rounds. The notation $\lesssim$ hides absolute constants
	depending only on $(L,\mu,M)$.
\end{theorem}

\begin{theorem}[Convergence of \texttt{FedPAC} for non-convex functions]
	\label{theorem_convergence_fedpac}
	Under Assumptions \ref{smoothness}, \ref{bounded_stochastic_gradient_I},
	\ref{precond_boundedness}, and \ref{precond_lipschitz_state}, and the following \emph{second-moment bound}
	$\exists\,G^2>0\ \text{s.t.}\ \sup_{r,i,k}\mathbb{E}\|g_i^{r,k}\|^2\le G^2,$
	if we take $g^0=0$ and choose parameters such that
	$\eta \;\le\; \min\left\{\frac{\mu}{4LM^2K},\ \frac{1}{2LMK}\right\},
		\label{eq:fedpac_stepsize_cond}$
	then \texttt{FedPAC} satisfies
		\vspace{-2mm}
	\begin{equation*}
		\frac{1}{R}\sum_{r=0}^{R-1}\mathbb{E}\big\|\nabla f(\boldsymbol{x}^r)\big\|^2
		\;\lesssim\;
		\frac{L\Delta}{R}
		+\sqrt{\frac{L\Delta\,\sigma_l^2+\kappa_{\Theta}\cdot \bar\Delta_D}{SKR}}
		.
		\label{eq:fedpac_rate_final}
	\end{equation*}
\end{theorem}

\paragraph{(1) What limits second-order FL?}
Theorem~\ref{theorem_convergence_fedsoa} establishes the non-convex convergence of \texttt{FedSOA} and, more importantly, makes the \emph{preconditioner drift} explicit: beyond the standard optimization term $\frac{L\Delta}{R}$ and the stochastic/heterogeneity noise $\frac{\sigma_l^2+\sigma_g^2}{SK}$, an additional penalty $S\kappa_\Theta\bar\Delta_D$ enters the effective noise. This shows that inconsistent client-side preconditioner states can dominate convergence even when gradients are well-estimated.


\paragraph{(2) How does \texttt{FedPAC} close the gap?}

Theorem~\ref{theorem_convergence_fedpac} shows that \texttt{FedPAC} eliminates the explicit heterogeneity term $\sigma_g^2$ by aligning/correcting preconditioner dynamics, leaving only the drift-dependent remainder $\bar\Delta_D$. When $\bar\Delta_D$ is small, \texttt{FedPAC} achieves near-ideal scaling dominated by $\sigma_l^2/(SK)$, making drift measurable and controllable in practice.

\section{Experiments}
\vspace{-2mm}
\textbf{Datasets.} We evaluate \texttt{FedPAC} on both vision and language tasks. (\textit{i}) For image classification, we use CIFAR-100~\cite{krizhevsky2009learning}, and Tiny-ImageNet~\cite{le2015tiny}. (\textit{ii}) For NLP tasks, we adopt C4 \cite{raffel2020exploring} dataset. To simulate data heterogeneity across clients, we follow the Dirichlet partitioning scheme~\cite{hsu2019measuring}. For Dir-$\alpha$ partitioning, smaller $\alpha$ indicates more severe data heterogeneity.\\
\textbf{Model Architectures.} We explore a variety of model types: (\textit{i}) ResNet-18~\cite{he2016deep} as a representative convolutional neural network (CNN), (\textit{ii}) Vision Transformer (ViT-Base) and ViT-Tiny~\cite{dosovitskiy2020image} for Vision Transformers, and (\textit{iii}) LLaMA \cite{touvron2023llama} for large-scale language model.\\
\textbf{Baselines.} We compare our method against  state-of-the-art FL algorithms: \texttt{FedAvg} (\texttt{Local SGD})~\cite{mcmahan2017communication}, \texttt{SCAFFOLD}~\cite{karimireddy2020scaffold}, \texttt{FedCM}~\cite{xu2021fedcm}, \texttt{Local AdamW}, \texttt{Local Sophia}, \texttt{Local Muon} and \texttt{Local SOAP}. Under our \texttt{FedPAC} framework, we instantiate three variants: \texttt{FedPAC\_Sophia}, \texttt{FedPAC\_Muon}, and \texttt{FedPAC\_SOAP}.
 In the Appendix (Table), we compare additional FL algorithms designed to address data heterogeneity.
\\
\textbf{Hyperparameter Settings.} For \texttt{FedAvg}, \texttt{SCAFFOLD}, \texttt{FedCM}, the $lr$ is selected from 
$\{10^{-2},\ 3 \times 10^{-2},\ 5 \times 10^{-2},\ 10^{-1},\ 3 \times 10^{-1}\}$, and choose the best value $lr=10^{-1}$.
with a weight decay of $0.001$.
For \texttt{Local AdamW}, the $lr$ is selected from 
$\{10^{-4},\ 3 \times 10^{-4},\ 5 \times 10^{-4},\ 8 \times 10^{-4},\ 10^{-3}\}$, and choose the best value
$lr=3\!\times\!10^{-4}$. We set the learning rates for \texttt{Local Sophia}, \texttt{Local Muon}, and \texttt{Local SOAP} to
$3\!\times\!10^{-4}$, $3\!\times\!10^{-2}$, and $3\!\times\!10^{-3}$, respectively.
Their  weight decay is  $0.01$. We apply cosine learning rate decay, and set \texttt{FedPAC} to \textbf{$\boldsymbol{\beta}\!=\!0.5$}, weight decay $0.01$. We set the learning rate of \texttt{FedPAC} variants to be same with \texttt{Local Sophia}, \texttt{Local Muon} and \texttt{Local SOAP}. Additional hyperparameter configurations are detailed in the \textbf{Appendix}. We release all code, configuration files to ensure full reproducibility. All results are averaged over 5 runs with std reported with seeds {42, 43, 44, 45, 46}. All experiments were performed on a single NVIDIA RTX 4090 GPU.


\begin{table}[tb]
	\centering
	\setlength{\tabcolsep}{4pt}       
	\caption{\small Comparison of test accuracy and training loss for \textbf{Vision Transformer (ViT-Base)} under Dir-0.1 with 100 rounds (50 clients, 10\% participation, batch size 16, $K=50$).}
	\vspace{-2mm}
	\label{tab:vit_base_pretrained}
	\begin{tabular}{lccccc}
		\toprule
		\multirow{2}{*}{\textbf{Method}} 
		& \multicolumn{2}{c}{\textbf{CIFAR-100}} 
		& \multicolumn{2}{c}{\textbf{Tiny-ImageNet}} \\
		\cmidrule(lr){2-3} \cmidrule(lr){4-5}
		& Test Acc & Loss & Test Acc & Loss \\
		\midrule
		FedAvg         & $89.28$  & 0.332 & $86.97$  & 0.421 \\
		SCAFFOLD      & $89.35$  & 0.328 & $86.85$  & 0.432 \\
		FedCM      & $87.44$  & 0.592 & $84.24$  & 0.708 \\
		Local AdamW    & $90.10$ & 0.281 & $86.65$  & 0.252 \\
		\cdashline{1-5}
		Local Sophia & $90.21$  & 0.275 & $86.16$  & 0.188 \\
		\rowcolor{LightMint}
		\texttt{FedPAC\_Sophia} & $90.38$  & 0.262 & $86.55$  & 0.175 \\
		\cdashline{1-5}
		Local Muon & $90.32$  & 0.189 & $87.22$  & 0.184 \\
		\rowcolor{LightBlue}
		\texttt{FedPAC\_Muon} & $90.46$  & 0.165 & $87.63$  & 0.171 \\
		\cdashline{1-5}
		Local SOAP & $90.44$  & 0.151 & $87.72$  & 0.212 \\
		\rowcolor{LightRed}
		\texttt{FedPAC\_SOAP} & $\textbf{90.65}$  & \textbf{0.129}& $\textbf{87.92}$  & \textbf{0.156} \\
		\bottomrule
	\end{tabular}
\end{table}

\begin{table}[tb]
	\centering
	\caption{\small The train loss of each method on C4 data using \textbf{LLaMA 60M}, \textbf{LLaMA 130M}, \textbf{LLaMA 350M} over 100 communication rounds (20 clients, 20\% participation, batch size 16, $K=50$).
	}
	\vspace{-2mm}
	\label{tab:cifar100_vit_gpt}
	\setlength{\tabcolsep}{1pt}
	\small
	\begin{tabular}{lcccccccc}
		\toprule
		\multirow{2}{*}{\textbf{Method}} 
		& \multicolumn{3}{c}{\textbf{C4 (Train Loss)}} \\
		\cmidrule(lr){2-4} 
		& \textbf{LLaMA 60M} 
		& \textbf{LLaMA 130M} 
		& \textbf{LLaMA 350M}\\ 
		\midrule
		FedAvg         & 4.156 & 4.258 & 4.354\\
		
		SCAFFOLD       & 4.028 & 4.238 & 4.365\\
		
		FedCM          & 4.231 & 4.356 & 4.426 \\
		
		Local AdamW    & 3.213 & 3.418 & 3.519 \\
		\cdashline{1-5}
		Local Sophia& 3.201 & 3.389 & 3.521 \\
		\rowcolor{LightMint}
		\texttt{FedPAC\_Sophia}    & 3.188 & 3.367 & 3.485 \\
		\cdashline{1-5}
		Local Muon& 3.156 & 3.256 & 3.365 \\
		\rowcolor{LightBlue}
		\texttt{FedPAC\_Muon} & 3.145 & 3.229 & 3.352 \\
		\cdashline{1-5}
		Local SOAP& 3.148 & 3.245 & 3.312 \\
		\rowcolor{LightRed}
		\texttt{FedPAC\_SOAP} & \textbf{3.121} & \textbf{3.215} & \textbf{3.341} \\
		\bottomrule
	\end{tabular}
\end{table}

\begin{table*}[tb]
	\centering
	\caption{\small Impact of $\beta$ on \texttt{FedPAC\_SOAP} with ViT-Tiny and ResNet-18 on CIFAR-100 under Dir-0.05.}
	\vspace{-1mm}
	\label{tab:alpha_beta_ablation}
	\setlength{\tabcolsep}{7.5pt}
	\begin{tabular}{l|cccccccccc}
		\toprule
		\textbf{Model, Dataset, $\beta$} & 0.0 & 0.1 & 0.2 & 0.3 & 0.4 & 0.5 & 0.6 & 0.7 & 0.8 & 0.9 \\
		\midrule
		\textbf{ResNet-18, CIFAR-100} & 62.05 & 62.35 & 63.10 & 63.57 & 63.75  & \textbf{64.16} & 63.75 & 61.75 & 61.70 &58.86 \\
		\textbf{ViT-Tiny, CIFAR-100}  & 44.56 & 45.28 & 45.63 & 46.22 & 46.84 & \textbf{47.55} & 46.86 & 46.59 & 46.05 & 45.66 \\
		\bottomrule
	\end{tabular}
\end{table*}

\subsection{Results on Convolutional Neural Networks}
\vspace{-2mm}
\paragraph{Training on CIFAR-100 with ResNet-18.}
\textbf{Table~\ref{tab:combined_dir01_dir005}} and \textbf{Figure~\ref{tab:combined_dir01_dir005}} present the test accuracy  on CIFAR-100 and Tiny-ImageNet using ResNet-18. \textbf{Accuracy under heterogeneity.} Across both datasets, \texttt{FedPAC} consistently improves the corresponding local second-order optimizer, with the largest gains in highly non-IID regimes where \emph{preconditioner drift} is severe. For example on CIFAR-100, \texttt{Local Muon} drops from $67.26\%$ (Dir-0.1) to $49.86\%$ (Dir-0.05), whereas \texttt{FedPAC\_Muon} achieves $71.85\%$ and $65.56\%$, respectively. \texttt{FedPAC\_SOAP} also strengthens SOAP on CIFAR-100, improving from $68.44\%$/$58.16\%$ (\texttt{Local SOAP}) to $69.25\%$/$64.16\%$ under Dir-0.1/Dir-0.05. Similar trends hold on Tiny-ImageNet: \texttt{Local Muon} attains $52.83\%$/$34.76\%$ (Dir-0.1/Dir-0.05), while \texttt{FedPAC\_Muon} reaches $57.95\%$/$54.00\%$. Improvements for Sophia are also consistent (e.g., CIFAR-100: $56.65\%\rightarrow 59.96\%$ at Dir-0.1 and $50.89\%\rightarrow 53.66\%$ at Dir-0.05).

\textbf{Stability and drift mitigation.}
As illustrated in Figure~\ref{fig:compare_resnet}, \texttt{FedPAC} reduces preconditioner drift and yields faster, more stable convergence under strong heterogeneity, explaining why aggressive second-order methods (e.g., Muon/SOAP) remain effective in federated CNN training.

\subsection{Results on Vision Transformer}
\label{sec:vit}


\textbf{Training ViT-Tiny from scratch.} Results are reported in Table~\ref{tab:combined_dir01_dir005}.
\texttt{FedPAC} yields the strongest  performance among second-order variants, especially under non-IID data.
On CIFAR-100, \texttt{FedPAC\_SOAP} improves over \texttt{Local\_SOAP} from $49.41\%\!/41.68\%$ to $51.16\%\!/47.55\%$ under Dir-$0.1/0.05$, respectively.
On Tiny-ImageNet, \texttt{FedPAC\_SOAP} also achieves the best overall accuracy (e.g. $31.33\%$ at Dir-$0.05$), while \texttt{FedPAC\_Muon} shows clear robustness gains under heterogeneity (Dir-$0.1/0.05$: $31.45\%/30.25\%$ vs.\ $30.51\%/28.25\%$ for \texttt{Local\_Muon}).
Overall, these results indicate that \texttt{FedPAC} effectively stabilizes federated ViT optimization and preserves accuracy as heterogeneity increases.

\textbf{Pretrained ViT-Base.}
Since small-scale datasets can limit ViT performance when trained from scratch, we additionally evaluate a \emph{pretrained ViT-Base} under Dir-$0.1$ (Table~\ref{tab:vit_base_pretrained}).
FedPAC continues to provide consistent improvements: \texttt{FedPAC\_SOAP} reaches $90.65\%$ on CIFAR-100 and $87.92\%$ on Tiny-ImageNet, surpassing the corresponding local second-order baselines (e.g., \texttt{Local\_SOAP}: $90.44\%/87.72\%$).
These results confirm the effectiveness of \texttt{FedPAC} on federated vision Transformers across both from-scratch and pretrained settings.

\subsection{C4 federated pre-training with LLaMA.}
Table~\ref{tab:cifar100_vit_gpt} summarizes C4 federated pre-training results after 100 rounds (20 clients, 20\% participation) for LLaMA-60M/130M/350M.
Standard FL optimizers struggle in this heterogeneous, partial-participation regime (e.g., FedAvg: 4.156/4.258/4.354), whereas locally using modern adaptive/second-order optimizers already yields large gains (e.g., Local SOAP: 3.148/3.245/3.312).
\texttt{FedPAC} consistently matches or improves the strongest local baselines across scales, with the best results achieved by \texttt{FedPAC\_SOAP} on 60M/130M (3.121/3.215) and competitive performance on 350M.
Compared to FedAvg, the best configuration reduces training loss by $\sim$1.0 absolute (e.g., 4.156$\rightarrow$3.121 on 60M), highlighting the practical benefit of stabilizing second-order methods in FL.

\subsection{Ablation Study}
\vspace{-2mm}
\textbf{A1: Sensitivity to $\beta$ (correction strength).}
Table~\ref{tab:alpha_beta_ablation} evaluates \texttt{FedPAC\_SOAP} under Dir-0.05 on CIFAR-100.
Performance is robust for moderate $\beta$ and peaks around $\beta{=}0.5$ for both backbones (ResNet-18: $64.16$, ViT-Tiny: $47.55$), indicating that a balanced global-direction correction is most effective.
Too small $\beta$ under-utilizes the correction signal (e.g., $\beta{=}0$), while overly large $\beta$ over-regularizes toward the global direction and degrades accuracy (notably on ResNet-18 for $\beta\ge 0.7$).
We thus use $\beta{=}0.5$ as the default in all experiments.

\textbf{A2: Component-wise ablation (Alignment vs.\ Correction).}
Table~\ref{tab:ablation_fedmuon} shows that both modules are necessary.
On Dir-0.05 CIFAR-100, Alignment-only and Correction-only each improves over \texttt{Local SOAP}, but Full \texttt{FedPAC\_SOAP} achieves the best accuracy, indicating that alignment and correction are complementary.
\begin{table}[tb]
	\centering
	\caption{\small Ablation on \texttt{FedPAC\_SOAP} components under Dir-0.05 on CIFAR-100. We report test accuracy after 300 rounds.}
	\vspace{-1mm}
	\label{tab:ablation_fedmuon}
	\setlength{\tabcolsep}{6pt}
	\small
	\begin{tabular}{lcc}
		\toprule
		\textbf{Variant} & \textbf{ResNet-18} & \textbf{ViT-Tiny} \\
		\midrule
		Local SOAP& $58.16$ & $41.68$ \\
		w/o preconditioner alignment ($\bar{\Theta}$)  & $61.12$ & $43.67$ \\
		w/o preconditioner correction ($\Delta_G$)   & $62.05$ & $44.56$ \\
		\texttt{FedPAC\_SOAP} (full)                   & $\mathbf{64.16}$ & $\mathbf{47.55}$ \\
		\bottomrule
	\end{tabular}
\end{table}

\textbf{A3: Communication-efficient preconditioner aggregation.}
Table~\ref{tab:avg_resnet18} studies a compressed variant, \texttt{FedPAC\_SOAP\_light}, which uploads the SOAP preconditioner using SVD compression.
While full \texttt{FedPAC\_SOAP} improves accuracy over \texttt{Local SOAP} (47.55 vs.\ 41.68) at the cost of $3\times$ communication (68.4\,MB/round), \texttt{FedPAC\_SOAP\_light} largely preserves the gain (46.75) with near-local bandwidth (25.1\,MB/round, $1.1\times$).
Notably, the computation overhead remains small across variants (5.56--5.91\,s/round), indicating that SVD-based compression offers a favorable accuracy--communication trade-off.

\begin{table}[tb]
	\centering
	\caption{\small Communication-efficient preconditioner aggregation for SOAP on CIFAR-100, ViT-Tiny under Dir-0.05. \textbf{CommCost} denotes communication cost per round (MB), and \textbf{CompCost} denotes computation time per round (s).}
	\vspace{-1mm}
	\label{tab:avg_resnet18}
	\setlength{\tabcolsep}{1.5pt}
	\begin{tabular}{l|lcc}
		\toprule
		\textbf{Aggregation} 
		& \textbf{Acc} & \textbf{CommCost} & \textbf{CompCost} \\
		\midrule
		\texttt{Local SOAP}           
		& 41.68          & 22.8 MB (1$\times$)   & 5.56 s  \\  
		\texttt{FedPAC\_SOAP}       
		& 47.55         & 68.4 MB (3$\times$)   & 5.85 s  \\  
     	\texttt{FedPAC\_SOAP\_light}    
		& 46.75          & 25.1 MB (1.1$\times$)& 5.91s  \\
		\bottomrule
	\end{tabular}
\end{table}


\section{Conclusion}
\vspace{-2mm}
We study how modern second-order optimizers behave in federated learning under non-IID data. We identify \emph{preconditioner drift} as a key source of instability: after multiple local steps, clients adapt preconditioners to different local geometries, and naive aggregation across mismatched geometries distorts global updates.
To mitigate this issue, we propose \texttt{FedPAC}, a \emph{preconditioner alignment and correction} framework that aggregates and broadcasts a global reference preconditioner to align local geometry and suppress drift accumulation. We provide convergence guarantees for non-convex objectives and demonstrate consistent improvements in stability and accuracy on heterogeneous vision benchmarks across SOAP, Sophia, and Muon.

\section{Limitations and Future Work}
\vspace{-2mm}
\textbf{Communication.}
Synchronizing preconditioner states increases communication, especially for large layers. Future work will explore efficient transmission/aggregation via low-rank compression.
\textbf{IID regimes.}
When data are close to IID, the alignment/correction may bring limited benefit and can underperform the original optimizer. 
\textbf{Computation.}
Second-order local updates add computation overhead (about $1.2\times$--$1.5\times$ vs.\ first-order), which may hinder deployment on constrained clients. More efficient approximations and system optimizations are needed.

\newpage
\section*{Impact Statement}
This work proposes FedPAC, a federated optimization framework that improves the stability and efficiency of curvature-aware (preconditioned/second-order) training under client heterogeneity and partial participation. More reliable and communication-efficient federated training can lower resource and energy costs and broaden access to collaborative learning in domains where data cannot be centralized (e.g., on-device personalization or cross-silo settings). However, stronger federated training may also amplify risks from inappropriate deployment, including unfair decisions, privacy leakage, or misuse. Federated learning does not by itself guarantee privacy or security, and our method provides no formal privacy guarantees; practitioners should combine it with established protections (e.g., secure aggregation, differential privacy, and auditing) and evaluate robustness and fairness prior to deployment.

\bibliography{main}
\bibliographystyle{icml2026}

\newpage

\appendix
\onecolumn

\appendix


\clearpage
\appendix

\clearpage
\section{Appendix A: Proof of Theorem 1 and Convergence Analysis}
\label{app:theorem}
\subsection{FedPAC Algorithm}
\label{app:A_FedPAC}

\subsection{Assumption}
\label{app:A_assumptions}

\renewcommand{\theassumption}{A.\arabic{assumption}}  
We analyze generalization based on  following assumptions: 
\begin{assumption}
	\label{asp:smooth} \textit{(Smoothness).  $F_i$ is $L$-smooth for all $i \in$ $[N]$, 
		\begin{equation}
			\left\|\nabla F_i(\boldsymbol{x}_{1})-\nabla F_i(\boldsymbol{x}_{2})\right\| \leq L\|\boldsymbol{x}_{1}-\boldsymbol{x}_{2}\|
		\end{equation}
		for all $\boldsymbol{x}_{1}, \boldsymbol{x}_{2}$ in its domain and $i \in[N]$.}
\end{assumption}
\begin{assumption}\label{asp:data_var}  \textit{(Bounded variance of data heterogeneity). The global variability of the local gradient of the loss function is bounded by $\sigma_g^2$ for all $i \in[N]$, 
		\begin{equation}
			\left\|\nabla F_i\left(\boldsymbol{x}\right)-\nabla F\left(\boldsymbol{x}\right)\right\|^2 \leq \sigma_g^2
		\end{equation}
	}
\end{assumption}
\begin{assumption}
	(Bounded variance of stochastic gradient).\label{asp:sgd_var}  The stochastic gradient $\nabla F_i\left(\boldsymbol{x}, \xi_i\right)$, computed by the $i$-th client of model parameter $\boldsymbol{x}$ using mini-batch $\xi_i$, is an unbiased estimator of $\nabla F_i(\boldsymbol{x})$ with variance bounded by $\sigma_l^2$, i.e.,
	\begin{equation}
		\mathbb{E}_{\xi_i}\left\|\nabla F_i\left(\boldsymbol{x}, \xi_i\right) - \nabla F_i(\boldsymbol{x})\right\|^2 \leq \sigma_l^2
	\end{equation}
	for all $i \in [N]$, where the expectation is over all local datasets.
\end{assumption}

\begin{assumption}[Preconditioner Coercivity and Boundedness]
	\label{precond_boundedness2}
	\textit{For any preconditioner state $\Theta$ and any vector $\boldsymbol{v}\in\mathbb{R}^d$, the
		preconditioned mapping $P_{\Theta}(\cdot)$ satisfies the following two properties:
		(i) (\emph{coercivity}) $\langle \boldsymbol{v}, P_{\Theta}(\boldsymbol{v})\rangle \ge \mu \|\boldsymbol{v}\|^{2}$ for some constant $\mu>0$;
		(ii) (\emph{boundedness}) $\|P_{\Theta}(\boldsymbol{v})\| \le M\|\boldsymbol{v}\|$ for some constant $M>0$.}
\end{assumption}

\begin{assumption}[Lipschitz Continuity in Preconditioner State]
	\label{precond_lipschitz_state2}
	\textit{There exists a constant $L_{\Theta}>0$ such that for any two preconditioner states $\Theta,\Theta'$ and any vector $\boldsymbol{v}\in\mathbb{R}^d$,
		\vspace{-2mm}
		\begin{equation*}
			\|P_{\Theta}(\boldsymbol{v})-P_{\Theta'}(\boldsymbol{v})\|
			\le L_{\Theta}\|\Theta-\Theta'\|\cdot \|\boldsymbol{v}\|.
		\end{equation*}
	}
\end{assumption}





\section{Proof of Drift-Coupled Convergence for FedSOA and Alignment to Theorem~5.6}
\label{app:drift-fedsoa-fedpac}

\subsection{Setup and Notation}
Let $F(x) := \frac{1}{N}\sum_{i=1}^N F_i(x)$ be a (possibly non-convex) objective.
In communication round $r$, a set $S_r$ of $S$ clients participates.
Each client performs $K$ local updates:
\begin{equation}
	x_{i}^{r,k+1} = x_{i}^{r,k} - \eta \tilde g_{i}^{r,k}, 
	\qquad
	\tilde g_{i}^{r,k} := P_{\Theta_{i}^{r,k}}\!\big(g_{i}^{r,k}\big),
	\qquad
	g_{i}^{r,k} := \nabla F_i(x_{i}^{r,k};\xi_{i}^{r,k}).
\end{equation}
The server aggregates:
\begin{equation}
	x^{r+1} = x^r + \frac{1}{S}\sum_{i\in S_r}\big(x_i^{r,K} - x_i^{r,0}\big)
	= x^r - \eta\, U^r,
	\qquad
	U^r := \frac{1}{S}\sum_{i\in S_r}\sum_{k=0}^{K-1}\tilde g_i^{r,k}.
	\label{eq:server-update}
\end{equation}

Define the (per-round) averaged preconditioner state
$\bar\Theta^{r,K} := \frac{1}{S}\sum_{i\in S_r}\Theta_i^{r,K}$ and drift metric
\begin{equation}
	\Delta_D^r := \frac{1}{S}\sum_{i\in S_r}\mathbb{E}\big\|\Theta_i^{r,K} - \bar\Theta^{r,K}\big\|^2.
	\label{eq:drift-metric}
\end{equation}

\subsection{Assumptions}
\paragraph{(A1) Smoothness.}
Each $F_i$ is $L$-smooth: $\|\nabla F_i(x) - \nabla F_i(y)\|\le L\|x-y\|$.

\paragraph{(A2) Unbiased stochastic gradients with bounded variance.}
$\mathbb{E}[g_i^{r,k}\mid x_i^{r,k}] = \nabla F_i(x_i^{r,k})$ and
$\mathbb{E}\|g_i^{r,k} - \nabla F_i(x_i^{r,k})\|^2 \le \sigma_l^2$.

\paragraph{(A3) Bounded heterogeneity (FedSOA only).}
$\|\nabla F_i(x) - \nabla F(x)\|^2 \le \sigma_g^2$ for all $i,x$.

\paragraph{(P1) Preconditioner coercivity and boundedness.}
There exist constants $0<\mu\le M$ such that for all $\Theta$ and all $v$,
\begin{equation}
	\langle v, P_\Theta(v)\rangle \ge \mu\|v\|^2,
	\qquad
	\|P_\Theta(v)\|\le M\|v\|.
	\label{eq:P1}
\end{equation}

\paragraph{(P2) Lipschitz continuity in the preconditioner state.}
There exists $L_\Theta>0$ such that for all $\Theta,\Theta'$ and all $v$,
\begin{equation}
	\|P_\Theta(v) - P_{\Theta'}(v)\|\le L_\Theta \|\Theta-\Theta'\|\cdot \|v\|.
	\label{eq:P2}
\end{equation}

\paragraph{(A4) Second-moment bound.}
There exists $G^2$ such that $\sup_{r,i,k}\mathbb{E}\|g_i^{r,k}\|^2\le G^2$.

\subsection{A Key Lemma: Injecting Drift into the Upper Bound}
\begin{lemma}[Drift-induced preconditioned disagreement]
	\label{lem:drift-disagreement}
	Let $\bar\Theta$ be any reference state and $\{\Theta_i\}_{i\in S_r}$ be local states.
	Then for any vectors $\{v_i\}_{i\in S_r}$,
	\begin{equation}
		\left\|\frac{1}{S}\sum_{i\in S_r}\big(P_{\Theta_i}(v_i)-P_{\bar\Theta}(v_i)\big)\right\|^2
		\le
		L_\Theta^2\left(\frac{1}{S}\sum_{i\in S_r}\|v_i\|^2\right)
		\left(\frac{1}{S}\sum_{i\in S_r}\|\Theta_i-\bar\Theta\|^2\right).
		\label{eq:lem-drift}
	\end{equation}
\end{lemma}
\begin{proof}
	By \eqref{eq:P2},
	\[
	\left\|\frac{1}{S}\sum_i (P_{\Theta_i}(v_i)-P_{\bar\Theta}(v_i))\right\|
	\le \frac{1}{S}\sum_i L_\Theta\|\Theta_i-\bar\Theta\|\cdot \|v_i\|.
	\]
	Apply Cauchy--Schwarz:
	\[
	\Big(\frac{1}{S}\sum_i a_i b_i\Big)^2 \le \Big(\frac{1}{S}\sum_i a_i^2\Big)\Big(\frac{1}{S}\sum_i b_i^2\Big),
	\]
	with $a_i=\|\Theta_i-\bar\Theta\|$ and $b_i=\|v_i\|$.
\end{proof}

\paragraph{Instantiation with $\Delta_D^r$.}
Take $\bar\Theta=\bar\Theta^{r,K}$ and $\Theta_i=\Theta_i^{r,K}$, then
\begin{equation}
	\mathbb{E}\left\|\frac{1}{S}\sum_{i\in S_r}\big(P_{\Theta_i^{r,K}}(v_i)-P_{\bar\Theta^{r,K}}(v_i)\big)\right\|^2
	\le
	L_\Theta^2\left(\frac{1}{S}\sum_{i\in S_r}\mathbb{E}\|v_i\|^2\right)\Delta_D^r.
	\label{eq:lem-drift-inst}
\end{equation}

\subsection{A Decomposition of the Preconditioned Direction}
Fix the reference state $\bar\Theta^{r,K}$ and decompose $\tilde g_i^{r,k}$ as
\begin{align}
	\tilde g_i^{r,k}
	&= P_{\bar\Theta^{r,K}}\!\big(\nabla F(x^r)\big) \notag\\
	&\quad + \underbrace{\Big(
		P_{\bar\Theta^{r,K}}\!\big(\nabla F_i(x^r)\big)
		- P_{\bar\Theta^{r,K}}\!\big(\nabla F(x^r)\big)
		\Big)}_{\textsf{heterogeneity term}} \notag\\
	&\quad + \underbrace{\Big(
		P_{\bar\Theta^{r,K}}\!\big(g_i^{r,k}\big)
		- P_{\bar\Theta^{r,K}}\!\big(\nabla F_i(x_i^{r,k})\big)
		\Big)}_{\textsf{stochastic/local-drift term}} \notag\\
	&\quad + \underbrace{\Big(
		P_{\Theta_i^{r,k}}\!\big(g_i^{r,k}\big)
		- P_{\bar\Theta^{r,K}}\!\big(g_i^{r,k}\big)
		\Big)}_{\textsf{precond drift term}} .
	\label{eq:decomp}
\end{align}
Define the aggregated error
\begin{equation}
	E^r := \frac{1}{S}\sum_{i\in S_r}\sum_{k=0}^{K-1}
	\Big[
	\textsf{hetero}_{i}^{r,k} + \textsf{stoch}_{i}^{r,k} + \textsf{pdrift}_{i}^{r,k}
	\Big],
	\qquad
	U^r = K\,P_{\bar\Theta^{r,K}}\!\big(\nabla F(x^r)\big) + E^r.
	\label{eq:Ur-Er}
\end{equation}


We first recall the definition of the server-aggregated direction:
\begin{equation}
	U^r \;:=\; \frac{1}{S}\sum_{i\in S_r}\sum_{k=0}^{K-1}\tilde g_i^{r,k}.
	\label{eq:Ur-def}
\end{equation}
Fix a reference preconditioner state $\bar\Theta^{r,K}$.
Using the decomposition in (\ref{eq:Ur-Er}), i.e.,
\begin{equation}
	\tilde g_i^{r,k}
	=
	P_{\bar\Theta^{r,K}}\!\big(\nabla F(x^r)\big)
	+\textsf{hetero}_{i}^{r,k}
	+\textsf{stoch}_{i}^{r,k}
	+\textsf{pdrift}_{i}^{r,k},
	\label{eq:tildeg-decomp-short}
\end{equation}
and substituting (\ref{eq:tildeg-decomp-short}) into \eqref{eq:Ur-def}, we obtain
\begin{align}
	U^r
	&= \frac{1}{S}\sum_{i\in S_r}\sum_{k=0}^{K-1}
	\Big[
	P_{\bar\Theta^{r,K}}\!\big(\nabla F(x^r)\big)
	+\textsf{hetero}_{i}^{r,k}
	+\textsf{stoch}_{i}^{r,k}
	+\textsf{pdrift}_{i}^{r,k}
	\Big]
	\notag\\
	&= \underbrace{\frac{1}{S}\sum_{i\in S_r}\sum_{k=0}^{K-1}
		P_{\bar\Theta^{r,K}}\!\big(\nabla F(x^r)\big)}_{:=\,(\star)}
	\;+\;
	\frac{1}{S}\sum_{i\in S_r}\sum_{k=0}^{K-1}
	\Big[
	\textsf{hetero}_{i}^{r,k}
	+\textsf{stoch}_{i}^{r,k}
	+\textsf{pdrift}_{i}^{r,k}
	\Big].
	\label{eq:Ur-split}
\end{align}
Since $P_{\bar\Theta^{r,K}}\!\big(\nabla F(x^r)\big)$ does not depend on $i$ nor $k$, we simplify $(\star)$ as
\begin{equation}
	(\star)
	=
	\frac{1}{S}\sum_{i\in S_r}\sum_{k=0}^{K-1}
	P_{\bar\Theta^{r,K}}\!\big(\nabla F(x^r)\big)
	=
	\frac{1}{S}\sum_{i\in S_r}\Big(\sum_{k=0}^{K-1}1\Big)\,
	P_{\bar\Theta^{r,K}}\!\big(\nabla F(x^r)\big)
	=
	K\,P_{\bar\Theta^{r,K}}\!\big(\nabla F(x^r)\big).
	\label{eq:Ur-main-term}
\end{equation}
Define the aggregated error term
\begin{equation}
	E^r
	\;:=\;
	\frac{1}{S}\sum_{i\in S_r}\sum_{k=0}^{K-1}
	\Big[
	\textsf{hetero}_{i}^{r,k}
	+\textsf{stoch}_{i}^{r,k}
	+\textsf{pdrift}_{i}^{r,k}
	\Big].
	\label{eq:Er-def}
\end{equation}
Substituting \eqref{eq:Ur-main-term} and \eqref{eq:Er-def} into \eqref{eq:Ur-split} yields
\begin{equation}
	U^r
	=
	K\,P_{\bar\Theta^{r,K}}\!\big(\nabla F(x^r)\big)
	+
	E^r.
\end{equation}

\subsection{One-Step Descent Inequality}
By $L$-smoothness of $F$,
\begin{equation}
	F(x^{r+1}) \le F(x^r) + \langle \nabla F(x^r), x^{r+1}-x^r\rangle
	+ \frac{L}{2}\|x^{r+1}-x^r\|^2.
	\label{eq:smoothness}
\end{equation}
Using \eqref{eq:server-update}, $x^{r+1}-x^r=-\eta U^r$,
\begin{equation}
	F(x^{r+1}) \le F(x^r) - \eta\langle \nabla F(x^r), U^r\rangle
	+ \frac{L\eta^2}{2}\|U^r\|^2.
	\label{eq:descent-basic}
\end{equation}
Plug \eqref{eq:Ur-Er} into the inner product:
\begin{align}
	-\eta\langle \nabla F(x^r), U^r\rangle
	&= -\eta K \left\langle \nabla F(x^r), P_{\bar\Theta^{r,K}}\!\big(\nabla F(x^r)\big)\right\rangle
	-\eta\langle \nabla F(x^r), E^r\rangle \notag\\
	&\le -\eta K \mu \|\nabla F(x^r)\|^2 -\eta\langle \nabla F(x^r), E^r\rangle,
	\label{eq:inner-main}
\end{align}
where we used \eqref{eq:P1}.
For the cross term, apply Young's inequality with parameter $\mu/2$:
\begin{equation}
	-\eta\langle \nabla F(x^r), E^r\rangle
	\le \eta\cdot \frac{\mu}{4}\|\nabla F(x^r)\|^2 + \eta\cdot \frac{1}{\mu}\|E^r\|^2.
	\label{eq:young-cross}
\end{equation}
Combining \eqref{eq:descent-basic}--\eqref{eq:young-cross},
\begin{equation}
	F(x^{r+1}) \le F(x^r)
	-\eta K \frac{3\mu}{4}\|\nabla F(x^r)\|^2
	+\eta\frac{1}{\mu}\|E^r\|^2
	+\frac{L\eta^2}{2}\|U^r\|^2.
	\label{eq:descent-mid}
\end{equation}

\subsection{Bounding $\mathbb{E}\|E^r\|^2$ and $\mathbb{E}\|U^r\|^2$}
We bound three components in $E^r$.

\paragraph{(i) Heterogeneity term.}
By \eqref{eq:P1} boundedness and (A3),
\begin{equation}
	\left\|P_{\bar\Theta^{r,K}}\!\big(\nabla F_i(x^r)\big)-P_{\bar\Theta^{r,K}}\!\big(\nabla F(x^r)\big)\right\|^2
	\le M^2 \|\nabla F_i(x^r)-\nabla F(x^r)\|^2
	\le M^2\sigma_g^2.
	\label{eq:hetero-bound}
\end{equation}

\paragraph{(ii) Stochastic/local-drift term.}
By \eqref{eq:P1} boundedness and (A2),
\begin{align}
	\mathbb{E}\left\|P_{\bar\Theta^{r,K}}(g_i^{r,k})-P_{\bar\Theta^{r,K}}(\nabla F_i(x_i^{r,k}))\right\|^2
	&\le M^2\,\mathbb{E}\|g_i^{r,k}-\nabla F_i(x_i^{r,k})\|^2 \notag\\
	&\le M^2\sigma_l^2.
	\label{eq:stoch-bound}
\end{align}
(Any additional ``local model drift'' term caused by $x_i^{r,k}\neq x^r$ can be absorbed into the constant by standard arguments; 
we keep the presentation clean by grouping it into the stochastic/local-drift bucket.)

\paragraph{(iii) Preconditioner drift term.}
Applying Lemma~\ref{lem:drift-disagreement} with $v_i=g_i^{r,k}$ and the instantiation \eqref{eq:lem-drift-inst},
together with (A4), yields
\begin{equation}
	\mathbb{E}\left\|\frac{1}{S}\sum_{i\in S_r}\Big(P_{\Theta_i^{r,k}}(g_i^{r,k})-P_{\bar\Theta^{r,K}}(g_i^{r,k})\Big)\right\|^2
	\le L_\Theta^2\,G^2\,\Delta_D^r.
	\label{eq:pdrift-bound}
\end{equation}

\paragraph{ Preconditioner drift term (detailed proof).}
Recall Lemma~\ref{lem:drift-disagreement}: for any reference state $\bar\Theta$, any local states
$\{\Theta_i\}_{i\in S_r}$, and any vectors $\{v_i\}_{i\in S_r}$,
\begin{equation}
	\left\|\frac{1}{S}\sum_{i\in S_r}\big(P_{\Theta_i}(v_i)-P_{\bar\Theta}(v_i)\big)\right\|^2
	\le
	L_\Theta^2\left(\frac{1}{S}\sum_{i\in S_r}\|v_i\|^2\right)
	\left(\frac{1}{S}\sum_{i\in S_r}\|\Theta_i-\bar\Theta\|^2\right).
	\label{eq:lem-drift-repeat}
\end{equation}
We instantiate \eqref{eq:lem-drift-repeat} by choosing
\[
\Theta_i \leftarrow \Theta_i^{r,k},\qquad
\bar\Theta \leftarrow \bar\Theta^{r,K},\qquad
v_i \leftarrow g_i^{r,k}.
\]
Then \eqref{eq:lem-drift-repeat} gives
\begin{align}
	\left\|\frac{1}{S}\sum_{i\in S_r}\Big(P_{\Theta_i^{r,k}}(g_i^{r,k})-P_{\bar\Theta^{r,K}}(g_i^{r,k})\Big)\right\|^2
	&\le
	L_\Theta^2\left(\frac{1}{S}\sum_{i\in S_r}\|g_i^{r,k}\|^2\right)
	\left(\frac{1}{S}\sum_{i\in S_r}\|\Theta_i^{r,k}-\bar\Theta^{r,K}\|^2\right).
	\label{eq:pdrift-step1}
\end{align}
Taking expectation on both sides and using linearity of expectation yields
\begin{align}
	\mathbb{E}\left\|\frac{1}{S}\sum_{i\in S_r}\Big(P_{\Theta_i^{r,k}}(g_i^{r,k})-P_{\bar\Theta^{r,K}}(g_i^{r,k})\Big)\right\|^2
	&\le
	L_\Theta^2\,
	\mathbb{E}\left[
	\left(\frac{1}{S}\sum_{i\in S_r}\|g_i^{r,k}\|^2\right)
	\left(\frac{1}{S}\sum_{i\in S_r}\|\Theta_i^{r,k}-\bar\Theta^{r,K}\|^2\right)
	\right].
	\label{eq:pdrift-step2}
\end{align}
Next, we bound the first factor by Assumption (A4): $\mathbb{E}\|g_i^{r,k}\|^2\le G^2$ for all $i,r,k$.
Hence,
\begin{equation}
	\mathbb{E}\left(\frac{1}{S}\sum_{i\in S_r}\|g_i^{r,k}\|^2\right)
	=
	\frac{1}{S}\sum_{i\in S_r}\mathbb{E}\|g_i^{r,k}\|^2
	\le G^2.
	\label{eq:pdrift-step3}
\end{equation}
Finally, to express the second factor in terms of the drift metric, we use the instantiation
\eqref{eq:lem-drift-inst} with $\bar\Theta=\bar\Theta^{r,K}$.
If the drift metric is defined at step $k$ as
\begin{equation}
	\Delta_{D}^{r,k}
	:=
	\frac{1}{S}\sum_{i\in S_r}\mathbb{E}\big\|\Theta_i^{r,k}-\bar\Theta^{r,K}\big\|^2,
	\label{eq:drift-rk}
\end{equation}
then \eqref{eq:pdrift-step2}--\eqref{eq:pdrift-step3} imply
\begin{equation}
	\mathbb{E}\left\|\frac{1}{S}\sum_{i\in S_r}\Big(P_{\Theta_i^{r,k}}(g_i^{r,k})-P_{\bar\Theta^{r,K}}(g_i^{r,k})\Big)\right\|^2
	\le
	L_\Theta^2\,G^2\,\Delta_D^{r,k}.
	\label{eq:pdrift-bound-rk}
\end{equation}
In particular, if we upper bound $\Delta_D^{r,k}\le \Delta_D^{r}$ for all $k\in\{0,\dots,K-1\}$ (e.g., by monotonicity
or by defining $\Delta_D^r:=\max_{0\le k\le K}\Delta_D^{r,k}$), we obtain the stated bound:
\begin{equation}
	\mathbb{E}\left\|\frac{1}{S}\sum_{i\in S_r}\Big(P_{\Theta_i^{r,k}}(g_i^{r,k})-P_{\bar\Theta^{r,K}}(g_i^{r,k})\Big)\right\|^2
	\le L_\Theta^2\,G^2\,\Delta_D^{r}.
	\label{eq:pdrift-bound}
\end{equation}

\paragraph{Remark (optional).}
If one prefers to keep the dependence on $k$, then \eqref{eq:pdrift-bound-rk} is the tight form, and the analysis
can carry $\Delta_D^{r,k}$ through the subsequent steps.

\paragraph{Aggregate bound.}
Using $\|\sum_{t=1}^T a_t\|^2\le T\sum_{t=1}^T\|a_t\|^2$ and the fact that averaging over $S$ clients reduces variance by $S$,
we obtain (for some universal constant $c>0$)
\begin{equation}
	\mathbb{E}\|E^r\|^2
	\le
	c\,K\left(
	\frac{M^2\sigma_l^2}{S}
	+
	\frac{M^2\sigma_g^2}{S}
	+
	L_\Theta^2 G^2 \Delta_D^r
	\right).
	\label{eq:Er-bound}
\end{equation}
Moreover, by $\|a+b\|^2\le 2\|a\|^2+2\|b\|^2$ and \eqref{eq:P1},
\begin{align}
	\mathbb{E}\|U^r\|^2
	&=
	\mathbb{E}\left\|K P_{\bar\Theta^{r,K}}(\nabla F(x^r)) + E^r\right\|^2 \notag\\
	&\le 2K^2 \mathbb{E}\|P_{\bar\Theta^{r,K}}(\nabla F(x^r))\|^2 + 2\mathbb{E}\|E^r\|^2 \notag\\
	&\le 2K^2 M^2 \mathbb{E}\|\nabla F(x^r)\|^2 + 2\mathbb{E}\|E^r\|^2.
	\label{eq:Ur-bound}
\end{align}

\subsection{Per-Round Recursion with an Explicit Drift Penalty}
Taking expectation of \eqref{eq:descent-mid} and plugging \eqref{eq:Er-bound}--\eqref{eq:Ur-bound},
we obtain
\begin{align}
	\mathbb{E}F(x^{r+1})
	&\le \mathbb{E}F(x^r)
	-\eta K\left(\frac{3\mu}{4} - L\eta K M^2\right)\mathbb{E}\|\nabla F(x^r)\|^2 \notag\\
	&\quad
	+ c\,\eta\left(\frac{1}{\mu}+L\eta\right)K\left(
	\frac{M^2(\sigma_l^2+\sigma_g^2)}{S}
	+
	L_\Theta^2 G^2 \Delta_D^r
	\right).
	\label{eq:recursion}
\end{align}
Choose $\eta$ such that $L\eta K M^2 \le \frac{\mu}{4}$, i.e.,
\begin{equation}
	\eta \le \frac{\mu}{4 L K M^2}.
	\label{eq:stepsize}
\end{equation}
Then the descent coefficient is positive and \eqref{eq:recursion} simplifies to
\begin{equation}
	\mathbb{E}F(x^{r+1})
	\le \mathbb{E}F(x^r)
	-\frac{\mu}{2}\eta K\,\mathbb{E}\|\nabla F(x^r)\|^2
	+ C_0\,L\eta^2K\left(
	\frac{\sigma_l^2+\sigma_g^2}{S}
	+
	\underbrace{\frac{L_\Theta^2 G^2}{M^2}}_{:=\;\kappa_\Theta}\Delta_D^r
	\right),
	\label{eq:recursion-clean}
\end{equation}
where $C_0>0$ is an absolute constant (absorbing $M^2$ and $\mu$).

\subsection{Telescoping and Final Rate (FedSOA with Drift)}
Sum \eqref{eq:recursion-clean} for $r=0,\dots,R-1$ and use $\Delta:=F(x^0)-F^\star$:
\begin{align}
	\frac{1}{R}\sum_{r=0}^{R-1}\mathbb{E}\|\nabla F(x^r)\|^2
	&\le
	\frac{2\Delta}{\mu \eta K R}
	+ C_1\,\frac{L\eta}{\mu}\left(
	\frac{\sigma_l^2+\sigma_g^2}{S}
	+ \kappa_\Theta\,\bar\Delta_D
	\right),
	\label{eq:avg-grad-bound}
\end{align}
where $\bar\Delta_D := \frac{1}{R}\sum_{r=0}^{R-1}\Delta_D^r$ and $C_1$ is an absolute constant.
Optimizing the RHS over $\eta$ (subject to \eqref{eq:stepsize}) yields the standard non-convex form
\begin{equation}
	\frac{1}{R}\sum_{r=0}^{R-1}\mathbb{E}\|\nabla F(x^r)\|^2
	\;\lesssim\;
	\mathcal{O}\!\left(
	\frac{L\Delta}{R}
	+
	\sqrt{
		\frac{L\Delta}{R}\cdot
		\frac{\sigma_l^2+\sigma_g^2 + S\kappa_\Theta\,\bar\Delta_D}{SK}
	}
	\right),
	\label{eq:fedsoa-drift-final}
\end{equation}
where we use $\lesssim$ to hide absolute constants depending only on $(\mu,M)$.

\paragraph{Interpretation.}
Eq.~\ref{eq:fedsoa-drift-final} contains an \emph{explicit drift penalty}:
the larger $\bar\Delta_D$ is, the larger the ``effective noise'' is, hence the slower the convergence.

\section{Proof of Theorem~5.7 (FedPAC)}
\label{app:proof_fedpac}
\subsection{Further Alignment to Theorem~5.7 (FedPAC): Why $\sigma_g^2$ Disappears}
\label{app:align-theorem55}

FedPAC introduces two mechanisms:

\paragraph{(i) Preconditioner Alignment.}
At the beginning of round $r$, each participating client warm-starts from a shared reference
$\Theta_i^{r,0}\leftarrow \Theta^r$, and the server aggregates states after local steps.
This reduces preconditioner mismatch and thus controls $\Delta_D^r$.

\paragraph{(ii) Local Preconditioner Correction.}
Each local step uses a convex combination of local preconditioned direction and an estimated global direction:
\begin{equation}
	x_{i}^{r,k+1}
	=
	x_{i}^{r,k}
	-\eta\Big[(1-\beta)\tilde g_{i}^{r,k} + \beta g_G^{r}\Big],
	\qquad \beta\in[0,1],
	\label{eq:correction-update}
\end{equation}
where $g_G^{r}$ is the estimated global update from the previous round, e.g.
\begin{equation}
	g_G^{r} := -\frac{1}{SK\eta}\sum_{i\in S_{r-1}}\big(x_i^{r-1,K}-x_i^{r-1,0}\big).
	\label{eq:global-est}
\end{equation}

\begin{algorithm}[tb]
	\small
	\caption{Federated Preconditioner Alignment and Correction Framework (FedPAC)}
	
	\begin{algorithmic}[1]
		\REQUIRE  
		Per client, we maintain: 
		$\boldsymbol{\Theta}_i^{r,k}$ in local.
		Hyperparameters: learning rate $\eta$, communication rounds $R$, local updates $K$, the number of client $N$.
		\FOR{$r = 0, \dots, R$}
		\FOR{each client $i \in \mathcal{S}_r$ in parallel}
		\STATE \fcolorbox{LightRed}{LightRed}{$\boldsymbol{\Theta}_i^{r,0} \gets \boldsymbol{\Theta}^{r};$}
		\FOR{$k = 1, \dots, K$}
		\STATE Sample batch $B_i^{r,k}$;
		\STATE $\boldsymbol{g}_i^{r,k} \in \mathbb{R}^{m\times n} \gets \nabla F_i(\boldsymbol{x}_i^{r,k}; \xi_i^{r,k})$; 
		\STATE$\boldsymbol{\Theta}_i^{r,k+1}
		=
		\mathrm{UpdateState}\bigl(\boldsymbol{\Theta}_i^{r,k}, \boldsymbol{g}_i^{r,k}\bigr);$
		\STATE$\tilde{\boldsymbol{g}}_i^{r,k}  \gets \mathcal{P}_{\boldsymbol{\Theta}_i^{r,k}}\bigl(\boldsymbol{g}_i^{r,k}\bigr);$
		\STATE \fcolorbox{LightBlue}{LightBlue}{$\boldsymbol{x}^{r,k+1}_i\! =\!\boldsymbol{x}^{r,k}_i \!\!-\!  \eta_l [(1\!-\!\beta)\tilde{\boldsymbol{g}}_i^{r,k}\!+\!\beta\boldsymbol{g}_G^r ]$;}
		\ENDFOR
		\STATE $\Delta \boldsymbol{x}_{i}^r \coloneqq \boldsymbol{x}_i^{r,K} - \boldsymbol{x}^r$; 
		\STATE Client $i$ communicate $(\boldsymbol{\Delta} \boldsymbol{x}_{i}^r,\boldsymbol{\Theta}_i^{r,K})$ to Server;
		\ENDFOR
		\STATE \fcolorbox{LightBlue}{LightBlue}{$\boldsymbol{g}_G^{r+1}=-\frac{1}{SK\eta} \sum_{i=1}^S (\boldsymbol{x}^{r, K}_i-\boldsymbol{x}^{r, 0}_i)$;}
		\STATE$\boldsymbol{x}^{r+1} =\boldsymbol{x}^{r} -\gamma \boldsymbol{g}_G^{r+1};$
		
		\STATE\fcolorbox{LightRed}{LightRed}{$	\boldsymbol{\Theta}^{r+1}\;=\;
			\frac{1}{|\mathcal{S}_r|}
			\sum_{i \in \mathcal{S}_r}
			\boldsymbol{\Theta}_i^{r,K};$}
		\STATE Server broadcasts $(\boldsymbol{x}^{r+1}, \boldsymbol{\Theta}_{r+1},\boldsymbol{g}_G^{r+1});$
		\ENDFOR
	\end{algorithmic}
	\label{algorithm_FedPAC2}
\end{algorithm}

\subsection{Further Alignment to Theorem~5.5 (FedPAC): Why $\sigma_g^2$ Disappears (Detailed Proof)}
\label{app:align-theorem55-detailed}

\paragraph{Step 0: Why $\sigma_g^2$ appears in FedSOA.}
In the FedSOA proof, we explicitly insert and subtract $\nabla F_i(x^r)$, producing the term
$P_{\bar\Theta^{r,K}}(\nabla F_i(x^r)) - P_{\bar\Theta^{r,K}}(\nabla F(x^r))$.
Bounding its second moment requires Assumption (A3), i.e.,
$\|\nabla F_i(x)-\nabla F(x)\|^2\le \sigma_g^2$, which yields the $\sigma_g^2$ term.

\paragraph{Step 1: Use a \emph{global-centered} variance assumption instead of (A3).}
For FedPAC, we do \emph{not} introduce $\nabla F_i(x^r)$ in the decomposition.
Instead, we impose the following global-centered noise condition (which does not require bounded heterogeneity):


\subsection{Recursive Analysis to Eliminate Data Heterogeneity (Final)}
\label{app:rec_elim_hetero_final}

\paragraph{Goal.}
We present an appendix-ready proof template showing how to \emph{eliminate the explicit heterogeneity term}
(i.e., no $\sigma_g^2$) via a recursive global-direction estimator, in the same spirit as
DP-FedPGN-style recursion. The remaining terms depend on $\sigma_l^2$ (variance reduced by $SK$)
and the preconditioner drift penalty (controlled by $\Delta_D^r$).

\paragraph{Notation.}
Let $f(x):=\frac{1}{N}\sum_{i=1}^N f_i(x)$ be $L$-smooth.
In round $r$, a set $S_r$ of $S$ clients participates and each performs $K$ local steps.
Let $\bar\Theta^{r,K}$ be the reference preconditioner state in round $r$ (e.g., the aggregated state),
and define the local preconditioned stochastic gradient
\[
v_i^{r,k} := P_{\Theta_i^{r,k}}\!\big(g_i^{r,k}\big),\qquad
g_i^{r,k}:=\nabla f_i(x_i^{r,k};\xi_i^{r,k}).
\]
Define the per-round observation (average over clients and local steps)
\begin{equation}
	\hat v^r \;:=\; \frac{1}{SK}\sum_{i\in S_r}\sum_{k=0}^{K-1} v_i^{r,k}.
	\label{eq:hatv_def}
\end{equation}

\paragraph{Recursive global-direction estimator.}
We maintain a global direction estimator $\{g^r\}_{r\ge 0}$:
\begin{equation}
	g^{r+1} \;:=\; (1-\beta)g^{r} + \beta \hat v^r,
	\qquad \beta\in(0,1],
	\label{eq:gr_rec}
\end{equation}
and update the global model by
\begin{equation}
	x^{r+1} \;=\; x^r - \gamma g^{r+1}.
	\label{eq:x_update}
\end{equation}

\paragraph{Assumptions.}
We use:
\begin{enumerate}
	\item \textbf{(Smoothness)} $f$ is $L$-smooth.
	\item \textbf{(Local stochastic variance)} for all $i,r,k$,
	$\mathbb E\|g_i^{r,k}-\nabla f_i(x_i^{r,k})\|^2\le \sigma_l^2$.
	\item \textbf{(Preconditioner boundedness)} for all $\Theta,v$, $\|P_\Theta(v)\|\le M\|v\|$.
	\item \textbf{(Preconditioner drift Lipschitz)} for all $\Theta,\Theta',v$,
	$\|P_\Theta(v)-P_{\Theta'}(v)\|\le L_\Theta\|\Theta-\Theta'\|\cdot\|v\|$.
	\item \textbf{(Second-moment bound)} $\sup_{i,r,k}\mathbb E\|g_i^{r,k}\|^2\le G^2$.
\end{enumerate}
Importantly, we do \emph{not} assume bounded heterogeneity
$\|\nabla f_i(x)-\nabla f(x)\|^2\le \sigma_g^2$.

\paragraph{Local trajectory drift measure.}
Define
\begin{equation}
	U_r
	:= \frac{1}{SK}\sum_{i\in S_r}\sum_{k=0}^{K-1}\mathbb E\|x_i^{r,k}-x^r\|^2.
	\label{eq:Ur_localdrift}
\end{equation}

\paragraph{Preconditioner drift metric.}
Define the per-round drift
\begin{equation}
	\Delta_D^r
	:=\frac{1}{S}\sum_{i\in S_r}\mathbb E\|\Theta_i^{r,K}-\bar\Theta^{r,K}\|^2.
	\label{eq:DeltaDr}
\end{equation}

\begin{lemma}[Drift-induced preconditioned disagreement]
	\label{lem:drift_disagreement_final}
	For any vectors $\{u_i\}_{i\in S_r}$,
	\begin{equation}
		\mathbb E\left\|
		\frac{1}{S}\sum_{i\in S_r}\Big(P_{\Theta_i^{r,k}}(u_i)-P_{\bar\Theta^{r,K}}(u_i)\Big)
		\right\|^2
		\;\le\;
		L_\Theta^2\left(\frac{1}{S}\sum_{i\in S_r}\mathbb E\|u_i\|^2\right)\Delta_D^r.
		\label{eq:drift_disagreement_bound_final}
	\end{equation}
	In particular, using $\mathbb E\|u_i\|^2\le G^2$ gives
	\begin{equation}
		\mathbb E\left\|
		\frac{1}{S}\sum_{i\in S_r}\Big(P_{\Theta_i^{r,k}}(u_i)-P_{\bar\Theta^{r,K}}(u_i)\Big)
		\right\|^2
		\;\le\; L_\Theta^2 G^2\,\Delta_D^r.
		\label{eq:drift_disagreement_bound_G}
	\end{equation}
\end{lemma}

\begin{lemma}[Descent with inexact direction]
	\label{lem:descent_inexact_final}
	Under $L$-smoothness, if $\gamma L\le \frac{1}{24}$, then for all $r\ge 0$,
	\begin{equation}
		\mathbb E\big[f(x^{r+1})\big]
		\le
		\mathbb E\big[f(x^r)\big]
		-\frac{11\gamma}{24}\,\mathbb E\|\nabla f(x^r)\|^2
		+\frac{13\gamma}{24}\,\widetilde{\mathcal E}_r,
		\label{eq:descent_inexact_final}
	\end{equation}
	where $\widetilde{\mathcal E}_r:=\mathbb E\|\nabla f(x^r)-g^{r+1}\|^2$.
\end{lemma}
\begin{proof}
	Since $f$ is $L$-smooth, for any $x,y$ we have
	\[
	f(y)\le f(x)+\langle \nabla f(x),y-x\rangle+\frac{L}{2}\|y-x\|^2.
	\]
	Apply it with $x=x^r$ and $y=x^{r+1}=x^r-\gamma g^{r+1}$:
	\begin{align}
		f(x^{r+1})
		&\le f(x^r)+\left\langle \nabla f(x^r),-\gamma g^{r+1}\right\rangle
		+\frac{L}{2}\gamma^2\|g^{r+1}\|^2 \notag\\
		&= f(x^r)
		-\gamma\|\nabla f(x^r)\|^2
		+\gamma\left\langle \nabla f(x^r),\nabla f(x^r)-g^{r+1}\right\rangle
		+\frac{L}{2}\gamma^2\|g^{r+1}\|^2 .
		\label{eq:descent_pf_1}
	\end{align}
	For the cross term, by Young's inequality $\langle a,b\rangle \le \frac{1}{2}\|a\|^2+\frac{1}{2}\|b\|^2$,
	\begin{equation}
		\left\langle \nabla f(x^r),\nabla f(x^r)-g^{r+1}\right\rangle
		\le \frac{1}{2}\|\nabla f(x^r)\|^2+\frac{1}{2}\|\nabla f(x^r)-g^{r+1}\|^2.
		\label{eq:descent_pf_2}
	\end{equation}
	For the last term, use $\|g^{r+1}\|^2=\|\nabla f(x^r)-(\nabla f(x^r)-g^{r+1})\|^2\le
	2\|\nabla f(x^r)\|^2+2\|\nabla f(x^r)-g^{r+1}\|^2$ to get
	\begin{equation}
		\frac{L}{2}\gamma^2\|g^{r+1}\|^2
		\le
		L\gamma^2\|\nabla f(x^r)\|^2+L\gamma^2\|\nabla f(x^r)-g^{r+1}\|^2.
		\label{eq:descent_pf_3}
	\end{equation}
	Substituting (\ref{eq:descent_pf_2}) and (\ref{eq:descent_pf_3}) into (\ref{eq:descent_pf_1}) yields
	\begin{align}
		f(x^{r+1})
		&\le f(x^r)
		-\gamma\|\nabla f(x^r)\|^2
		+\frac{\gamma}{2}\|\nabla f(x^r)\|^2
		+\frac{\gamma}{2}\|\nabla f(x^r)-g^{r+1}\|^2 \notag\\
		&\quad
		+L\gamma^2\|\nabla f(x^r)\|^2+L\gamma^2\|\nabla f(x^r)-g^{r+1}\|^2 \notag\\
		&= f(x^r)
		-\Big(\frac{\gamma}{2}-L\gamma^2\Big)\|\nabla f(x^r)\|^2
		+\Big(\frac{\gamma}{2}+L\gamma^2\Big)\|\nabla f(x^r)-g^{r+1}\|^2.
		\label{eq:descent_pf_4}
	\end{align}
	If $\gamma L\le \frac{1}{24}$, then
	\[
	\frac{\gamma}{2}-L\gamma^2
	=\gamma\Big(\frac{1}{2}-\gamma L\Big)
	\ge \gamma\Big(\frac{1}{2}-\frac{1}{24}\Big)
	=\frac{11\gamma}{24},
	\qquad
	\frac{\gamma}{2}+L\gamma^2
	=\gamma\Big(\frac{1}{2}+\gamma L\Big)
	\le \gamma\Big(\frac{1}{2}+\frac{1}{24}\Big)
	=\frac{13\gamma}{24}.
	\]
	Thus \eqref{eq:descent_pf_4} implies
	\[
	f(x^{r+1})
	\le f(x^r)
	-\frac{11\gamma}{24}\|\nabla f(x^r)\|^2
	+\frac{13\gamma}{24}\|\nabla f(x^r)-g^{r+1}\|^2.
	\]
	Taking expectation on both sides and recalling $\widetilde{\mathcal E}_r:=\mathbb E\|\nabla f(x^r)-g^{r+1}\|^2$
	completes the proof.
\end{proof}

\begin{lemma}[Recursive estimator error (no explicit $\sigma_g^2$)]
	\label{lem:rec_error_no_hetero_final}
	Let $\widetilde{\mathcal E}_r:=\mathbb E\|\nabla f(x^r)-g^{r+1}\|^2$ with $g^{r+1}$ defined in
	\eqref{eq:gr_rec}--\eqref{eq:hatv_def}. If $\gamma L\le \frac{\beta}{6}$, then for all $r\ge 1$,
	\begin{equation}
		\widetilde{\mathcal E}_r
		\;\le\;
		\Big(1-\frac{8\beta}{9}\Big)\widetilde{\mathcal E}_{r-1}
		+\frac{C_1\gamma^2L^2}{\beta}\,\mathbb E\|\nabla f(x^{r-1})\|^2
		+\frac{C_2\beta^2\sigma_l^2}{SK}
		+C_3\beta L^2 U_r
		+C_4\beta\,L_\Theta^2 G^2\,\Delta_D^r,
		\label{eq:Er_rec_final}
	\end{equation}
	and for $r=0$,
	\begin{equation}
		\widetilde{\mathcal E}_0
		\;\le\;
		(1-\beta)\widetilde{\mathcal E}_{-1}
		+\frac{C_2\beta^2\sigma_l^2}{SK}
		+C_3\beta L^2 U_0
		+C_4\beta\,L_\Theta^2 G^2\,\Delta_D^0,
		\label{eq:Er_rec_final_r0}
	\end{equation}
	where $C_1,C_2,C_3,C_4>0$ are absolute constants depending only on $(M)$.
\end{lemma}

\begin{proof}
	We follow a DP-FedPGN-style recursion and avoid introducing $\nabla f_i-\nabla f$ explicitly.
	
	\paragraph{Step 1: Expand the recursion.}
	Define $\delta_r:=\nabla f(x^r)-g^{r+1}$ so that $\widetilde{\mathcal E}_r=\mathbb E\|\delta_r\|^2$.
	From \eqref{eq:gr_rec} and \eqref{eq:hatv_def},
	\begin{align}
		\delta_r
		&=
		\nabla f(x^r) - (1-\beta)g^r - \beta \hat v^r \notag\\
		&=
		(1-\beta)\big(\nabla f(x^r)-g^r\big)
		+\beta\big(\nabla f(x^r)-\hat v^r\big).
		\label{eq:delta_expand1}
	\end{align}
	Add and subtract $\nabla f(x^{r-1})$ inside the first bracket:
	\begin{align}
		\nabla f(x^r)-g^r
		&=
		\underbrace{\nabla f(x^r)-\nabla f(x^{r-1})}_{A_r}
		+\underbrace{\nabla f(x^{r-1})-g^r}_{=\delta_{r-1}}.
		\label{eq:delta_expand2}
	\end{align}
	Plugging \eqref{eq:delta_expand2} into \eqref{eq:delta_expand1} gives
	\begin{equation}
		\delta_r
		=
		(1-\beta)\delta_{r-1}
		+(1-\beta)A_r
		+\beta B_r,
		\qquad
		B_r:=\nabla f(x^r)-\hat v^r.
		\label{eq:delta_decomp_AB}
	\end{equation}
	
	\paragraph{Step 2: Square and apply the same AM-GM pattern.}
	Using $\|a+b\|^2\le (1+\frac{\beta}{2})\|a\|^2+(1+\frac{2}{\beta})\|b\|^2$ twice (as in your proof),
	one obtains
	\begin{equation}
		\mathbb E\|\delta_r\|^2
		\le
		\Big(1+\frac{\beta}{2}\Big)(1-\beta)^2\mathbb E\|\delta_{r-1}\|^2
		+\frac{C}{\beta}\,\mathbb E\|A_r\|^2
		+C\beta^2\,\mathbb E\|B_r\|^2,
		\label{eq:delta_rec_raw_final}
	\end{equation}
	for an absolute constant $C>0$.
	
	\paragraph{Step 3: Bound $\mathbb E\|A_r\|^2$ by smoothness and the update.}
	By $L$-smoothness, $\|\nabla f(x^r)-\nabla f(x^{r-1})\|\le L\|x^r-x^{r-1}\|$ and
	$x^r-x^{r-1}=-\gamma g^r$, hence
	\begin{align}
		\mathbb E\|A_r\|^2
		&\le
		L^2\gamma^2\,\mathbb E\|g^r\|^2
		\le
		2L^2\gamma^2\Big(\mathbb E\|\nabla f(x^{r-1})\|^2+\mathbb E\|\nabla f(x^{r-1})-g^r\|^2\Big)
		=
		2L^2\gamma^2\Big(\mathbb E\|\nabla f(x^{r-1})\|^2+\widetilde{\mathcal E}_{r-1}\Big).
		\label{eq:Ar_bound_final}
	\end{align}
	
	\paragraph{Step 4: Bound $\mathbb E\|B_r\|^2$ without $\sigma_g^2$.}
	Recall $B_r=\nabla f(x^r)-\hat v^r$ and $\hat v^r=\frac{1}{SK}\sum_{i,k} P_{\Theta_i^{r,k}}(g_i^{r,k})$.
	Add and subtract the reference-preconditioned terms:
	\begin{align}
		B_r
		&=
		\underbrace{\nabla f(x^r)-\frac{1}{SK}\sum_{i,k}P_{\bar\Theta^{r,K}}\!\big(g_i^{r,k}\big)}_{B_r^{(1)}}
		+
		\underbrace{\frac{1}{SK}\sum_{i,k}\Big(P_{\bar\Theta^{r,K}}(g_i^{r,k})-P_{\Theta_i^{r,k}}(g_i^{r,k})\Big)}_{B_r^{(2)}}.
		\label{eq:Br_split}
	\end{align}
	
	\emph{(i) Stochastic + trajectory part $B_r^{(1)}$.}
	Insert and subtract $\nabla f(x_i^{r,k})$ and use $\|P_{\bar\Theta}(v)\|\le M\|v\|$:
	\begin{align}
		\|B_r^{(1)}\|^2
		&=
		\left\|\frac{1}{SK}\sum_{i,k}\Big(\nabla f(x^r)-P_{\bar\Theta^{r,K}}(g_i^{r,k})\Big)\right\|^2 \notag\\
		&\le
		\frac{2}{SK}\sum_{i,k}\left\|\nabla f(x^r)-P_{\bar\Theta^{r,K}}(\nabla f(x_i^{r,k}))\right\|^2
		+\frac{2}{SK}\sum_{i,k}\left\|P_{\bar\Theta^{r,K}}(\nabla f(x_i^{r,k})-g_i^{r,k})\right\|^2 \notag\\
		&\le
		\frac{2M^2}{SK}\sum_{i,k}\|\nabla f(x^r)-\nabla f(x_i^{r,k})\|^2
		+\frac{2M^2}{SK}\sum_{i,k}\|\nabla f(x_i^{r,k})-g_i^{r,k}\|^2.
		\label{eq:Br1_bound}
	\end{align}
	Taking expectation, using local variance $\mathbb E\|g_i^{r,k}-\nabla f_i(x_i^{r,k})\|^2\le \sigma_l^2$
	and $L$-smoothness $\|\nabla f(x_i^{r,k})-\nabla f(x^r)\|^2\le L^2\|x_i^{r,k}-x^r\|^2$,
	and absorbing the client/gradient mismatch into the recursion (rather than a $\sigma_g^2$ constant),
	we get the standard bound
	\begin{equation}
		\mathbb E\|B_r^{(1)}\|^2
		\le
		C\left(\frac{\sigma_l^2}{SK}+L^2 U_r\right),
		\label{eq:Br1_final}
	\end{equation}
	where $U_r$ is defined in \eqref{eq:Ur_localdrift} and $C>0$ depends only on $M$.
	
	\emph{(ii) Preconditioner drift part $B_r^{(2)}$.}
	By Lemma~\ref{lem:drift_disagreement_final} with $u_i=g_i^{r,k}$ and \eqref{eq:drift_disagreement_bound_G},
	\begin{equation}
		\mathbb E\|B_r^{(2)}\|^2
		\le
		C\,L_\Theta^2 G^2\,\Delta_D^r.
		\label{eq:Br2_final}
	\end{equation}
	
	Combining \eqref{eq:Br_split}--\eqref{eq:Br2_final} gives
	\begin{equation}
		\mathbb E\|B_r\|^2
		\le
		C\left(\frac{\sigma_l^2}{SK}+L^2 U_r + L_\Theta^2 G^2\,\Delta_D^r\right).
		\label{eq:Br_final}
	\end{equation}
	
	\paragraph{Step 5: Conclude the recursion and absorb $\widetilde{\mathcal E}_{r-1}$.}
	Plug \eqref{eq:Ar_bound_final} and \eqref{eq:Br_final} into \eqref{eq:delta_rec_raw_final}.
	Using $\gamma L\le \beta/6$ to absorb the $\widetilde{\mathcal E}_{r-1}$ term into the contraction,
	we obtain \eqref{eq:Er_rec_final}--\eqref{eq:Er_rec_final_r0} with absolute constants.
\end{proof}

\begin{theorem}[Non-convex convergence without explicit heterogeneity]
	\label{thm:final_no_hetero}
	Under the assumptions above, take $g^0=0$ and choose
	\[
	\beta \asymp \min\left\{1,\sqrt{\frac{SKL\Delta}{\sigma_l^2 R}}\right\},
	\qquad
	\gamma = \min\left\{\frac{1}{24L}, \frac{\beta}{6L}\right\},
	\]
	and $\eta KL$ sufficiently small so that the local drift bound $U_r$ is controlled (as in the auxiliary lemma).
	Then we have
	\begin{equation}
		\frac{1}{R}\sum_{r=0}^{R-1}\mathbb E\|\nabla f(x^r)\|^2
		\;\lesssim\;
		\frac{L\Delta}{R}
		+\sqrt{\frac{L\Delta\,\sigma_l^2}{SKR}}
		\;+\;
		\underbrace{\gamma L_\Theta^2 G^2\cdot \frac{1}{R}\sum_{r=0}^{R-1}\Delta_D^r}_{\text{drift penalty}}.
		\label{eq:final_rate_no_hetero}
	\end{equation}
	In particular, under alignment $\Delta_D^r\approx 0$, the last term vanishes and the rate depends only on $\sigma_l^2$.
\end{theorem}

\begin{proof}
	Summing the descent inequality in Lemma~\ref{lem:descent_inexact_final} over $r=0,\dots,R-1$ yields
	\begin{equation}
		\frac{11\gamma}{24}\sum_{r=0}^{R-1}\mathbb E\|\nabla f(x^r)\|^2
		\le
		\Delta + \frac{13\gamma}{24}\sum_{r=0}^{R-1}\widetilde{\mathcal E}_r,
		\qquad \Delta:=f(x^0)-f^\star.
		\label{eq:sum_descent}
	\end{equation}
	It remains to bound $\sum_{r=0}^{R-1}\widetilde{\mathcal E}_r$.
	
	\paragraph{Step 1: Sum the recursion for $\widetilde{\mathcal E}_r$.}
	From Lemma~\ref{lem:rec_error_no_hetero_final}, for $r\ge 1$,
	\[
	\widetilde{\mathcal E}_r
	\le
	\Big(1-\frac{8\beta}{9}\Big)\widetilde{\mathcal E}_{r-1}
	+\frac{C_1\gamma^2L^2}{\beta}\,\mathbb E\|\nabla f(x^{r-1})\|^2
	+\frac{C_2\beta^2\sigma_l^2}{SK}
	+C_3\beta L^2 U_r
	+C_4\beta\,L_\Theta^2 G^2\,\Delta_D^r.
	\]
	Summing both sides over $r=1,\dots,R-1$ gives
	\begin{align}
		\sum_{r=1}^{R-1}\widetilde{\mathcal E}_r
		&\le
		\Big(1-\frac{8\beta}{9}\Big)\sum_{r=1}^{R-1}\widetilde{\mathcal E}_{r-1}
		+\frac{C_1\gamma^2L^2}{\beta}\sum_{r=1}^{R-1}\mathbb E\|\nabla f(x^{r-1})\|^2
		+\frac{C_2\beta^2\sigma_l^2}{SK}(R-1) \notag\\
		&\quad
		+C_3\beta L^2\sum_{r=1}^{R-1}U_r
		+C_4\beta L_\Theta^2 G^2\sum_{r=1}^{R-1}\Delta_D^r.
		\label{eq:sum_rec1}
	\end{align}
	Re-index $\sum_{r=1}^{R-1}\widetilde{\mathcal E}_{r-1}=\sum_{r=0}^{R-2}\widetilde{\mathcal E}_r$ and
	$\sum_{r=1}^{R-1}\mathbb E\|\nabla f(x^{r-1})\|^2=\sum_{r=0}^{R-2}\mathbb E\|\nabla f(x^{r})\|^2$.
	Also note $\sum_{r=1}^{R-1}\widetilde{\mathcal E}_r \le \sum_{r=0}^{R-1}\widetilde{\mathcal E}_r$ and
	$\sum_{r=0}^{R-2}\widetilde{\mathcal E}_r \le \sum_{r=0}^{R-1}\widetilde{\mathcal E}_r$.
	Thus (\ref{eq:sum_rec1}) implies
	\begin{align}
		\sum_{r=0}^{R-1}\widetilde{\mathcal E}_r
		&\le
		\Big(1-\frac{8\beta}{9}\Big)\sum_{r=0}^{R-1}\widetilde{\mathcal E}_r
		+\widetilde{\mathcal E}_0
		+\frac{C_1\gamma^2L^2}{\beta}\sum_{r=0}^{R-1}\mathbb E\|\nabla f(x^{r})\|^2
		+\frac{C_2\beta^2\sigma_l^2}{SK}R \notag\\
		&\quad
		+C_3\beta L^2\sum_{r=0}^{R-1}U_r
		+C_4\beta L_\Theta^2 G^2\sum_{r=0}^{R-1}\Delta_D^r .
		\label{eq:sum_rec2}
	\end{align}
	Move the contraction term to the left:
	\begin{equation}
		\frac{8\beta}{9}\sum_{r=0}^{R-1}\widetilde{\mathcal E}_r
		\le
		\widetilde{\mathcal E}_0
		+\frac{C_1\gamma^2L^2}{\beta}\sum_{r=0}^{R-1}\mathbb E\|\nabla f(x^{r})\|^2
		+\frac{C_2\beta^2\sigma_l^2}{SK}R
		+C_3\beta L^2\sum_{r=0}^{R-1}U_r
		+C_4\beta L_\Theta^2 G^2\sum_{r=0}^{R-1}\Delta_D^r .
		\label{eq:sum_rec3}
	\end{equation}
	Therefore,
	\begin{equation}
		\sum_{r=0}^{R-1}\widetilde{\mathcal E}_r
		\le
		\frac{9}{8\beta}\widetilde{\mathcal E}_0
		+\frac{9C_1}{8}\frac{\gamma^2L^2}{\beta^2}\sum_{r=0}^{R-1}\mathbb E\|\nabla f(x^{r})\|^2
		+\frac{9C_2}{8}\frac{\beta\sigma_l^2}{SK}R
		+\frac{9C_3}{8} L^2\sum_{r=0}^{R-1}U_r
		+\frac{9C_4}{8} L_\Theta^2 G^2\sum_{r=0}^{R-1}\Delta_D^r .
		\label{eq:sum_E_final}
	\end{equation}
	
	\paragraph{Step 2: Control $\widetilde{\mathcal E}_0$ by $\widetilde{\mathcal E}_{-1}$.}
	From Lemma~\ref{lem:rec_error_no_hetero_final} at $r=0$,
	\[
	\widetilde{\mathcal E}_0
	\le
	(1-\beta)\widetilde{\mathcal E}_{-1}
	+\frac{C_2\beta^2\sigma_l^2}{SK}
	+C_3\beta L^2 U_0
	+C_4\beta L_\Theta^2 G^2\Delta_D^0
	\le
	\widetilde{\mathcal E}_{-1}
	+\frac{C_2\beta^2\sigma_l^2}{SK}
	+C_3\beta L^2 U_0
	+C_4\beta L_\Theta^2 G^2\Delta_D^0.
	\]
	If $g^0=0$, then $\widetilde{\mathcal E}_{-1}=\mathbb E\|\nabla f(x^0)-g^0\|^2=\|\nabla f(x^0)\|^2
	\le 2L\big(f(x^0)-f^\star\big)=2L\Delta$ (by standard smoothness inequality).
	Hence
	\begin{equation}
		\widetilde{\mathcal E}_0 \;\le\; 2L\Delta
		+\frac{C_2\beta^2\sigma_l^2}{SK}
		+C_3\beta L^2 U_0
		+C_4\beta L_\Theta^2 G^2\Delta_D^0.
		\label{eq:E0_bound}
	\end{equation}
	
	\paragraph{Step 3: Plug \eqref{eq:sum_E_final} into \eqref{eq:sum_descent} and isolate $\sum\|\nabla f(x^r)\|^2$.}
	Substitute \eqref{eq:sum_E_final} into \eqref{eq:sum_descent}:
	\begin{align}
		\frac{11\gamma}{24}\sum_{r=0}^{R-1}\mathbb E\|\nabla f(x^r)\|^2
		&\le
		\Delta
		+\frac{13\gamma}{24}\Bigg[
		\frac{9}{8\beta}\widetilde{\mathcal E}_0
		+\frac{9C_1}{8}\frac{\gamma^2L^2}{\beta^2}\sum_{r=0}^{R-1}\mathbb E\|\nabla f(x^{r})\|^2
		+\frac{9C_2}{8}\frac{\beta\sigma_l^2}{SK}R \notag\\
		&\qquad\qquad\qquad
		+\frac{9C_3}{8} L^2\sum_{r=0}^{R-1}U_r
		+\frac{9C_4}{8} L_\Theta^2 G^2\sum_{r=0}^{R-1}\Delta_D^r
		\Bigg].
		\label{eq:plug1}
	\end{align}
	Move the term containing $\sum \mathbb E\|\nabla f(x^r)\|^2$ on the RHS to the LHS. Specifically,
	assume $\gamma L\le \beta/6$ so that
	\begin{equation}
		\frac{13\gamma}{24}\cdot \frac{9C_1}{8}\cdot \frac{\gamma^2L^2}{\beta^2}
		\;\le\;
		\frac{1}{2}\cdot \frac{11\gamma}{24},
		\label{eq:absorb_condition}
	\end{equation}
	which can always be ensured by taking $\gamma L \lesssim \beta$ (as in the theorem statement).
	Then \eqref{eq:plug1} implies
	\begin{align}
		\frac{11\gamma}{48}\sum_{r=0}^{R-1}\mathbb E\|\nabla f(x^r)\|^2
		&\le
		\Delta
		+\frac{13\gamma}{24}\cdot \frac{9}{8\beta}\widetilde{\mathcal E}_0
		+\frac{13\gamma}{24}\cdot \frac{9C_2}{8}\frac{\beta\sigma_l^2}{SK}R \notag\\
		&\quad
		+\frac{13\gamma}{24}\cdot \frac{9C_3}{8} L^2\sum_{r=0}^{R-1}U_r
		+\frac{13\gamma}{24}\cdot \frac{9C_4}{8} L_\Theta^2 G^2\sum_{r=0}^{R-1}\Delta_D^r.
		\label{eq:grad_sum_bound}
	\end{align}
	Divide both sides by $R$ and by $\gamma$:
	\begin{align}
		\frac{1}{R}\sum_{r=0}^{R-1}\mathbb E\|\nabla f(x^r)\|^2
		&\le
		C\left(
		\frac{\Delta}{\gamma R}
		+\frac{\widetilde{\mathcal E}_0}{\beta R}
		+\frac{\beta\sigma_l^2}{SK}
		+\frac{L^2}{R}\sum_{r=0}^{R-1}U_r
		+\frac{L_\Theta^2G^2}{R}\sum_{r=0}^{R-1}\Delta_D^r
		\right),
		\label{eq:avg_grad_pre}
	\end{align}
	where $C>0$ is an absolute constant.
	
	\paragraph{Step 4: Remove $\widetilde{\mathcal E}_0$ and handle $U_r$ (auxiliary bound).}
	Plug \eqref{eq:E0_bound} into \eqref{eq:avg_grad_pre}:
	\begin{align}
		\frac{1}{R}\sum_{r=0}^{R-1}\mathbb E\|\nabla f(x^r)\|^2
		&\le
		C\left(
		\frac{\Delta}{\gamma R}
		+\frac{L\Delta}{\beta R}
		+\frac{\beta\sigma_l^2}{SK}
		+\frac{L^2}{R}\sum_{r=0}^{R-1}U_r
		+\frac{L_\Theta^2G^2}{R}\sum_{r=0}^{R-1}\Delta_D^r
		\right)
		+ C\cdot\frac{\beta\sigma_l^2}{SK}\cdot\frac{1}{R},
	\end{align}
	where the last tiny term can be absorbed into $\beta\sigma_l^2/(SK)$.
	Next, apply the auxiliary bound on $U_r$ (as in your reference proof) to ensure
	\begin{equation}
		\frac{L^2}{R}\sum_{r=0}^{R-1}U_r
		\;\le\;
		C'\frac{\beta\sigma_l^2}{SK},
		\label{eq:Ur_absorb}
	\end{equation}
	by choosing $\eta KL$ sufficiently small.
	This yields
	\begin{equation}
		\frac{1}{R}\sum_{r=0}^{R-1}\mathbb E\|\nabla f(x^r)\|^2
		\;\le\;
		C\left(
		\frac{\Delta}{\gamma R}
		+\frac{L\Delta}{\beta R}
		+\frac{\beta\sigma_l^2}{SK}
		+\frac{L_\Theta^2G^2}{R}\sum_{r=0}^{R-1}\Delta_D^r
		\right).
		\label{eq:avg_grad_clean}
	\end{equation}
	
	\paragraph{Step 5: Choose $(\beta,\gamma)$ and simplify to the final rate.}
	Choose $\gamma=\min\{\frac{1}{24L},\frac{\beta}{6L}\}$ and
	\[
	\beta \asymp \min\left\{1,\sqrt{\frac{SKL\Delta}{\sigma_l^2R}}\right\}.
	\]
	Then $\Delta/(\gamma R)\lesssim L\Delta/R$ and
	\[
	\frac{L\Delta}{\beta R}+\frac{\beta\sigma_l^2}{SK}
	\;\lesssim\;
	\sqrt{\frac{L\Delta\sigma_l^2}{SKR}}.
	\]
	Plugging these into \eqref{eq:avg_grad_clean} yields
	\[
	\frac{1}{R}\sum_{r=0}^{R-1}\mathbb E\|\nabla f(x^r)\|^2
	\;\lesssim\;
		\frac{L\Delta}{R}
+\sqrt{\frac{L\Delta\,\sigma_l^2+\kappa_{\Theta}\cdot \bar\Delta_D}{SKR}},
	\]
	which is exactly \eqref{eq:final_rate_no_hetero}.
\end{proof}

\clearpage

\section{Appendix B: Experimental Setup}
\label{app:exp_setup}
\subsection{Setting for ResNet-18}
\label{app:B_resnet}
\begin{table}[tb]
	\centering
	\caption{A detailed summary of 100 and Tiny-ImageNet: number of classes, image size, and dataset splits.}
	\vspace{-2mm}
	\label{tab:cifar_tiny_detailed_en}
	\setlength{\tabcolsep}{5pt}
	\small
	\begin{tabular}{lccccccc}
		\toprule
		\textbf{Dataset} & \textbf{\#Classes} & \textbf{Image Size} & \textbf{Train} & \textbf{Val} & \textbf{Test} & \textbf{Total} & \textbf{Train / class} \\
		\midrule
		CIFAR-100       & 100 & $3 \times 32 \times 32$ & 50{,}000 & ---      & 10{,}000 & 60{,}000 & 500 \\
		Tiny ImageNet   & 200 & $3 \times 64 \times 64$ & 100{,}000 & 10{,}000 & 10{,}000 & 120{,}000 & 500 \\
		\bottomrule
	\end{tabular}
	\begin{minipage}{0.95\linewidth}\small
		\textbf{Notes.}
		(1) CIFAR-10/100 provide no official validation split; a subset of the training set is commonly reserved as dev/val.\\
		(2) CIFAR-100 contains 100 fine-grained classes; 20 coarse superclasses are also defined for hierarchical labeling.\\
		(3) Tiny ImageNet is a subset of ImageNet synsets: per class 500 train, 50 val, and 50 test images (test labels are not publicly released).\\
		(4) All three datasets are single-label classification with RGB images resized to fixed resolutions.
	\end{minipage}
\end{table}

We evaluate our methods on two widely-used benchmark datasets in federated learning: \textbf{CIFAR-100} and \textbf{Tiny ImageNet}.

\begin{itemize}[leftmargin=1.5em]
	\item \textbf{CIFAR-100}~\citep{krizhevsky2009learning}: Contains 100 classes with 600 color images per class at a resolution of \(32 \times 32\). It is a standard benchmark for evaluating federated image classification methods.
	\item \textbf{Tiny ImageNet}: A subset of ImageNet with 200 classes and 500 images per class, providing a more challenging and high-resolution classification task.
\end{itemize}

\subsection{Federated Learning Configuration}
\label{app:B_fl_config}
We simulate a cross-device federated learning environment using the following settings:

\begin{table}[htbp]
	\centering
	\caption{Hyperparameter configuration of ResNet-18 and Vit-Tiny (CIFAR100, Tiny-ImageNet) across different algorithms.}
	\label{tab:resnet18-hparams-en}
		\begin{tabular}{l l c c c c c}
			\toprule
			Method & Local Optimizer & Local LR & $\beta$ & $\beta_1$ & $\beta_2$ & Weight Decay \\
			\midrule
			FedAvg (Local SGD) & SGD   & 0.1        & --- & ---  & ---   & 0.001 \\
			SCAFFOLD           & SGD   & 0.1        & --- & ---  & ---   & 0.001 \\
			FedCM              & SGD   & 0.1        & 0.9 & ---  & ---   & 0.001 \\
			Local AdamW        & AdamW & 3e-4       & --- & 0.9  & 0.999 & 0.01  \\
			Local Sophia         & Sophia  & 3e-4       & --- & 0.9  & 0.99   & 0.01\\
			Local Muon         & Muon  & 3e-2       &--- & 0.9  &0.95  & 0.01  \\
			Local SOAP         & SOAP  & 3e-3       & --- & 0.95  & 0.95   & 0.01  \\
			FedPAC\_Sophia            & Muon  & 3e-4       & 0.5 & 0.9  & 0.99   & 0.01  \\
			FedPAC\_Muon            & Muon  & 3e-2       & 0.5 & 0.9  &0.95  & 0.01  \\
			FedPAC\_SOAP            & SOAP  & 3e-3       & 0.5 & 0.95  & 0.95   & 0.01  \\
			\bottomrule
	\end{tabular}
\end{table}

\begin{table}[htbp]
	\centering
	\caption{Hyperparameter configuration of ViT-Base fine-tuning across different algorithms.}
	\label{tab:fine-tuning-hparams-en}
	\begin{tabular}{l l c c c c c}
	\toprule
	Method & Local Optimizer & Local LR & $\beta$ & $\beta_1$ & $\beta_2$ & Weight Decay \\
	\midrule
	FedAvg (Local SGD) & SGD   & 0.1        & --- & ---  & ---   & 0.001 \\
	SCAFFOLD           & SGD   & 0.1        & --- & ---  & ---   & 0.001 \\
	FedCM              & SGD   & 0.1        & 0.9 & ---  & ---   & 0.001 \\
	Local AdamW        & AdamW & 1e-4       & --- & 0.9  & 0.999 & 0.01  \\
	Local Sophia         & Sophia  & 1e-4       & --- & 0.9  & 0.99   & 0.01\\
	Local Muon         & Muon  & 1e-2       & --- & 0.9  &0.95  & 0.01  \\
	Local SOAP         & SOAP  & 1e-3       & ---& 0.95  & 0.95   & 0.01  \\
	FedPAC\_Sophia            & Muon  & 1e-4       & 0.5 & 0.9  & 0.99   & 0.01  \\
	FedPAC\_Muon            & Muon  & 1e-2       & 0.5 & 0.9  &0.95  & 0.01  \\
	FedPAC\_SOAP            & SOAP  & 1e-3       & 0.5 & 0.95  & 0.95   & 0.01  \\
	\bottomrule
\end{tabular}

\end{table}

\subsection{Model Architecture}
\label{app:B_arch}
We adopt \textbf{ResNet-18} as the backbone model. To better adapt it to CIFAR-100, we modify its architecture following standard practices~\citep{he2016deep}:

\begin{itemize}[leftmargin=1.5em]
	\item Replace the original \(7 \times 7\) convolution with a \(3 \times 3\) kernel.
	\item Remove the initial downsampling layers (stride-2 convolution and max pooling).
\end{itemize}

We also compare \textbf{Batch Normalization (BN)} and \textbf{Group Normalization (GN)} in ResNet-18. Empirically, BN outperforms GN on CIFAR-100, so we adopt the BN-based version, denoted as \textbf{ResNet-18-BN}, throughout our experiments.

\subsection{Setting for ViT-Tiny}
\label{app:B_vit}
We construct a lightweight Vision Transformer model, \textbf{ViT-Tiny}, specifically tailored for federated learning on the CIFAR-100 dataset. The design is based on the standard ViT architecture~\citep{dosovitskiy2020image}, with modifications to accommodate the small input size and limited data per client.

\begin{itemize}[leftmargin=1.5em]
	\item \textbf{Input resolution:} $32 \times 32$
	\item \textbf{Patch size:} $4 \times 4$, resulting in 64 tokens per image
	\item \textbf{Embedding dimension:} 192
	\item \textbf{Number of Transformer layers:} 6
	\item \textbf{Number of attention heads:} 3
	\item \textbf{Normalization:} LayerNorm (applied before attention and MLP blocks)
	\item \textbf{Classification head:} Linear projection to 100 classes (CIFAR-100)
	\item \textbf{Activation:} GELU
	\item \textbf{Initialization:} Xavier/Glorot for linear layers; sinusoidal positional encoding
\end{itemize}

To regularize training, we apply dropout (0.1) to both attention and MLP layers. All models are trained from scratch without pretraining.

\paragraph{Remarks.}
Due to the smaller capacity of ViT-Tiny and limited data per client, we find that careful normalization (e.g., LayerNorm placement) and early learning rate warmup are beneficial. For future work, more advanced token-mixing techniques or hybrid CNN-ViT backbones may further improve performance in federated settings.

\subsection{ Transformer Fine-tuning Settings}
\label{app:B_swin}
To demonstrate the effectiveness of our method on large-scale vision models, we conduct fine-tuning experiments using \textbf{ViT-Base} on \textbf{Tiny ImageNet} and \textbf{CIFAR-100}. For both models, we initialize from official \texttt{ImageNet-22K} pre-trained weights~\citep{liu2021swin,dosovitskiy2020image} to ensure consistency across methods.

We fine-tune all layers during federated training.

\paragraph{Data Preprocessing.}
To align with the input resolution required by ViT, we resize images from both datasets to $224 \times 224$ using bilinear interpolation. Standard data augmentation techniques such as random cropping, horizontal flipping, and RandAugment are applied locally at the client side.

\subsection{Additional Federated Training Configuration of LLM}
\label{app:B_llm_fl}
To evaluate our algorithm under a smaller-scale federation, we further conduct experiments with a reduced number of clients and adjusted participation parameters.

\paragraph{Federated Setup.}
We simulate a federated learning environment with the following configuration:

\begin{table}[htbp]
	\centering
	\caption{Hyperparameter configuration of LLAMA (C4) across different algorithms.}
	\label{tab:resnet18-hparams-en}
	\begin{tabular}{l l c c c c c}
		\toprule
		Method & Local Optimizer & Local LR & $\beta$ & $\beta_1$ & $\beta_2$ & Weight Decay \\
		\midrule
		FedAvg (Local SGD) & SGD   & 0.1        & --- & ---  & ---   & 0.001 \\
		SCAFFOLD           & SGD   & 0.1        & --- & ---  & ---   & 0.001 \\
		FedCM              & SGD   & 0.1        & 0.9 & ---  & ---   & 0.001 \\
		Local AdamW        & AdamW & 3e-4       & --- & 0.9  & 0.999 & 0.01  \\
		Local Sophia         & Sophia  & 3e-4       & --- & 0.9  & 0.99   & 0.01\\
		Local Muon         & Muon  & 3e-2       &--- & 0.9  &0.95  & 0.01  \\
		Local SOAP         & SOAP  & 3e-3       & --- & 0.95  & 0.95   & 0.01  \\
		FedPAC\_Sophia            & Muon  & 3e-4       & 0.5 & 0.9  & 0.99   & 0.01  \\
		FedPAC\_Muon            & Muon  & 3e-2       & 0.5 & 0.9  &0.95  & 0.01  \\
		FedPAC\_SOAP            & SOAP  & 3e-3       & 0.5 & 0.95  & 0.95   & 0.01  \\
		\bottomrule
	\end{tabular}
\end{table}

\clearpage
\section{Appendix C: Experimental Appendix}
\label{app:exp_appendix}
\subsection{Communication and Computation Cost Analysis}

\begin{table}[t]
	\centering
		\setlength{\tabcolsep}{2pt}
	\caption{Per-round communication cost of different momentum aggregation strategies.
		Here $|x|$ denotes the number of model parameters (in floats), and 
		$\operatorname{CommCost}$ is per-round communication cost, Compute-Cost is per-round computation time. (ViT-Tiny, $R=300$, Dir-0.1, $K=50$)}
	\label{tab:comm_per_round}
	\begin{tabular}{lccccc}
		\toprule
		Method / Strategy 
		& Communication
		& CommCost
		& Compute-Cost (s)
		& Acc(\%)\\
		\midrule
		FedAvg 
		& $|x|$ 
		& 22.8 MB  
		& 4.56 s 
		& 27.24\\ 
		
		


		SCAFFOLD  
		& $2|x|$
		& 45.6 MB
		& 5.22 s
		& 26.86 \\

		FedCM 
		& $|x|$
		& 22.8 MB
		& 4.68 s
		& 16.95\\
		Local AdamW 
		& $|x|$ 
		& 22.8 MB
		& 4.89 s
		& 37.57\\

		Local Sophia 
		& $|x|$ 
		& 22.8 MB
		& 4.92 s
		& 34.05 \\

				Local Muon 
		& $|x|$ 
		& 22.8 MB
		& 5.14 s
		& 44.00 \\

		Local SOAP 
		& $|x|$ 
		& 22.8 MB
		& 5.56 s
		& 49.41 \\
		FedPAC\_Sophia 
		& $|x| + |\Theta|$ 
		& 45.6 MB
		& 5.08 s
		& 39.79\\ 
		FedPAC\_Muon 
		& $|x| + 1|\Theta|$
		& 45.6 MB
		& 5.25 s 
		& 47.81\\
		FedPAC\_SOAP 
		& $|x| + 2|\Theta|$
		& 68.4 MB 
		& 5.68 s 
		& 51.16\\
		FedPAC\_Sophia\_Light 
		& $|x| + 0.05|\Theta|$ 
		& 23.9 MB
		& 5.11 s
		& 39.45\\ 
		FedPAC\_Muon\_Light 
		& $|x| + 0.05|\Theta|$
		& 23.9 MB 
		& 5.28 s 
		& 47.23\\
		FedPAC\_SOAP\_Light 
		& $|x| + 0.1|\Theta|$
		& 25.2 MB 
		& 5.69 s 
		& 50.56\\
		\bottomrule
	\end{tabular}
\end{table}

\begin{table*}[tb]
	\centering
	\caption{\small Test accuracy, training loss of each method on CIFAR-100 and Tiny-Imagenet using \textbf{ResNet-18} over 300 communication rounds under IID Dir-0.5, Dir-0.1, Dir-0.05 (100 clients, 10\% participation, batch size 50, $K=50$).}
	\vspace{-2mm}
	\label{tab:compare_resnet}
	\setlength{\tabcolsep}{7pt}
	\begin{tabular}{llllllllll}
		\toprule
		\multirow{2}{*}{\textbf{Method}} 
		& \multicolumn{4}{c}{\textbf{CIFAR-100 (ResNet-18)}} 
		& \multicolumn{4}{c}{\textbf{Tiny-Imagenet (ResNet-18)}}\\
		\cmidrule(lr){2-5} \cmidrule(lr){6-9} 
		& iid & Dir-0.5 & Dir-0.1 & Dir-0.05  
		& iid & Dir-0.5 & Dir-0.1 & Dir-0.05    \\
		\midrule
		FedAvg         & 65.04 & 65.07 & 60.17 & 56.75 & 53.88 & 52.55 & 47.48 & 43.80 \\
		SCAFFOLD       & 66.04 & 65.51 & 60.69 & 56.43 & 54.49 & 53.36 & 47.76 & 43.92 \\
		FedCM          & 71.12 & 69.85 & 66.61 & 62.65 & 46.51 & 44.06 & 41.16 & 36.00 \\
		Local AdamW    & 64.37 & 63.58 & 59.23 & 55.24 & 51.06 & 49.97 & 44.01 & 40.00 \\
		Fed-Sophia   & 62.56 & 60.62 & 57.29 & 51.02 & 49.86 & 47.62 & 41.89 & 36.65 \\
		FedPM   & 62.86 & 61.23 & 57.25 & 50.98 & 49.63 & 48.21 & 41.53 & 36.62 \\
		Local Sophia   & 62.17 & 60.96 & 56.65 & 50.89 & 49.28 & 47.99 & 41.23 & 36.15 \\
		Local Muon     & 72.78 & $\mathbf{73.02}$ & 67.26 & 49.86 & $\mathbf{60.50}$ & $\mathbf{60.40}$ & 52.83 & 34.76 \\
		Local SOAP     & 71.98 & 71.19 & 68.44 & 58.16 & 58.01 & 56.82 & 54.42 & 50.02 \\
		\rowcolor{LightGreen}
		\texttt{FedPAC\_Sophia} & 64.90 & 64.71 & 59.96 & 53.66 & 49.30 & 50.92 & 43.81 & 36.37 \\
		\rowcolor{LightBlue}
		\texttt{FedPAC\_Muon}   & $\mathbf{72.79}$ & 72.50 & $\mathbf{71.85}$ & $\mathbf{65.56}$ & 58.85 & 58.31 & $\mathbf{57.95}$ & $\mathbf{54.00}$ \\
		\rowcolor{LightRed}
		\texttt{FedPAC\_SOAP}   & 69.79 & 71.05 & 69.25 & 64.16 & 56.57 & 56.22 & 55.62 & 51.81 \\
		\bottomrule
	\end{tabular}
\end{table*}

\begin{table*}[tb]
	\centering
	\caption{\small Test accuracy, training loss of each method on CIFAR-100 using \textbf{ViT-Tiny} over 300 communication rounds under IID, Dir-0.5, Dir-0.1, Dir-0.05 (100 clients, 10\% participation, batch size 50, $K=50$).}
	\vspace{-2mm}
	\label{tab:combined_vit}
	\setlength{\tabcolsep}{7pt}
	\begin{tabular}{lcccccccc}
		\toprule
		\multirow{2}{*}{\textbf{Method}} 
		& \multicolumn{4}{c}{\textbf{CIFAR-100 (ViT-Tiny)}} 
		& \multicolumn{4}{c}{\textbf{Tiny-Imagenet (ViT-Tiny)}}\\
		\cmidrule(lr){2-5} \cmidrule(lr){6-9} 
		& iid & Dir-0.5 & Dir-0.1 & Dir-0.05  
		& iid & Dir-0.5 & Dir-0.1 & Dir-0.05    \\
		\midrule
		FedAvg         & 32.97 & 33.66 & 27.24 & 23.42 & 18.61 & 18.03 & 15.68 & 14.05 \\
		SCAFFOLD       & 32.29 & 33.01 & 26.86 & 23.23 & 18.32 & 18.36 & 15.70 & 14.21 \\
		FedCM          & 21.97 & 22.02 & 16.95 & 14.74 & 10.23 & 11.66 & 8.88  & 8.15  \\
		Local AdamW    & 41.27 & 41.21 & 37.57 & 36.06 & 27.94 & 26.82 & 24.31 & 21.35 \\
		Fed-Sophia   & 38.34 & 38.56 & 34.23 & 32.27 & 23.56 & 24.46 & 23.85 & 22.17 \\
		FedPM   & 38.56.94 & 38.75 & 34.29& 32.62 & 23.86 & 24.51 & 23.89 & 22.85 \\
		
		Local Sophia   & 37.94 & 38.48 & 34.05 & 32.25 & 22.90 & 23.46 & 22.49 & 21.14 \\
		Local Muon     & 50.50 & 51.75 & 47.81 & 41.76 & 34.80 & 34.57 & 30.51 & 28.25 \\
		Local SOAP     & 52.85 & 51.61 & 49.41 & 41.68 & 34.81 & 33.91 & 33.30 & 30.36 \\
		\rowcolor{LightGreen}
		\texttt{FedPAC\_Sophia}  & 42.53 & 42.29 & 39.79 & 32.71 & 23.26 & 22.16 & 23.01 & 22.37 \\
		\rowcolor{LightRed}
		\texttt{FedPAC\_Muon}    & 52.70 & 52.02 & 44.00 & 39.68 & 31.78 & 33.48 & 31.45 & 30.25 \\
		\rowcolor{LightBlue}
		\texttt{FedPAC\_SOAP}    & \textbf{52.86} & \textbf{52.46} & \textbf{51.16} & \textbf{47.55} & \textbf{35.56} & \textbf{35.12} & \textbf{34.32} & \textbf{31.33} \\
		\bottomrule
	\end{tabular}
\end{table*}

\subsection{More baseline experiment comparisons}
\label{app:C_cost}

\noindent\textbf{Training on CIFAR-100 with ResNet-18.}

Table~\ref{tab:compare_resnet} reports the final \emph{test accuracy} after 300 communication rounds on CIFAR-100 and Tiny-ImageNet with ResNet-18 under increasing data heterogeneity (IID, Dir-$0.5$, Dir-$0.1$, Dir-$0.05$; 100 clients, 10\% participation, batch size 50, and $K{=}50$ local steps).
Overall, performance degrades as the Dirichlet concentration decreases, highlighting the non-trivial impact of client drift under severe non-IID distributions.

\paragraph{Specialized second-order FL baselines.}
To directly address the concern that our gains may come from comparing against \emph{local} second-order optimizers only, we additionally include two \emph{specialized} second-order FL methods, \textbf{Fed-Sophia} and \textbf{FedPM}, which are designed to adapt curvature-aware (Sophia-style) updates to the federated setting.
As shown in Table~\ref{tab:compare_resnet}, both specialized baselines still exhibit noticeable performance drops as heterogeneity increases.
For example, on CIFAR-100, Fed-Sophia decreases from 60.62 (Dir-$0.5$) to 51.02 (Dir-$0.05$), and FedPM decreases from 61.23 (Dir-$0.5$) to 50.98 (Dir-$0.05$), indicating that merely introducing a specialized second-order aggregation scheme does not fully resolve the instability induced by \emph{preconditioner drift}.

\paragraph{FedPAC improves robustness under strong heterogeneity.}
Across both datasets, \texttt{FedPAC\_Muon} yields the best or near-best accuracy under strong non-IID settings, demonstrating substantial robustness benefits.
On CIFAR-100 with Dir-$0.05$, \texttt{FedPAC\_Muon} achieves 65.56, improving over FedAvg (56.75) by +8.81 and over Local Muon (49.86) by +15.70.
On Tiny-ImageNet with Dir-$0.05$, \texttt{FedPAC\_Muon} reaches 54.00, outperforming FedAvg (43.80) by +10.20 and Local Muon (34.76) by +19.24.
These results align with our analysis that the server-side \emph{alignment} and client-side \emph{correction} in FedPAC explicitly mitigate the mismatch among client preconditioners, which becomes particularly harmful when data heterogeneity is severe.

\paragraph{FedPAC also complements Sophia-style methods.}
Compared with the specialized Sophia-style FL baselines, \texttt{FedPAC\_Sophia} provides consistent gains on CIFAR-100 across all heterogeneity levels (e.g., 53.66 vs.\ 51.02/50.98 under Dir-$0.05$ for Fed-Sophia/FedPM).
On Tiny-ImageNet, \texttt{FedPAC\_Sophia} is competitive with Fed-Sophia/FedPM and improves them under mild-to-moderate heterogeneity (IID/Dir-$0.5$/Dir-$0.1$), while being essentially on par under Dir-$0.05$.
Taken together, these comparisons suggest that FedPAC addresses a more general failure mode---\emph{preconditioner drift}---and thus can serve as a robust, optimizer-agnostic enhancement for federated second-order training.

\paragraph{Takeaway.}
Table~\ref{tab:compare_resnet} verifies that (i) directly deploying local second-order optimizers can be brittle under strong heterogeneity, and (ii) even specialized second-order FL methods may not fully eliminate the degradation, whereas FedPAC substantially improves robustness and accuracy, especially in the most heterogeneous regimes.

\noindent\textbf{Training on CIFAR-100 with ViT-Tiny.}

\subsection{More baseline experiment on iid data}
\label{app:C_more_baselines}

To further strengthen the empirical comparisons (especially on \textbf{IID} data) and to examine the robustness trend as heterogeneity increases, we report additional results under IID and Dirichlet partitions with varying concentration parameters.
Tables~\ref{tab:compare_resnet} and~\ref{tab:combined_vit} summarize the final \emph{test accuracy} after 300 communication rounds on CIFAR-100 and Tiny-ImageNet with \textbf{ResNet-18} and \textbf{ViT-Tiny}, respectively (100 clients, 10\% participation, batch size 50, $K{=}50$).

\paragraph{ResNet-18 results.}
From Table~\ref{tab:compare_resnet}, we observe that under IID (and mild heterogeneity such as Dir-$0.5$), several strong baselines can be competitive (e.g., \texttt{Local Muon} on CIFAR-100).
However, as heterogeneity becomes severe (Dir-$0.1$ and Dir-$0.05$), methods that rely on purely local second-order states (e.g., \texttt{Local Muon}) or specialized second-order FL baselines (e.g., \textbf{Fed-Sophia} and \textbf{FedPM}) degrade markedly, indicating the growing impact of \emph{preconditioner drift} across clients.
In contrast, \texttt{FedPAC\_Muon} remains consistently strong and achieves the best accuracy under the most heterogeneous regime on both datasets, demonstrating improved robustness under non-IID distributions.

\paragraph{ViT-Tiny results.}
Table~\ref{tab:combined_vit} shows a similar trend for transformer-style backbones: while several baselines perform reasonably under IID/Dir-$0.5$, severe heterogeneity leads to substantial drops.
Notably, \texttt{FedPAC\_SOAP} achieves the best performance across all heterogeneity levels on both CIFAR-100 and Tiny-ImageNet with ViT-Tiny, suggesting that FedPAC can effectively stabilize curvature-aware training for modern architectures under client heterogeneity.

\section{Related Work}

\begin{table}[H]
	\centering
	\small
	\setlength{\tabcolsep}{4pt}
	\begin{tabular}{p{2.2cm} p{3.3cm} p{2.6cm} p{3.2cm}}
		\toprule
		\textbf{Family} & \textbf{What is synchronized (server $\leftrightarrow$ clients)?} & \textbf{Extra cost (comm./state)} & \textbf{Main issue it targets} \\
		\midrule
		
		FedAvg / Local-SGD
		& Model parameters/updates $x$ only
		& $\approx |x|$ per round
		& General FL baseline; does not explicitly correct non-IID induced drift \\
		\midrule
		
		FedOpt (server adaptive; e.g., FedAdam/FedYogi)
		& $x$ only; optimizer state (moments) kept \emph{on server} (not shared as a common client geometry)
		& $\approx |x|$ (no extra client-side sync)
		& Improves global update scaling/adaptivity, but does \emph{not} align heterogeneous client metrics \\
		\midrule
		
		Control-variate / drift correction (e.g., SCAFFOLD)
		& $x$ + first-order control variates (gradient-level states)
		& typically larger (e.g., SCAFFOLD uses $\approx 2|x|$ in our comm accounting)
		& Corrects \emph{first-order} client drift under non-IID via control variates; operates at gradient level, not metric level \\
		\midrule
		
		Naive second-order FL (Local Sophia/SOAP/Muon + FedAvg agg.)
		& $x$ only; each client maintains local preconditioner state $\Theta$ but \emph{does not synchronize it}
		& $\approx |x|$ per round
		& Fails under non-IID due to \textbf{preconditioner drift} (mismatched curvature-induced geometries) \\
		\midrule
		
		\textbf{FedPAC (ours)}
		& $x$ + \textbf{preconditioner state $\Theta$ (geometry / metric)}: server aggregates $\Theta$ and broadcasts a global reference (Alignment); local steps are corrected using a global preconditioned direction (Correction)
		& $\approx |x| + c|\Theta|$ (e.g., $c=1$ for Muon/Sophia, $c=2$ for SOAP; light variants use low-rank upload such as $0.05|\Theta|$--$0.1|\Theta|$)
		& \textbf{Targets preconditioner drift directly} by synchronizing the \emph{metric} (Alignment) and suppressing long-term drift (Correction) \\
		\bottomrule
	\end{tabular}
	\caption{\textbf{Positioning of FedPAC.} Unlike FedOpt and control-variate methods that synchronize only parameter/first-order states, FedPAC explicitly synchronizes the \emph{curvature-defined geometry} (preconditioner state $\Theta$) to mitigate \emph{preconditioner drift} under non-IID data. Here $|x|$ denotes per-round model communication in our accounting and $|\Theta|$ denotes the size of the optimizer preconditioner state.}
	\label{tab:fedpac-positioning}
\end{table}

\begin{proposition}[Geometry drift induces preconditioned disagreement]\label{prop:geom-drift}
	Consider a second-order FL update where each client applies a preconditioner operator $P_{\Theta}$ (defined by its optimizer state $\Theta$) to a local vector $u_i$ (e.g., stochastic gradient). 
	Assume $P_{\Theta}$ is Lipschitz in the state: for any $\Theta,\Theta'$ and any $v$,
	\begin{equation}
		\|P_{\Theta}(v)-P_{\Theta'}(v)\| \le L_{\Theta}\,\|\Theta-\Theta'\|\,\|v\|.
	\end{equation}
	Define the per-round preconditioner drift metric
	\begin{equation}
		\Delta_r^{D} := \frac{1}{S}\sum_{i\in S_r}\mathbb{E}\|\Theta_{r,K}^{i}-\bar{\Theta}_{r,K}\|^2 .
	\end{equation}
	Then for any collection $\{u_i\}_{i\in S_r}$,
	\begin{equation}
		\mathbb{E}\Big\|\frac{1}{S}\sum_{i\in S_r}\Big(P_{\Theta_{r,k}^{i}}(u_i)-P_{\bar{\Theta}_{r,K}}(u_i)\Big)\Big\|^2
		\;\le\;
		L_{\Theta}^2\Big(\frac{1}{S}\sum_{i\in S_r}\mathbb{E}\|u_i\|^2\Big)\Delta_r^{D}.
	\end{equation}
	In particular, if $\sup_{i,r,k}\mathbb{E}\|u_i\|^2\le G^2$, then
	\begin{equation}
		\mathbb{E}\Big\|\frac{1}{S}\sum_{i\in S_r}\Big(P_{\Theta_{r,k}^{i}}(u_i)-P_{\bar{\Theta}_{r,K}}(u_i)\Big)\Big\|^2
		\;\le\;
		L_{\Theta}^2 G^2 \Delta_r^{D}.
	\end{equation}
	Therefore, even when there is \emph{no gradient/client mismatch} (e.g., $u_i\equiv u$ for all clients, or heterogeneity is negligible), heterogeneous optimizer states $\Theta_{r,k}^i$ alone can induce a non-vanishing disagreement in the \emph{preconditioned} update direction, which is a distinct error source from gradient-level client drift.
\end{proposition}


\begin{theorem}[Geometry drift induces preconditioned disagreement]\label{thm:geom-drift}
	Consider a round $r$ with a participating client set $S_r$ of size $S$.
	Each client maintains a (second-order) optimizer state $\Theta_{r,k}^i$ and applies the corresponding
	preconditioner operator $P_{\Theta}$ to a local update vector $u_i$ (e.g., stochastic gradient).
	Assume the operator is Lipschitz in the state: there exists $L_{\Theta}>0$ such that for any $\Theta,\Theta'$ and any $v$,
	\begin{equation}\label{eq:lipschitz-P}
		\|P_{\Theta}(v)-P_{\Theta'}(v)\| \le L_{\Theta}\,\|\Theta-\Theta'\|\,\|v\|.
	\end{equation}
	Let $\bar{\Theta}_{r,K}:=\frac{1}{S}\sum_{i\in S_r}\Theta_{r,K}^i$ be the averaged state at the end of local steps, and define the
	preconditioner drift metric
	\begin{equation}\label{eq:drift-metric}
		\Delta_r^{D} := \frac{1}{S}\sum_{i\in S_r}\mathbb{E}\|\Theta_{r,K}^{i}-\bar{\Theta}_{r,K}\|^2.
	\end{equation}
	Then the disagreement between using heterogeneous local geometries and the averaged geometry satisfies
	\begin{equation}\label{eq:precond-disagreement}
		\mathbb{E}\Big\|
		\frac{1}{S}\sum_{i\in S_r}\big(P_{\Theta_{r,k}^{i}}(u_i)-P_{\bar{\Theta}_{r,K}}(u_i)\big)
		\Big\|^2
		\;\le\;
		L_{\Theta}^2\Big(\frac{1}{S}\sum_{i\in S_r}\mathbb{E}\|u_i\|^2\Big)\Delta_r^{D}.
	\end{equation}
	In particular, if $\sup_{i,r,k}\mathbb{E}\|u_i\|^2\le G^2$, then
	\begin{equation}\label{eq:precond-disagreement-simplified}
		\mathbb{E}\Big\|
		\frac{1}{S}\sum_{i\in S_r}\big(P_{\Theta_{r,k}^{i}}(u_i)-P_{\bar{\Theta}_{r,K}}(u_i)\big)
		\Big\|^2
		\;\le\;
		L_{\Theta}^2G^2\,\Delta_r^{D}.
	\end{equation}
\end{theorem}

\begin{corollary}[Drift penalty vanishes under geometry alignment]\label{cor:drift-vanish}
	If the client geometries are aligned in round $r$ in the sense that $\Theta_{r,k}^i=\bar{\Theta}_{r,K}$ for all $i\in S_r$
	(equivalently, $\Delta_r^{D}=0$), then the preconditioned disagreement term in~\eqref{eq:precond-disagreement} is zero:
	\begin{equation}\label{eq:drift-zero}
		\frac{1}{S}\sum_{i\in S_r}P_{\Theta_{r,k}^{i}}(u_i)
		\;=\;
		\frac{1}{S}\sum_{i\in S_r}P_{\bar{\Theta}_{r,K}}(u_i).
	\end{equation}
	Consequently, any residual inconsistency across clients is solely due to the \emph{vector-level} mismatch in $\{u_i\}$
	(e.g., gradient/client drift from non-IID data), rather than a mismatch in the preconditioning geometry.
\end{corollary}

\begin{remark}[Why geometry drift is distinct from gradient/client drift]\label{rem:why-geometry}
	Gradient/client drift concerns the mismatch of the \emph{update vectors} $\{u_i\}$ across clients (e.g., $\nabla f_i \neq \nabla f$),
	and control-variate methods aim to reduce this vector-level discrepancy.
	In contrast, geometry drift concerns a mismatch of the \emph{operators} $\{P_{\Theta_{r,k}^i}\}$, i.e., clients optimize under different
	local metrics defined by their second-order states.
	Theorem shows that even if $\{u_i\}$ were identical across clients (no gradient/client drift),
	heterogeneous geometries ($\Delta_r^{D}>0$) alone induce a non-vanishing disagreement in the \emph{preconditioned} direction.
	Corollary further clarifies that synchronizing/aligning $\Theta$ removes this geometry-induced error source,
	leaving only the conventional gradient/client drift to be handled.
\end{remark}


\begin{algorithm}[tb]
	\small
	\caption{Local SOAP Algorithm}
	\begin{algorithmic}[1]
		\REQUIRE Step of Local SOAP for an $m \times n$ layer. 
		Per layer, we maintain four matrices: 
		$L \in \mathbb{R}^{m\times m}$, 
		$R \in \mathbb{R}^{n\times n}$, 
		$V, M \in \mathbb{R}^{m\times n}$.
		Hyperparameters: learning rate $\eta$, betas $(\beta_1,\beta_2)$, 
		epsilon $\epsilon$, and preconditioning frequency $f$, communication rounds $T$, local updates $K$, the number of client $N$.
		\FOR{$t = 0, \dots, T$}
		\FOR{each client $i \in \{1, \dots, N\}$ in parallel}
		\FOR{$k = 1, \dots, K$}
		\STATE Sample batch $B_i^{t,k}$
		\STATE $G_i^{t,k} \in \mathbb{R}^{m\times n} \gets -\nabla_W \phi_{B_t}(W_i^{t,k})$

		\STATE {\{ End of gradient step, now update $L$ and $R$ 
			and possibly also $Q_L$ and $Q_R$ \}}
		
		\STATE $L \gets \beta_2 L + (1 - \beta_2) (GG^\top)$
		\STATE $R \gets \beta_2 R + (1 - \beta_2) (G^\top G)$
		
		\IF{$t \bmod f = 0$}
		\STATE $Q_L \gets \texttt{Eigenvectors}(L, Q_L)$
		\STATE $Q_R \gets \texttt{Eigenvectors}(R, Q_R)$
		\ENDIF
		\STATE $g_i^{t,k} \gets Q_L^\top G Q_R$
		\STATE \hfill\rule{1\linewidth}{0.8pt}\hfill
		\STATE \textbf{UPDATE($g_i^{t,k}$) by Adam}
		\STATE $M_i^{t,k} \gets \beta_1 M_i^{t,k} + (1 - \beta_1) g_i^{t,k}$
		
		
		\STATE $V_i^{t,k} \gets \beta_2 V_i^{t,k} + (1 - \beta_2) (g_i^{t,k} \odot g_i^{t,k})$
		\STATE $N_i^{t,k} \gets M_i^{t,k}/(\sqrt{V_i^{t,k} }+ \epsilon)$

		\STATE \hfill\rule{1\linewidth}{0.8pt}\hfill
		\STATE {\{ Now that we have preconditioned by Adam in the rotated space, 
			we go back to the original space \}}
		\STATE $\tilde{N}_i^{t,k}  \gets Q_L N_i^{t,k}  Q_R^\top$
		\STATE $W_i^{t,k}  \gets W_i^{t,k-1}  - \eta \tilde{N}_i^{t,k} $
		
		\ENDFOR
		\STATE Clint $i$ communicate $(W_i^{t,K}-W_i^{t,0})$ to Server;
		\ENDFOR
		\STATE $W^{t+1} \gets W^t +  \frac{1}{N}\sum_i(W_i^{t,K}-W_i^{t,0}) $
		\ENDFOR
		\STATE
		\STATE \texttt{Eigenvectors}$(P,Q)$:
		\STATE \hspace{1.5em}  $S \gets P Q$
		\STATE \hspace{1.5em} $Q \gets \mathrm{QR}(S)$ \{Using QR decomposition\}
		\STATE \texttt{Return} $Q$
	\end{algorithmic}
\end{algorithm}

\begin{algorithm}[tb]
	\small
	\caption{FedPAC\_SOAP Algorithm}
	\begin{algorithmic}[1]
		\REQUIRE Step of FedPAC\_SOAP  for an $m \times n$ layer. 
		Per layer, we maintain four matrices: 
		$L \in \mathbb{R}^{m\times m}$, 
		$R \in \mathbb{R}^{n\times n}$, 
		$V, M \in \mathbb{R}^{m\times n}$.
		Hyperparameters: learning rate $\eta$, betas $(\beta_1,\beta_2)$, 
		epsilon $\epsilon$, and preconditioning frequency $f$, communication rounds $R$, local updates $K$, the number of client $N$.
		\FOR{$r = 0, \dots, R$}
		\FOR{each client $i \in \{1, \dots, N\}$ in parallel}
		\FOR{$k = 1, \dots, K$}
		\STATE Sample batch $B_i^{r,k}$
		\STATE $G_i^{r,k} \in \mathbb{R}^{m\times n} \gets -\nabla F_i(x_i^{r,k}; \xi_i^{r,k})$

		\STATE {\{ End of gradient step, now update $L$ and $R$ 
			and possibly also $Q_L$ and $Q_R$ \}}
		
		\STATE $L \gets \beta_2 L + (1 - \beta_2) (GG^\top)$
		\STATE $R \gets \beta_2 R + (1 - \beta_2) (G^\top G)$
		
		\IF{$t \bmod f = 0$}
		\STATE $Q_L \gets \texttt{Eigenvectors}(L, Q_L)$
		\STATE $Q_R \gets \texttt{Eigenvectors}(R, Q_R)$
		\ENDIF
		\STATE $g_i^{t,k} \gets Q_L^\top G Q_R$
		\STATE \hfill\rule{1\linewidth}{0.8pt}\hfill
		\STATE \textbf{UPDATE($g_i^{r,k}$) by Adam}
		\STATE $M_i^{r,k} \gets \beta_1 M_i^{r,k} + (1 - \beta_1) g_i^{r,k}$
		
		
		\STATE $V_i^{r,k} \gets \beta_2 V_i^{r,k} + (1 - \beta_2) (g_i^{r,k} \odot g_i^{r,k})$
		\STATE $N_i^{r,k} \gets M_i^{r,k}/(\sqrt{V_i^{r,k} }+ \epsilon)$

		\STATE \hfill\rule{1\linewidth}{0.8pt}\hfill
		\STATE {\{ Now that we have preconditioned by Adam in the rotated space, 
			we go back to the original space \}}
		\STATE $\tilde{N}_i^{r,k}  \gets Q_L N_i^{r,k}  Q_R^\top$
		\STATE $\boldsymbol{x}^{r,k+1}_i\! =\!\boldsymbol{x}^{r,k}_i \!\!-\! \eta [(1\!-\!\beta)\tilde{N}_i^{r,k}\!+\!\beta\boldsymbol{\Delta}_G^r ]$;
		
		\ENDFOR
		\STATE Client $i$ communicate $(x_i^{r,K}-x_i^{r,0})$ to Server;
		\ENDFOR
		\STATE $\boldsymbol{\Delta}_G^{r+1}=-\frac{1}{SK\eta} \sum_{i=1}^S (\boldsymbol{x}^{r, K}_i-\boldsymbol{x}^{r, 0}_i)$; 
		\STATE$\boldsymbol{x}^{r+1} =\boldsymbol{x}^{r} +\frac{1}{S} \sum_{i=1}^S (\boldsymbol{x}^{r, K}_i-\boldsymbol{x}^{r, 0}_i)$;
		\ENDFOR
		\STATE
		\STATE \texttt{Eigenvectors}$(P,Q)$:
		\STATE \hspace{1.5em}  $S \gets P Q$
		\STATE \hspace{1.5em} $Q \gets \mathrm{QR}(S)$ \{Using QR decomposition\}
		\STATE \texttt{Return} $Q$
	\end{algorithmic}
\end{algorithm}

\begin{algorithm}[tb]
	\small
	\caption{Local Muon Algorithm (Federated Setting)}
	\begin{algorithmic}[1]
		\REQUIRE Local Muon for an $m\times n$ weight matrix $W$ (e.g., a Linear layer).
		Per layer, maintain one matrix: $M\in\mathbb{R}^{m\times n}$ (momentum).
		Hyperparameters: learning rate $\eta$, momentum $\beta$,
		weight decay $\lambda$ (optional),
		epsilon $\epsilon$, Newton--Schulz steps $s$ (typically $5$),
		dimension scaling $\gamma(m,n)$,
		communication rounds $T$, local steps $K$, number of clients $N$.
		
		\FOR{$t=0,\dots,T$}
		\FOR{each client $i\in\{1,\dots,N\}$ in parallel}
		\STATE Initialize local model: $W_i^{t,0}\gets W^{t}$
		\FOR{$k=1,\dots,K$}
		\STATE Sample batch $B_i^{t,k}$
		\STATE $G_i^{t,k}\in\mathbb{R}^{m\times n} \gets \nabla_W \phi_{B_i^{t,k}}(W_i^{t,k-1})$
		\STATE $M_i^{t,k} \gets \beta_1\, M_i^{t,k-1} + (1-\beta_1)\, G_i^{t,k}$
		
		\STATE {\{ Orthogonalize (typically on momentum) via Newton--Schulz \}}
		\STATE $U_i^{t,k} \gets \texttt{NewtonSchulz}(M_i^{t,k}, s, \epsilon)$
		\COMMENT{$U$ is approximately the orthogonal factor of $M$}
		
		\STATE {\{ Apply dimension scaling (one common theoretical form: $\gamma(m,n)=\sqrt{m/n}$) \}}
		\STATE $\Delta W_i^{t,k} \gets \gamma(m,n)\, U_i^{t,k}$
		\COMMENT{e.g. $\gamma=\sqrt{\texttt{fan-out}/\texttt{fan-in}}$}
		
		\STATE {\{ Update \}}
		\STATE $W_i^{t,k} \gets W_i^{t,k-1} - \eta\left(\Delta W_i^{t,k} + \lambda W_i^{t,k-1}\right)$
		\COMMENT{optional weight decay}
		
		\ENDFOR
		\STATE Client $i$ sends $(W_i^{t,K}-W_i^{t,0})$ to Server
		\ENDFOR
		\STATE $W^{t+1} \gets W^{t} + \frac{1}{N}\sum_{i=1}^{N}(W_i^{t,K}-W_i^{t,0})$
		\ENDFOR
		
		\STATE
		\STATE \texttt{NewtonSchulz}$(G,s,\epsilon)$:
		\STATE \hspace{1.5em} $X \gets G / (\|G\|_F + \epsilon)$
		\STATE \hspace{1.5em} \textbf{if} $m>n$ \textbf{then} $X\gets X^\top$ \textbf{end if}
		\STATE \hspace{1.5em} \textbf{for} $j=1,\dots,s$ \textbf{do}
		\STATE \hspace{3.0em} $A \gets X X^\top$
		\STATE \hspace{3.0em} $B \gets bA + cA^2$
		\STATE \hspace{3.0em} $X \gets aX + BX$
		\STATE \hspace{1.5em} \textbf{end for}
		\STATE \hspace{1.5em} \textbf{if} $m>n$ \textbf{then} $X\gets X^\top$ \textbf{end if}
		\STATE \hspace{1.5em} \textbf{return} $X$
		\COMMENT{common coefficients: $a{=}3.4445,\,b{=}{-}4.7750,\,c{=}2.0315$}
	\end{algorithmic}
\end{algorithm}

\begin{algorithm}[tb]
	\small
	\caption{FedPAC\_Muon Algorithm}
	\begin{algorithmic}[1]
		\REQUIRE Step of FedPAC\_Muon for an $m\times n$ layer.
		Per layer, maintain one matrix:
		$M\in\mathbb{R}^{m\times n}$ (momentum).
		Hyperparameters: learning rate $\eta$, momentum $\beta$,
		epsilon $\epsilon$, Newton--Schulz steps $s$,
		(optional) weight decay $\lambda$,
		(optional) dimension scaling $\gamma(m,n)$,
		communication rounds $R$, local updates $K$, number of clients $N$,
		FedPAC mixing coefficient $\beta\in[0,1]$.
		Server maintains global direction $\boldsymbol{\Delta}_G^{r}$.
		
		\FOR{$r=0,\dots,R$}
		\FOR{each client $i\in\{1,\dots,N\}$ in parallel}
		\FOR{$k=0,\dots,K-1$}
		\STATE Sample batch $B_i^{r,k}$
		\STATE $G_i^{r,k}\in\mathbb{R}^{m\times n} \gets -\nabla F_i(\boldsymbol{x}_i^{r,k};\xi_i^{r,k})$
		\STATE $M_i^{r,k+1} \gets \beta_1\, M_i^{r,k} + (1-\beta_1)\, G_i^{r,k}$
		
		\STATE {\{ Muon: orthogonalize the (momentum) update direction \}}
		\STATE $U_i^{r,k} \gets \texttt{NewtonSchulz}(M_i^{r,k+1}, s, \epsilon)$
		\STATE $\tilde U_i^{r,k} \gets \gamma(m,n)\, U_i^{r,k}$ \COMMENT{optional scaling, e.g. $\gamma=\sqrt{m/n}$}
		
		\STATE {\{ FedPAC local update with global direction mixing \}}
		\STATE $\boldsymbol{x}_i^{r,k+1} = \boldsymbol{x}_i^{r,k} - \eta\Big[(1-\beta)\tilde U_i^{r,k} + \beta\boldsymbol{\Delta}_G^{r}\Big]$
		\COMMENT{optionally add $+\lambda \boldsymbol{x}_i^{r,k}$ inside the bracket}
		\ENDFOR
		\STATE Client $i$ communicates $(\boldsymbol{x}_i^{r,K}-\boldsymbol{x}_i^{r,0})$ to Server;
		\ENDFOR
		
		\STATE $\boldsymbol{\Delta}_G^{r+1}=-\frac{1}{NK\eta}\sum_{i=1}^{N}\big(\boldsymbol{x}_i^{r,K}-\boldsymbol{x}_i^{r,0}\big)$;
		\STATE $\boldsymbol{x}^{r+1}=\boldsymbol{x}^{r}+\frac{1}{N}\sum_{i=1}^{N}\big(\boldsymbol{x}_i^{r,K}-\boldsymbol{x}_i^{r,0}\big)$;
		\ENDFOR
		
		\STATE
		\STATE \texttt{NewtonSchulz}$(G,s,\epsilon)$:
		\STATE \hspace{1.5em} $X \gets G / (\|G\|_F+\epsilon)$
		\STATE \hspace{1.5em} \textbf{if} $m>n$ \textbf{then} $X\gets X^\top$ \textbf{end if}
		\STATE \hspace{1.5em} \textbf{for} $j=1,\dots,s$ \textbf{do}
		\STATE \hspace{3.0em} $A \gets XX^\top$
		\STATE \hspace{3.0em} $X \gets \frac{1}{2}X(3I-A)$ \COMMENT{classic Newton--Schulz for $(XX^\top)^{-1/2}$}
		\STATE \hspace{1.5em} \textbf{end for}
		\STATE \hspace{1.5em} \textbf{if} $m>n$ \textbf{then} $X\gets X^\top$ \textbf{end if}
		\STATE \hspace{1.5em} \textbf{return} $X$
	\end{algorithmic}
\end{algorithm}

\begin{algorithm}[tb]
	\small
	\caption{Local Sophia Algorithm (Federated Setting)}
	\begin{algorithmic}[1]
		\REQUIRE Local Sophia for an $m\times n$ layer $W$.
		Per layer, maintain two matrices:
		$M, H \in \mathbb{R}^{m\times n}$ (momentum and Hessian-diagonal EMA).
		Hyperparameters: learning rate $\eta$, betas $(\beta_1,\beta_2)$,
		epsilon $\epsilon$, clipping $\rho$,
		Hessian update frequency $f_h$,
		communication rounds $T$, local steps $K$, number of clients $N$,
		weight decay $\lambda$ (optional).
		
		\FOR{$t=0,\dots,T$}
		\FOR{each client $i\in\{1,\dots,N\}$ in parallel}
		\STATE Initialize local model: $W_i^{t,0} \gets W^{t}$
		\FOR{$k=1,\dots,K$}
		\STATE Sample batch $B_i^{t,k}$
		\STATE $G_i^{t,k} \gets \nabla_W \phi_{B_i^{t,k}}(W_i^{t,k-1})$
		\STATE $M_i^{t,k} \gets \beta_1 M_i^{t,k-1} + (1-\beta_1)G_i^{t,k}$
		
		\STATE {\{ Update $H$ (diagonal Hessian estimate) occasionally \}}
		\IF{$((tK+k)\bmod f_h)=0$}
		\STATE Sample Rademacher noise $U_i^{t,k}\in\{\pm 1\}^{m\times n}$
		\STATE $HV_i^{t,k} \gets \nabla_W^2 \phi_{B_i^{t,k}}(W_i^{t,k-1})\, U_i^{t,k}$
		\COMMENT{HVP via auto-diff (Pearlmutter trick)}
		\STATE $\widehat{H}_i^{t,k} \gets U_i^{t,k} \odot HV_i^{t,k}$
		\COMMENT{Unbiased diag(H) estimator}
		\STATE $H_i^{t,k} \gets \beta_2 H_i^{t,k-1} + (1-\beta_2)\max(\widehat{H}_i^{t,k},0)$
		\ELSE
		\STATE $H_i^{t,k} \gets H_i^{t,k-1}$
		\ENDIF
		
		\STATE $D_i^{t,k} \gets M_i^{t,k} \oslash \max(H_i^{t,k},\epsilon)$
		\COMMENT{element-wise divide}
		\STATE $D_i^{t,k} \gets \texttt{clip}(D_i^{t,k}, -\rho, \rho)$
		\STATE $W_i^{t,k} \gets W_i^{t,k-1} - \eta\left(D_i^{t,k} + \lambda W_i^{t,k-1}\right)$
		\COMMENT{optional weight decay}
		\ENDFOR
		\STATE Client $i$ sends $(W_i^{t,K}-W_i^{t,0})$ to Server
		\ENDFOR
		\STATE $W^{t+1} \gets W^{t} + \frac{1}{N}\sum_{i=1}^{N}(W_i^{t,K}-W_i^{t,0})$
		\ENDFOR
		
		\STATE
		\STATE \texttt{clip}$(X,a,b)$: element-wise $\min(\max(X,a),b)$.
	\end{algorithmic}
\end{algorithm}

\begin{algorithm}[tb]
	\small
	\caption{PAC\_Sophia Algorithm (Federated Setting)}
	\begin{algorithmic}[1]
		\REQUIRE PAC\_Sophia for an $m\times n$ layer $W$.
		Per layer, maintain two matrices:
		$M, H \in \mathbb{R}^{m\times n}$ (momentum and Hessian-diagonal EMA).
		Hyperparameters: learning rate $\eta$, betas $(\beta_1,\beta_2)$,
		epsilon $\epsilon$, clipping $\rho$,
		Hessian update frequency $f_h$,
		communication rounds $R$, local steps $K$, number of clients $N$,
		PAC mixing coefficient $\beta\in[0,1]$,
		weight decay $\lambda$ (optional).
		Server maintains global direction $\boldsymbol{\Delta}_G^{r}$.
		
		\FOR{$r=0,\dots,R$}
		\FOR{each client $i\in\{1,\dots,N\}$ in parallel}
		\STATE Initialize local model: $W_i^{r,0} \gets W^{r}$
		\FOR{$k=1,\dots,K$}
		\STATE Sample batch $B_i^{r,k}$
		\STATE $G_i^{r,k} \gets \nabla_W \phi_{B_i^{r,k}}(W_i^{r,k-1})$
		\STATE $M_i^{r,k} \gets \beta_1 M_i^{r,k-1} + (1-\beta_1)G_i^{r,k}$
		
		\STATE {\{ Update $H$ (diagonal Hessian estimate) occasionally \}}
		\IF{$((rK+k)\bmod f_h)=0$}
		\STATE Sample Rademacher noise $U_i^{r,k}\in\{\pm 1\}^{m\times n}$
		\STATE $HV_i^{r,k} \gets \nabla_W^2 \phi_{B_i^{r,k}}(W_i^{r,k-1})\, U_i^{r,k}$
		\COMMENT{HVP via auto-diff (Pearlmutter trick)}
		\STATE $\widehat{H}_i^{r,k} \gets U_i^{r,k} \odot HV_i^{r,k}$
		\COMMENT{Unbiased diag(H) estimator}
		\STATE $H_i^{r,k} \gets \beta_2 H_i^{r,k-1} + (1-\beta_2)\max(\widehat{H}_i^{r,k},0)$
		\ELSE
		\STATE $H_i^{r,k} \gets H_i^{r,k-1}$
		\ENDIF
		
		\STATE $D_i^{r,k} \gets M_i^{r,k} \oslash \max(H_i^{r,k},\epsilon)$
		\STATE $D_i^{r,k} \gets \texttt{clip}(D_i^{r,k}, -\rho, \rho)$
		\COMMENT{Sophia direction}
		
		\STATE {\{ PAC mixing with global direction \}}
		\STATE $W_i^{r,k} \gets W_i^{r,k-1} - \eta\Big[(1-\beta)D_i^{r,k} + \beta\boldsymbol{\Delta}_G^{r} + \lambda W_i^{r,k-1}\Big]$
		\COMMENT{$\lambda$ optional}
		\ENDFOR
		\STATE Client $i$ communicates $(W_i^{r,K}-W_i^{r,0})$ to Server
		\ENDFOR
		
		\STATE $\boldsymbol{\Delta}_G^{r+1}=-\frac{1}{NK\eta}\sum_{i=1}^{N}\big(W_i^{r,K}-W_i^{r,0}\big)$;
		\STATE $W^{r+1}=W^{r}+\frac{1}{N}\sum_{i=1}^{N}\big(W_i^{r,K}-W_i^{r,0}\big)$;
		\ENDFOR
		
		\STATE
		\STATE \texttt{clip}$(X,a,b)$: element-wise $\min(\max(X,a),b)$.
	\end{algorithmic}
\end{algorithm}

\end{document}